\documentclass{wiley-article}

\usepackage{graphicx}        
\usepackage{ulem}

\usepackage{amsfonts,amsmath,amssymb,color}

\usepackage{mathrsfs}
\usepackage{comment}
\usepackage{subcaption}

\newcommand{\Gauss}{\mathrm{Gauss}}

\newcommand{\Zbf}{\mathbf{Z}}
\newcommand{\zbf}{\mathbf{z}}
\newcommand{\Cbb}{\mathbb{C}}

\newcommand{\Nbb}{\mathbb{N}}

\newcommand{\PC}{\mathscr{PC}}
\newcommand{\proE}{\mathrm{proE}}
\newcommand{\proHS}{\mathrm{proHS}}
\newcommand{\BW}{\mathrm{BW}}

\newcommand{\Xbf}{\mathbf{X}}

\newcommand{\cov}{\mathrm{cov}}     
             
\newcommand{\myspan}{\mathrm{span}}
\newcommand{\gbf}{\mathbf{g}}
\newcommand{\fbf}{\mathbf{f}}
\newcommand{\Abf}{\mathbf{A}}
\newcommand{\Bbf}{\mathbf{B}}
\newcommand{\Wbf}{\mathbf{W}}
\newcommand{\GP}{\mathrm{GP}}
\newcommand{\Srm}{\mathrm{S}}
\newcommand{\ep}{\epsilon}

\newcommand{\mapto}{\ensuremath{\rightarrow}}

\newcommand{\approach}{\ensuremath{\rightarrow}}
\newcommand{\imply}{\ensuremath{\Rightarrow}}

\newcommand{\equivalent}{\ensuremath{\Longleftrightarrow}}
\newcommand{\inclusion}{\ensuremath{\hookrightarrow}}

\newcommand{\Xcal}{\mathcal{X}}

\newcommand{\N}{\mathbb{N}}
\newcommand{\R}{\mathbb{R}}

\newcommand{\bP}{\mathbb{P}}
\newcommand{\bE}{\mathbb{E}}

\newcommand{\Ecal}{\mathcal{E}}

\newcommand{\la}{\langle}
\newcommand{\ra}{\rangle}
\newcommand{\Fcal}{\mathcal{F}}
\renewcommand{\H}{\mathcal{H}}
\newcommand{\Lcal}{\mathcal{L}}

\newcommand{\diag}{\mathrm{diag}}

\newcommand{\range}{\mathrm{range}}

\newcommand{\trace}{\mathrm{tr}}

\newcommand\HS{{\rm HS}}
\newcommand\eHS{{\rm HS_X}}

\newcommand{\Sym}{\mathrm{Sym}}
\newcommand{\logHS}{\mathrm{logHS}}

\newcommand{\logE}{\mathrm{logE}}

\newcommand{\tr}{\mathrm{tr}}

\newcommand{\Tr}{\mathrm{Tr}}

\newcommand{\x}{\mathbf{x}}

\newcommand{\z}{\mathbf{z}}

\renewcommand{\b}{\mathbf{b}}

\newcommand{\f}{\mathbf{f}}

\newcommand{\1}{\mathbf{1}}

\newcommand{\aiHS}{{\rm aiHS}}
\newcommand{\aiE}{{\rm aiE}}

\newcommand{\detX}{{\rm det_X}}

\newcommand{\Ncal}{\mathcal{N}}

\newcommand{\Bsc}{\mathscr{B}}
\newcommand{\Csc}{\mathscr{C}}

\papertype{Original Article}

\title{Estimation of Riemannian distances between covariance operators and Gaussian processes}
\author{H\`a Quang Minh}
\affil{RIKEN Center for Advanced Intelligence Project, Tokyo, JAPAN}
\corraddress{1-4-1 Nihonbashi, Tokyo, Japan}
\corremail{minh.haquang@riken.jp}
\fundinginfo{KAKENHI Grant Number JP20H04250}

\runningauthor{H\`a Quang Minh}

\begin{document}
	
\maketitle

\begin{abstract}
	In this work we study two Riemannian distances between infinite-dimensional positive definite Hilbert-Schmidt operators, namely  affine-invariant Riemannian and Log-Hilbert-Schmidt distances, in the context of covariance operators associated with functional stochastic processes, in particular Gaussian processes. Our first main results show that both distances converge in 
	the Hilbert-Schmidt norm. Using concentration results for Hilbert space-valued random variables,
	we then  show that both distances can be consistently and efficiently estimated from (i) sample covariance operators, (ii) finite, normalized covariance matrices, and (iii) finite samples
	generated by the given processes, all with dimension-independent convergence. Our theoretical analysis exploits extensively the methodology of reproducing kernel Hilbert space (RKHS) covariance and cross-covariance operators.
	The theoretical formulation is illustrated with numerical experiments on covariance operators of Gaussian processes.	
\end{abstract}


\section{Introduction}
\label{section:introduction}

This work studies two Riemannian distances, namely the affine-invariant Riemannian distance
\cite{Larotonda:2007} and Log-Hilbert-Schmidt distance \cite{MinhSB:NIPS2014} between centered Gaussian processes, and more generally, between covariance operators 
associated with functional stochastic processes.
Our main focus is on the estimation of these distances from finite samples generated by the given stochastic processes.
In both cases, we show that the distances can be consistently and efficiently estimated from finite samples,
with {\it dimension-independent} convergence rates.


The study of functional data has received increasing interests recently in statistics and machine learning, see e.g. \cite{Ramsay:2005functional,Ferraty:2006nonparametric,Horvath:2012functional}.
One particular approach for analyzing functional data has been via the analysis of covariance operators and the distance/divergence functions between them. 
Recent work along this direction includes \cite{Panaretos:jasa2010,Fremdt:2013testing}, which utilize the Hilbert-Schmidt distance between covariance operators and \cite{Pigoli:2014, Masarotto:2018Procrustes}, which utilize non-Euclidean distances, in particular the Procrustes distance, also known as Bures-Wasserstein distance.
The latter distance corresponds to
the $\Lcal^2$-Wasserstein distance between two centered Gaussian measures on Hilbert space in the context of optimal transport
\cite{Villani:2008Optimal}
and can better capture the intrinsic geometry of the set of covariance operators. 
In the context of covariance operators and Gaussian processes, the $\Lcal^2$-Wasserstein distance and its entropic regularization, the Sinkhorn divergence, have been analyzed in \cite{Minh:2021EntropicConvergenceGaussianMeasures,Minh2021:FiniteEntropicGaussian}. 
In \cite{Matthews2016sparseKL,Sun2019functionalKL}, the Kullback-Leibler divergence between stochastic processes was studied, the latter in the context of functional Bayesian neural networks.
In this work, we study the non-Euclidean distances between covariance operators
that arise from the Riemannian geometric viewpoint of positive definite Hilbert-Schmidt operators, including in particular the affine-invariant Riemannian and Log-Hilbert-Schmidt distances.


{\bf Contributions of this work}\footnote{An extended abstract summarizing several preliminary results of the current work, without proofs, in particular Theorems \ref{theorem:logHS-approx-finite-dim} and \ref{theorem:affineHS-approx-finite-dim}, was presented in the {\it Proceedings of the International Workshop on Functional and Operatorial Statistics (IWFOS 2020)}
	\cite{Minh:2020Functional}.} The following are the main novel contributions of the current work 
\begin{enumerate}
	\item We show that both the affine-invariant Riemannian distance \cite{Larotonda:2007} and
	Log-Hilbert-Schmidt distance \cite{MinhSB:NIPS2014} between positive definite Hilbert-Schmidt operators converge in the Hilbert-Schmidt norm.
	\item From the Hilbert-Schmidt norm convergence, we show that both the affine-invariant/Log-Hilbert-Schmidt distances between
	centered Gaussian processes/covariance operators can be consistently estimated using sample covariance operators.
	The convergence rate is {\it dimension-independent}.
	\item By representing the affine-invariant/Log-Hilbert-Schmidt distances between
	centered Gaussian processes/covariance operators via reproducing kernel Hilbert space (RKHS) covariance and cross-covariance operators, we show that they can be consistently estimated using (i) finite, normalized covariance matrices and (ii) finite samples generated by the corresponding random processes.
	The convergence rates in all cases are {\it dimension-independent}.
	\item We show the theoretical consistency of the empirical distances between Gaussian measures defined on an RKHS, 
	which are induced by a positive definite kernel \cite{MinhSB:NIPS2014}, as employed in computer vision applications.
\end{enumerate}

\section{Distances between Gaussian processes}
\label{section:distance-gaussian-processes}

We first review the correspondence between Gaussian processes and Gaussian measures/covariance operators
on Hilbert spaces, followed by the formal distance formulation. Throughout the paper, we assume the following 
\begin{enumerate}
	\item {\bf A1} $T$ is a $\sigma$-compact metric space, that is
	$T = \cup_{i=1}^{\infty}T_i$, where $T_1 \subset T_2 \subset \cdots$, with each $T_i$ being compact.
	
	\item {\bf A2} $\nu$ is a non-degenerate Borel probability measure on $T$, that is $\nu(B) > 0$
	for each open set $B \subset T$.
	\item {\bf A3} $K, K^1, K^2: T \times T \mapto \R$ are continuous, symmetric, positive definite kernels
	and $\exists \kappa > 0, \kappa_1 > 0, \kappa_2 > 0$ with
	\begin{align}
		\label{equation:Assumption-3}
		\int_{T}K(x,x)d\nu(x) \leq \kappa^2, \;\; \int_{T}K^i(x,x)d\nu(x) \leq \kappa_i^2.
	\end{align}
	\item {\bf A4} $\xi\sim \GP(0, K)$, $\xi^i \sim \GP(0, K^i)$, $i=1,2$, are {\it centered} Gaussian processes
	with covariance functions $K, K^i$, respectively, satisfying assumptions A1-A3.
\end{enumerate}
For $K$ satisfying assumption A3, 
let $\H_K$ denote the corresponding reproducing kernel Hilbert space (RKHS).
Let $K_x:T \mapto \R$ be defined by $K_x(t) = K(x,t)$. Assumption A3 implies in particular that
%
\begin{align}
	\int_{T}K(x,t)^2d\nu(t) < \infty \; \forall x \in T, \;\;\; \int_{T \times T}K(x,t)^2d\nu(x)d\nu(t) < \infty.
\end{align}
It follows that $K_x \in \Lcal^2(T,\nu)$ $\forall x \in T$, 
hence $\H_K \subset \Lcal^2(T,\nu)$ \cite{
	Sun2005MercerNoncompact}. Define the following linear operator
\begin{align}
	&R_K = R_{K,\nu}: \Lcal^2(T, \nu) \mapto \H_K,\;\;
	R_Kf = \int_{T}K_tf(t)d\nu(t), \;\;\; (R_Kf)(x) = \int_{T}K(x,t)f(t)d\nu(t).
	\label{equation:RK}
\end{align}
The operator $R_K$ is bounded, with
$||R_K:\Lcal^2(T,\nu) \mapto \H_K|| \leq \sqrt{\int_{T}K(t,t)d\nu(t)} \leq \kappa$.
Its adjoint is 
$R_K^{*}: \H_K \mapto \Lcal^2(T, \nu) = J:\H_K \inclusion \Lcal^2(T, \nu)$, the inclusion operator
from $\H_K$ into $\Lcal^2(T,\nu)$ \cite{Rosasco:IntegralOperatorsJMLR2010}. 
$R_K$ and $R_{K}^{*}$ together
induce 
the following self-adjoint, positive, trace class operator (e.g. \cite{CuckerSmale,Sun2005MercerNoncompact,Rosasco:IntegralOperatorsJMLR2010})
%
%
\begin{align}
	\label{equation:CK}
	&C_K = C_{K,\nu} = R_K^{*}R_K: \Lcal^2(T,\nu) \mapto \Lcal^2(T, \nu),
	\;\;
	(C_Kf)(x)  = \int_{T}K(x,t)f(t)d\nu(t), \;\; \forall f \in \Lcal^2(T,\nu),
	\\
	&\trace(C_K) = \int_{T}K(x,x)d\nu(x) \leq \kappa^2, \;\;||C_K||_{\HS(\Lcal^2(T,\nu))}^2  = \int_{T \times T}K(x,t)^2d\nu(x)d\nu(t) \leq \kappa^4.
\end{align}
Let $\{\lambda_k\}_{k \in \Nbb}$ be the eigenvalues
of $C_K$, 
with normalized eigenfunctions $\{\phi_k\}_{k\in \Nbb}$
forming an orthonormal basis in $\Lcal^2(T,\nu)$.
Mercer's Theorem (see version in \cite{Sun2005MercerNoncompact}) states that 
\begin{align}
	K(x,y) = \sum_{k=1}^{\infty}\lambda_k \phi_k(x)\phi_k(y)\;\;\;\forall (x,y) \in T \times T,
\end{align}
where the series converges absolutely for each pair $(x,y) \in T \times T$ and uniformly on any compact subset of $T$.
By Mercer's Theorem, $K$ is completely determined by $C_K$ and vice versa.

Consider now the correspondence between the trace class operator $C_K$ as defined in 
Eq.\eqref{equation:CK}
and Gaussian processes with paths in $\Lcal^2(T,\nu)$,
as
established 
in 
\cite{Rajput1972gaussianprocesses}.
Let $(\Omega, \Fcal, P)$ be a probability space, 
$\xi = (\xi(t))_{t \in T}
= (\xi(\omega,t))_{t \in T}$ be 
a real
Gaussian process on $(\Omega, \Fcal,P)$, with mean $m$ and covariance function $K$, denoted by $\xi\sim \GP(m,K)$, where 
\begin{align}
	m(t) = \bE{\xi(t)},\;\;K(s,t) = \bE[(\xi(s) - m(s))(\xi(t) - m(t))], \;\;s,t \in T.
\end{align}
The sample paths
$\xi(\omega,\cdot) \in \H = \Lcal^2(T, \nu)$ almost $P$-surely, i.e.
$\int_{T}\xi^2(\omega,t)d\nu(t) < \infty$ almost $P$-surely,
if and only if (\cite{Rajput1972gaussianprocesses}, Theorem 2 and Corollary 1)
\begin{align}
	\label{equation:condition-Gaussian-process-paths}
	\int_{T}m^2(t)d\nu(t) < \infty, \;\;\; \int_{T}K(t,t)d\nu(t) < \infty.
\end{align}
The condition for $K$ in Eq.\eqref{equation:condition-Gaussian-process-paths} is precisely assumption A3.
In this case, $\xi$ induces the following {\it Gaussian measure} $P_{\xi}$ on $(\H, \Bsc(\H))$:
$	P_{\xi}(B) = P\{\omega \in \Omega: \xi(\omega, \cdot) \in B\}, \; B \in \Bsc(\H)$,
with mean $m \in \H$ and covariance operator
$C_K: \H \mapto \H$, defined by Eq.\eqref{equation:CK}.
Conversely, let $\mu$ be a Gaussian measure on $(\H
, \Bsc(\H))$, then
there is a
Gaussian process $\xi = (\xi(t))_{t \in T}$
with sample paths in $\H$,
with induced probability measure $P_{\xi} = \mu$.

	{\bf Divergence between Gaussian processes}.
	Since Gaussian processes are fully determined by their means and covariance functions,
	the latter
	being fully determined by their covariance operators, we can define distance/divergence functions between two Gaussian processes as follows, see also e.g. \cite{Panaretos:jasa2010,Fremdt:2013testing,Pigoli:2014,
		masarotto2019procrustes}.
	Assume Assumptions A1-A4.
	Let $\H = \Lcal^2(T,\nu)$. Let $\Gauss(\H)$ denote the set of Gaussian measures on $\H$.
	Let $\xi^i \sim \GP(m_i,K_i)$, $i=1,2$,  be
	two Gaussian processes with mean $m_i \in \H$ and covariance function $K^i$.
	Let $D$ be a divergence function on $\Gauss(\H)\times \Gauss(\H)$. The corresponding divergence
	$D_{\GP}$ 
	between $\xi^1$ and $\xi^2$ is defined to be
	\begin{align}
		D_{\GP}(\xi^1|| \xi^2) = D(\Ncal(m_1,C_{K^1}) ||\Ncal(m_2, C_{K^2})).
	\end{align}
It is clear then that $D_{\GP}(\xi^1 ||\xi^2) \geq 0$ and by Mercer's Theorem
\begin{align}
D_{\GP}(\xi^1 ||\xi^2) = 0 \equivalent m_1 = m_2, C_{K^1} = C_{K^2} \equivalent m_1 = m_2, K^1 = K^2.
\end{align}
Subsequently, we assume $m_1  = m_2 = 0$ and focus on $D(\Ncal(0,C_{K^1}) ||\Ncal(0, C_{K^2}))$ with $D$ being a Riemannian distance.

\subsection{Background: Finite-dimensional distances}
\label{section:finite-distances}

In the finite-dimensional setting, 
many different distance and distance-like functions
between covariance matrices and Gaussian measures have been studied.
Specifically, let $A,B$ be two covariance matrices corresponding to two Borel probability measures in $\R^n$, then
$A,B \in \Sym^{+}(n)$, the set of $n \times n$ real, symmetric, positive semi-definite matrices.
Examples of distance functions that have been studied on $\Sym^{+}(n)$ include
\begin{enumerate}
	\item Euclidean (Frobenius) distance $d_E(A,B) = ||A-B||_F$, where $||\;||_F$ denotes the Frobenius norm.
	\item Square root distance \cite{Dryden:2009} $d_{1/2}(A,B) = ||A^{1/2} - B^{1/2}||_F = (\trace[A+B - 2(A^{1/2}B^{1/2})])^{1/2}$.
	\item Bures-Wasserstein distance (see e.g. \cite{Downson:1982, Givens:1984, Gelbrich:1990Wasserstein,Bhatia:2018Bures,Malago:WassersteinGaussian2018} $d_{\BW}(A,B) = (\trace[A+B - 2(B^{1/2}AB^{1/2})^{1/2}])^{1/2} = W_2(\Ncal(0,A), \Ncal(0,B))$,
	the $\Lcal^2$-Wasserstein distance between two zero-mean Gaussian probability measures in $\R^n$ with covariance matrices $A,B$.
	It coincides with the square root distance if and only if $A$ and $B$ commute.
\end{enumerate}

Consider now the set $\Sym^{++}(n)$ of $n \times n$ real, symmetric, positive definite (SPD) matrices.
Elements of this set include, for example, covariance matrices corresponding to Gaussian probability densities on $\R^n$.
The set $\Sym^{++}(n)$ is rich in intrinsic geometrical structures
and one common approach is to view it as a Riemannian manifold.
Examples of Riemannian metrics that have been studied on $\Sym^{++}(n)$ include

\begin{enumerate}
\item Affine-invariant Riemannian metric (see e.g. \cite{Pennec:IJCV2006,Bhatia:2007}), with the corresponding Riemannian distance
$d_{\aiE}(A,B) = ||\log(B^{-1/2}AB^{-1/2})||_F$, where $\log$ denotes the principal logarithm of $A$.
The
distance $d_{\aiE}(A,B)$ corresponds to 
the Fisher-Rao distance between two zero-mean Gaussian densities with covariance matrices $A,B$ in $\R^n$.
\item Log-Euclidean metric \cite{LogEuclidean:SIAM2007},
with the corresponding Riemannian distance
given by $d_{\logE}(A,B) = ||\log(A) - \log(B)||_F$.
\item When restricted on $\Sym^{++}(n)$, the Bures-Wasserstein distance is also the Riemannian distance
corresponding to a Riemannian metric \cite{Takatsu2011wasserstein}.
\end{enumerate}
{\bf Related generalizations}.
On $\Sym^{++}(n)$, the Frobenius, square root, and Log-Euclidean distances are all special cases of the power-Euclidean distances
\cite{Dryden:2009},
$d_{E,\alpha}(A,B) = \left\|\frac{A^{\alpha}- B^{\alpha}}{\alpha}\right\|, \alpha \in \R, \alpha \neq 0$,
with  
$\lim_{\alpha \approach 0}d_{E,\alpha}(A,B) = ||\log(A) - \log(B)||_F$.
Similarly, the Bures-Wasserstein and Log-Euclidean distances are special cases of the $\alpha$-Procrustes distances, which are Riemannian distances
corresponding to a family of Riemannian metrics on $\Sym^{++}(n)$ \cite{Minh:GSI2019}\cite{Minh:Alpha2019},
\begin{align}
d_{\proE}^{\alpha}(A,B) &= \frac{(\trace[A^{2\alpha} + B^{2\alpha} - 2(B^{\alpha}A^{2\alpha}B^{\alpha})^{1/2}]^{1/2}}{|\alpha|}, \alpha \in \R, \alpha \neq 0,
\lim_{\alpha \approach 0}d_{\proE}^{\alpha}(A,B) =||\log(A) - \log(B)||_F.
\nonumber
\end{align}
For a fixed
$\alpha \neq 0$, the $\alpha$-Procrustes and power-Euclidean distances coincides 
if and only if $A$ and $B$ commute.

{\bf Infinite-dimensional generalizations}. Consider now the setting of infinite-dimensional covariance operators.
The finite-dimensional Frobenius distance generalizes readily to the infinite-dimensional Hilbert-Schmidt distance
$||A-B||_{\HS}$, where $A,B$ are Hilbert-Schmidt operators. Similarly, 
the formulas for the square root and Bures-Wasserstein distances remain valid in the infinite-dimensional setting, 
where $A,B$ are positive trace class operators on a Hilbert space.
The situation is substantially different with the Log-Euclidean and affine-invariant Riemannian distances
(see also the discussion in \cite{Pigoli:2014}). This is due to the fact that a positive compact operator $A$ on a Hilbert space, 
such as a covariance operator, possesses a sequence of eigenvalues approaching zero, and hence both 
$A^{-1}$ and $\log(A)$ are unbounded. Thus the formulas
for the Log-Euclidean and affine-invariant Riemannian distances cannot be carried over directly  
to the covariance operator setting. Instead,
a proper infinite-dimensional generalization of the affine-invariant Riemannian and Log-Euclidean metrics on the set of SPD matrices 
have been proposed by using the concepts of {\it extended (unitized) Hilbert-Schmidt operators}, {\it positive definite (unitized) Hilbert-Schmidt operators}, and {\it extended Hilbert-Schmidt inner product and norm}
\cite{Larotonda:2007}.
We next discuss these concepts and show how they can be applied in the setting of covariance operators associated with random processes.

\section{Riemannian distances between positive definite Hilbert-Schmidt operators}
\label{section:infinite-distances}

We first discuss the concept of {\it positive definite (unitized)  Hilbert-Schmidt operators} on a Hilbert space \cite{Larotonda:2007}.
Specifically, let $\H,\H_1,\H_2$ be infinite-dimensional separable real Hilbert spaces.
Let $\Lcal(\H_1,\H_2)$ denote the set of bounded linear operators between $\H_1$ and $\H_2$.
For $\H_1 = \H_2 = \H$, we write $\Lcal(\H)$. 
The set of trace class operators on $\H$ is defined to be
$\Tr(\H) = \{A \in \Lcal(\H): ||A||_{\tr} = \sum_{k=1}^{\infty}\la e_k, (A^{*}A)^{1/2}e_k\ra < \infty\}$,
where $\{e_k\}_{k=1}^{\infty}$ is any orthonormal basis in $\H$ and
the trace norm
$||A||_{\tr}$ is independent of the choice of such basis. For $A \in \Tr(\H)$, the trace of $A$ is 
$\trace(A) = \sum_{k=1}^{\infty}\la e_k, A e_k\ra = \sum_{k=1}^{\infty}\lambda_k$, where $\{\lambda_k\}_{k=1}^{\infty}$ denote
the eigenvalues of $A$.
For two separable Hilbert spaces $\H_1,\H_2$,
the set of Hilbert-Schmidt operators between $\H_1$ and $\H_2$
is defined to be, see e.g. \cite{Kadison:1983},
$\HS(\H_1,\H_2) = \{A \in \Lcal(\H_1,\H_2): ||A||_{\HS}^2 = \trace(A^{*}A) = \sum_{k=1}^{\infty}||Ae_k||^2_{\H_2} < \infty\}$,
the Hilbert-Schmidt norm $||A||_{\HS}$ being independent of the choice of orthonormal basis $\{e_k\}_{k=1}^{\infty}$
in $\H_1$. For $\H_1 = \H_2 = \H$, we write $\HS(\H)$.
The set $\HS(\H_1,\H_2)$ is itself a Hilbert space with 
the Hilbert-Schmidt inner product $\la A,B\ra_{\HS} = \trace(A^{*}B) = \sum_{k=1}^{\infty}\la Ae_k, Be_k\ra_{\H_2}$.

The set of {\it extended (or unitized) Hilbert-Schmidt operators} on $\H$ is defined in \cite{Larotonda:2007} to be
\begin{align}
\HS_X(\H) = \{ A + \gamma I: A \in \HS(\H), \gamma \in \R\}.
\end{align}
This is a Hilbert space under the {\it extended Hilbert-Schmidt inner product} and {\it extended Hilbert-Schmidt norm}
\begin{align}
\la A+\gamma I, B + \nu I\ra_{\eHS} = \la A,B\ra_{\HS} + \gamma\nu,\;\;\;||A + \gamma I||^2_{\eHS} = ||A||^2_{\HS} + \gamma^2.
\end{align}
Under $\la ,\ra_{\eHS}$, the scalar operators $\gamma I$, $\gamma \in \R$, are orthogonal to the Hilbert-Schmidt operators.
With the norm $||\;||_{\eHS}$,
$||I||_{\eHS} = 1$, in contrast to the Hilbert-Schmidt norm, where $||I||_{\HS} = \infty$.

We recall that an operator $A \in \Lcal(\H)$ is said to be {\it positive definite} \cite{Petryshyn:1962} if there exists
a constant $M_A > 0$ such that $\la x, Ax\ra \geq M_A||x||^2$ $\forall x \in \H$.
This condition is equivalent to requiring that $A$ be both {\it strictly positive}, 
that is $\la x, Ax\ra > 0$ $\forall x \neq 0$, and {\it invertible}, with $A^{-1}\in \Lcal(\H)$.
Let $\bP(\H)$ be the set of {\it self-adjoint positive definite} bounded operators on $\H$.

{\bf Positive definite (unitized) Hilbert-Schmidt operators}. With the extended Hilbert-Schmidt operators, we define 
the set of {\it positive definite (unitized) Hilbert-Schmidt} operators on $\H$ to be
\begin{align}
\PC_2(\H) = \bP(\H) \cap \HS_X(\H) = \{A+ \gamma I > 0: A^{*} = A, \gamma \in \R\}.
\end{align}
This is a Hilbert manifold, being an open subset of the Hilbert space $\HS_X(\H)$.
On $\PC_2(\H)$, both $\log(A+\gamma I)$ and $(A+\gamma I)^{\alpha}$, $\alpha \in \R$, are well-defined and bounded. 

{\bf Affine-invariant Riemannian distance}.
The generalization of the {\it affine-invariant metric} on $\Sym^{++}(n)$ to the Hilbert manifold
$\PC_2(\H)$ was defined in
\cite{Larotonda:2007}, with the corresponding Riemannian distance given by
\begin{align}
d_{\aiHS}[(A+\gamma I), (B+ \nu I)] = ||\log[(B+\nu I)^{-1/2}(A+\gamma I)(B+\nu I)^{-1/2}]||_{\eHS}.
\end{align}

{\bf Log-Hilbert-Schmidt distance}.
Similarly, the generalization of the Log-Euclidean metric on $\Sym^{++}(n)$ to $\PC_2(\H)$ was defined in 
\cite{MinhSB:NIPS2014}, with the corresponding {\it Log-Hilbert-Schmidt distance} given by
\begin{align}
d_{\logHS}[(A+\gamma I), (B+ \nu I)] = ||\log(A+\gamma I) - \log(B+ \nu I)||_{\eHS}.
\end{align}
The definition of the extended Hilbert-Schmidt norm guarantees that
both $d_{\aiHS}[(A+\gamma I), (B+ \nu I)]$ and $d_{\logHS}[(A+\gamma I), (B+ \nu I)]$ are always well-defined and finite for any pair $(A+\gamma I), (B+\nu I)
\in \PC_2(\H)$.
In the setting of reproducing kernel Hilbert space (RKHS) covariance operators,
both the affine-invariant Riemannian and Log-Hilbert-Schmidt distances admit closed form formulas in terms
of the corresponding kernel Gram matrices \cite{MinhSB:NIPS2014}, \cite{Minh:GSI2015}.

{\bf Distances between positive Hilbert-Schmidt operators}.
In the case $\gamma = \nu > 0$ is fixed, both $d_{\aiHS}[(A+\gamma I), (B+\gamma I)]$ and 
$d_{\logHS}[(A+\gamma I), (B+\gamma I)]$ become distances on the set of self-adjoint, positive Hilbert-Schmidt operators on $\H$.
In the following, let $\Sym(\H) \subset \Lcal(\H)$ denote the set of self-adjoint, bounded operators 
and $\Sym^{+}(n) \subset \Sym(\H)$ the set of self-adjoint, positive, bounded operators on $\H$. We immediately have the following result.
\begin{theorem}
	\label{theorem:metric-positive-Hilbert-Schmidt}
Let $\gamma \in \R, \gamma > 0$ be fixed.
The distances $d_{\aiHS}[(A+\gamma I), (B+\gamma I)]$, $d_{\logHS}[(A+\gamma I), (B+ \gamma I)]$ are metrics on the set $\Sym^{+}(\H) \cap \HS(\H)$ of positive Hilbert-Schmidt operators on $\H$.
\end{theorem}

{\bf Related and further generalizations}. Similar to the extended Hilbert-Schmidt operators, we can define
the {\it extended trace class operators} \cite{Minh:LogDet2016}  to be $\Tr_X(\H) = \{A+\gamma I: A \in \Tr(\H), \gamma \in \R\}$ along with the {\it extended Fredholm determinant} $\detX(A+\gamma I)$ and subsequently the {\it extended Hilbert-Carleman determinant}
\cite{Minh:Positivity2020}. With these concepts, we obtained the {\it infinite-dimensional Alpha Log-Det divergences} \cite{Minh:LogDet2016}
and {\it Alpha--Beta Log-Det divergences} \cite{Minh:INGE2019} between positive definite (unitized) trace class operators and subsequently on the entire Hilbert manifold $\PC_2(\H)$
 \cite{Minh:Positivity2020}.  The Alpha-Beta Log-Det divergences form a highly general family of divergences on $\PC_2(\H)$ and include the affine-invariant Riemannian distance $d_{\aiHS}$ as a special case.
Closely related to the $\Lcal^2$-Wasserstein distance is its entropic regularization, the {\it Sinkhorn divergence}. For two Gaussian measures $\mu_i \sim \Ncal(m_i, C_i)$ on $\H$, $i=0,1$, it is given by \cite{Minh2020:EntropicHilbert}
\begin{align}
	\Srm^{\ep}_{2}(\mu_0, \mu_1) &= ||m_0 - m_1||^2 + \frac{\ep}{4}\trace\left[M^{\ep}_{00} - 2M^{\ep}_{01} + M^{\ep}_{11}\right] 
\label{equation:gauss-sinkhorn-infinite}
%
+ \frac{\ep}{4}\log\det\left[\frac{\left(I + \frac{1}{2}M^{\ep}_{01}\right)^2}{\left(I + \frac{1}{2}M^{\ep}_{00}\right)\left(I + \frac{1}{2}M^{\ep}_{11}\right)}\right], \;\;\;\ep > 0,
\end{align}
with $\lim_{\ep \approach 0}S^{\ep}_2(\mu_0, \mu_1) = W_2^2(\mu_0, \mu_1)$.
Here $M^{\ep}_{ij} = -I + \left(I + \frac{16}{\epsilon^2}C_i^{1/2}C_jC_i^{1/2}\right)^{1/2}$, $i,j=0,1$, and $\det$ denotes the Fredholm determinant.
The {\it $\alpha$-Procrustes distances} can also be generalized to the infinite-dimensional setting of $\PC_2(\H)$ and
include both the Bures-Wasserstein and Log-Hilbert-Schmidt distances as special cases \cite{Minh:GSI2019} \cite{Minh:Alpha2019},
\begin{align}
&d^{\alpha}_{\proHS}[(A+\gamma I), (B+ \gamma I)] 
= \frac{(\trace[(A+\gamma I)^{2\alpha} + (B+ \gamma I)^{2\alpha} -2 [(B+\gamma I)^{\alpha}(A+\gamma I)^{2\alpha}(B+\gamma I)^{\alpha}]^{1/2})])^{1/2}}{|\alpha|},\alpha \in \R, \alpha \neq 0,
\nonumber
\\
&\lim_{\alpha \approach 0}d^{\alpha}_{\proHS}[(A+\gamma I), (B+ \gamma I)] = ||\log(A+\gamma I) - \log(B+\gamma I)||_{\eHS}.
\end{align}
In particular, for $A,B \in \Sym^{+}(\H) \cap \Tr(\H)$ and $\alpha = 1/2$,
$\lim_{\gamma \approach 0}d^{1/2}_{\proHS}[(A+\gamma I), (B+ \gamma I)] 
=  2 (\trace[A+B - 2(B^{1/2}AB^{1/2})])^{1/2}$, which is twice
the Bures-Wasserstein distance.

\subsection{Finite-rank and finite-dimensional approximations}
\label{section:finite-rank-approx}

In practice, it is typically necessary to deal with {\it finite-rank} and/or {\it finite-dimensional approximations}
of infinite-dimensional distances. In the cases of $d_{\aiHS}$ and $d_{\logHS}$, finite-rank and finite-dimensional approximations are consequences of the following 
general convergence results, which are subsequently employed in the Gaussian process setting. 
We first note the decomposition $||\log(A+\gamma I) - \log(B+\nu I)||^2_{\eHS} = ||\log(\frac{A}{\gamma} + I) - \log(\frac{B}{\nu} + I)||_{\HS}^2 + (\log\frac{\gamma}{\nu})^2$
and similarly $||\log[(A+\gamma I)^{-1/2}(B+\nu I)(A+\gamma I)^{-1/2}]||^2_{\eHS} = ||\log[(\frac{A}{\gamma} + I)^{-1/2}(\frac{B}{\nu} + I)(\frac{A}{\gamma} + I)^{-1/2}]||_{\HS}^2 + (\log\frac{\gamma}{\nu})^2$, 
thus for $\gamma = \nu$
\begin{align}
	&||\log(A+\gamma I) - \log(B+\gamma I)||_{\eHS} = ||\log(A+\gamma I) - \log(B+\gamma I)||_{\HS} = \left\|\log\left(\frac{A}{\gamma} + I\right) - \log\left(\frac{B}{\gamma} + I\right)\right\|_{\HS},
	\\
	&||\log[(A+\gamma I)^{-1/2}(B+\gamma I)(A+\gamma I)^{-1/2}]||_{\eHS} = ||\log[(A+\gamma I)^{-1/2}(B+\gamma I)(A+\gamma I)^{-1/2}]||_{\HS}
	\nonumber
	\\
	&\quad =\left\|\log\left[\left(\frac{A}{\gamma} + I\right)^{-1/2}\left(\frac{B}{\gamma} + I\right)\left(\frac{A}{\gamma} + I\right)^{-1/2}\right]\right\|_{\HS}.
\end{align}
\begin{theorem}
	[\textbf{Convergence in Log-Hilbert-Schmidt distance}]
	\label{theorem:logHS-convergence}
	Let $\gamma \in \R, \gamma > 0$ be fixed.
	Let
	$A, \{A_n\}_{n \in \N} \in \Sym(\H) \cap \HS(\H)$.
	Assume that $(\gamma I+A) > 0, \gamma I+A_n > 0$ $\forall n \in \N$. Then
	$\log(\gamma I+A_n)$, $\log(\gamma I+A) \in \Sym(\H) \cap \HS_X(\H)$.

(i) If $A,A_n \in \Sym^{+}(\H) \cap \HS(\H)$ $\forall n \in \Nbb$, then
\begin{align}
	||\log(\gamma I+A_n) - \log(\gamma I+A)||_{\HS} \leq \frac{1}{\gamma}||A_n - A||_{\HS} \;\; \forall n \in \Nbb.
\end{align}
	
	(ii) In general, if $\lim_{n \approach \infty}||A_n - A|| = 0$, 
let $M_A >0$ be such that $\la x, (\gamma I+A)x\ra \geq M_A||x||^2$ $\forall x \in \H$.
For $0 < \epsilon < M_A$ fixed, let $N(\epsilon) \in \Nbb$ such that
$||A_n -A||<\epsilon$ $\forall n \geq N(\ep)$, then
\begin{align}
||\log(\gamma I +A_n) - \log(\gamma I+A)||_{\HS} \leq \frac{1}{M_A-\epsilon}||A_n - A||_{\HS} \;\;\forall n \geq N(\epsilon).
\end{align}
In both cases, $\lim_{n \approach \infty}||A_n - A||_{\HS} = 0$ implies
$\lim_{n \approach \infty}||\log(\gamma I+A_n) - \log(\gamma I+A)||_{\HS} = 0$.
\end{theorem} 

\begin{remark}
Scenario (i) in Theorem \ref{theorem:logHS-convergence} applies immediately to the case $A$, 
$A_n$ are covariance operators on $\H$. In general, the setting in (ii) is needed since we can have $I + A > 0$
without $A$ being positive.
For example, for $I+A > 0, I+B > 0$, as in Theorem \ref{theorem:affineHS-convergence} for the affine-invariant Riemannian distance,
we have $(I+B)^{-1/2}(I+A)(I+B)^{-1/2}  = I+C > 0$, where the operator $C = (I+B)^{-1/2}(A-B)(I+B)^{-1/2}$
can be positive or negative or indefinite.
\end{remark}

\begin{theorem}
	[\textbf{Approximation of Log-Hilbert-Schmidt distance}]
	\label{theorem:logHS-approx-sequence-0}
	Let $\gamma_i \in \R, \gamma_i > 0$, $i=1,2$ be fixed.
	Let
	$A,B, \{A_n\}_{n \in \N}$,
	$\{B_n\}_{n \in \N}$ 
	$\in \Sym(\H) \cap \HS(\H)$
	be such that $\gamma_1 I+A > 0$, $\gamma_2 I+B > 0$, $\gamma_1 I+A_n > 0, \gamma_2 I+B_n > 0 \forall n \in \Nbb$.
	
	(i) If $A_n, B_n , A,B \in \Sym^{+}(\H) \cap \HS(\H)$, then
	\begin{align}
	&\left|||\log(\gamma_1 I + A_n) - \log(\gamma_2 I +B_n)||_{\eHS} - ||\log(\gamma_1 I + A) - \log(\gamma_2 I +B)||_{\eHS}\right| 
	\nonumber
	\\
	&\leq \frac{1}{\gamma_1}||A_n -A||_{\HS} + \frac{1}{\gamma_2}||B_n - B||_{\HS}.
	\end{align}
(ii) In general case, assume that $\lim_{n \approach \infty}||A_n - A||_{\HS} = 0$, $\lim_{n \approach 0}||B_n - B||_{\HS} = 0$.
	Let $M_A, M_B  > 0$ be such that $\la x, (\gamma_1 I+A)x\ra \geq M_A||x||^2$, $\la x, (\gamma_2 I+B)x\geq M_B||x||^2$ $\forall x \in \H$. 
	Then $\forall 0 < \ep < \min\{M_A, M_B\}$, $\exists N(\ep) \in \Nbb$ such that $\forall n \geq N(\ep)$, $||A_n - A|| < \ep, ||B_n - B|| < \ep$, and
	\begin{align}
		&\left|||\log(\gamma_1 I + A_n) - \log(\gamma_2 I +B_n)||_{\eHS} - ||\log(\gamma_1 I + A) - \log(\gamma_2 I +B)||_{\eHS}\right| 
		\nonumber
		\\
		&\leq \frac{1}{M_A - \ep}||A_n -A||_{\HS} + \frac{1}{M_B-\ep}||B_n - B||_{\HS}.
	\end{align}
\end{theorem}

\begin{theorem}
	[\textbf{Convergence in affine-invariant Riemannian distance}]
	\label{theorem:affineHS-convergence}
	Let $\gamma \in \R, \gamma > 0$ be fixed.
Let
$A,\{A_n\}_{n \in \N}$
$\in \Sym(\H) \cap \HS(\H)$
be such that $\gamma I+A > 0$, $\gamma I+A_n > 0 \forall n \in \Nbb$, and $\lim_{n \approach \infty}||A_n - A||_{\HS} = 0$.
Let $M_A > 0$ be such that $\la x, (\gamma I+A)x\ra \geq M_A||x||^2$ $\forall x \in \H$. 
Then $\forall 0 < \ep < M_A$, $\exists N(\ep) \in \Nbb$ such that $\forall n \geq N(\ep)$, $||A_n - A|| < \ep$ and
\begin{align}
||\log[(\gamma I+A)^{-1/2}(\gamma I+A_n)(\gamma I+A)^{-1/2}]||_{\HS}\leq \frac{1}{M_A - \ep}||A_n - A||_{\HS}.
\end{align}
In particular, if $A \in \Sym^{+}(\H) \cap \HS(\H)$, then we can set $M_A = \gamma$ and consequently, $\forall 0 < \ep < \gamma$,
\begin{align}
	||\log[(\gamma I+A)^{-1/2}(\gamma I+A_n)(\gamma I+A)^{-1/2}]||_{\HS}\leq \frac{1}{\gamma - \ep}||A_n - A||_{\HS} \;\;\forall n \geq N(\ep).
\end{align}
\end{theorem}

\begin{theorem}
	[\textbf{Approximation of affine-invariant Riemannian distance}]
	\label{theorem:affineHS-approx-sequence-0}
	Let $\gamma_i \in \R, \gamma_i > 0$, $i=1,2$, be fixed.
	Let
	$A,B, \{A_n\}_{n \in \N}$,
	$\{B_n\}_{n \in \N}$ 
	$\in \Sym(\H) \cap \HS(\H)$
	be such that $\gamma_1 I+A > 0$, $\gamma_2 I+B > 0$, $\gamma_1 I+A_n > 0, \gamma_2 I+B_n > 0 \forall n \in \Nbb$, and $\lim_{n \approach \infty}||A_n - A||_{\HS} = 0$, $\lim_{n \approach 0}||B_n - B||_{\HS} = 0$.
	Let $M_A, M_B  > 0$ be such that $\la x, (\gamma_1 I+A)x\ra \geq M_A||x||^2$, $\la x, (\gamma_2 I+B)x\geq M_B||x||^2$ $\forall x \in \H$. 
	Then $\forall 0 < \ep < \min\{M_A, M_B\}$, $\exists N(\ep) \in \Nbb$ such that $\forall n \geq N(\ep)$, $||A_n - A|| < \ep, ||B_n - B|| < \ep$, and
	\begin{align}
		&\left|||\log[(\gamma_1 I+A_n)^{-1/2}(\gamma_2 I+B_n)(\gamma_1 I+A_n)^{-1/2}]||_{\eHS}
		-||[\log(\gamma_1 I+A)^{-1/2}(\gamma_2 I+B)(\gamma_1 I+A)^{-1/2}]||_{\eHS}\right|
		\nonumber
		\\
		&\leq \frac{1}{M_A - \ep}||A_n - A||_{\HS} + \frac{1}{M_B-\ep}||B_n -B||_{\HS}.
	\end{align}
	In particular, if $A,B \in \Sym^{+}(\H) \cap \HS(\H)$, then we can set $M_A = \gamma_1, M_B = \gamma_2$ and consequently, $\forall 0 < \ep < \min\{\gamma_1, \gamma_2\}$,
	\begin{align}
		&\left|||\log[(\gamma_1 I+A_n)^{-1/2}(\gamma_2 I+B_n)(\gamma I_1+A_n)^{-1/2}]||_{\eHS}
		-||[\log(\gamma I_1+A)^{-1/2}(\gamma_2 I+B)(\gamma_1 I+A)^{-1/2}]||_{\eHS}\right|
		\nonumber
		\\
		&\leq \frac{1}{\gamma_1- \ep}||A_n - A||_{\HS} + \frac{1}{\gamma_2-\ep}||B_n -B||_{\HS}.
	\end{align}
\end{theorem}



{\bf Finite-dimensional approximations via orthogonal projections}.
We now consider
the finite-dimensional approximations of 
$d_{\aiHS}$ and $d_{\logHS}$ via orthogonal projections.
Let $A \in \HS(\H)$.
Let $\{e_k\}_{k=1}^{\infty}$ be any orthonormal basis for $\H$. 
For any $f \in \H$, we have $f = \sum_{k=1}^{\infty}\la f, e_k\ra e_k$.
Let $N \in \N$ be fixed and consider the finite-dimensional subspace $\H_N = \myspan\{e_k\}_{k=1}^N$.
Consider next the projection operator $P_N = \sum_{k=1}^Ne_k \otimes e_k: \H \mapto \H_N$. For any $f\in \H$,
$P_Nf = \sum_{k=1}^N\la f, e_k\ra e_k$ and for the operator $P_NAP_N:\H \mapto \H$, 
\begin{align}
P_NAP_Nf
= P_N\sum_{k=1}^N\la f, e_k\ra Ae_k 
= \sum_{j=1}^N\left(\sum_{k=1}^N\la f, e_k\ra \la Ae_k,e_j\ra \right) e_j \in \H_N.
\end{align}
Thus $P_NAP_N$ is a finite rank operator, with rank at most $N$, and $\range(P_NAP_N) \subset \H_N$. In particular, $P_NAP_N|_{\H_N}: \H_N \mapto \H_N$ and for $f,g \in \H_N$, we have
\begin{align}
\label{equation:matrix-representation}
\la g, P_NAP_Nf\ra = \sum_{j,k=1}^N\la f,e_k\ra \la g, e_j\ra \la Ae_k, e_j\ra = \la \gbf, \Abf_N\fbf\ra_{\R^N},
\end{align}
where $\fbf = (\la f, e_k\ra)_{k=1}^N$, $\gbf = (\la g, e_k\ra)_{k=1}^N \in \R^N$ and $\Abf_N$ is the $N \times N$ matrix with $(\Abf_N)_{kj} = \la Ae_k, e_j\ra$.
Thus on $\H_N$ with basis $\{e_k\}_{k=1}^N$, the operator $P_NAP_N|_{\H_N}$ is represented by the matrix $\Abf_N$.
Furthermore, $A \in \Sym(\H) \imply P_NAP_N|_{\H_N} \in \Sym(\H_N)\imply \Abf \in \Sym(N)$ and $A \in \Sym^{+}(\H)
\imply P_NAP_N|_{\H_N} \in \Sym^{+}(\H_N) \imply \Abf \in \Sym^{+}(N)$.

Combining Theorems \ref{theorem:logHS-approx-sequence-0} and \ref{theorem:affineHS-approx-sequence-0}
with the finite-dimensional projection $P_N$,
we obtain the following results.
\begin{theorem}
	[\textbf{Finite-dimensional approximation of Log-Hilbert-Schmidt distance}]
\label{theorem:logHS-approx-finite-dim}	
	Assume that $(A+I), (B+I) \in \PC_2(\H)$. Let $A_N = P_NAP_N|_{\H_N}$ and $B = P_NBP_N|_{\H_N}$,
	with matrix representation $\Abf_N$ and $\Bbf_N$, in the basis $\{e_k\}_{k=1}^N$, respectively. Then
	\begin{align}
		\lim_{N \approach \infty}||\log(\Abf_N +I) - \log(\Bbf_N + I)||_F &=
	\lim_{N \approach \infty}||\log(A_N + I) - \log(B_N +I)||_{\HS} 
	=||\log(A+I) - \log(B+I)||_{\HS}.
	\end{align}
	Assume that $(A+\gamma I), (B+ \gamma I) \in \PC_2(\H)$, $\gamma \in \R, \gamma > 0$.
	 Then
	\begin{align}
		\lim_{N \approach \infty}||\log(\Abf_N + \gamma I) - \log(\Bbf_N + \gamma I)||_{F}
	&=\lim_{N \approach \infty}||\log(A_N + \gamma I) - \log(B_N + \gamma I)||_{\HS}
	\\
	& 
	= ||\log(A+ \gamma I) - \log(B+ \gamma I)||_{\HS}.
	\end{align}
\end{theorem}

\begin{theorem}
	[\textbf{Finite-dimensional approximation of Affine-invariant Riemannian distance}]
		\label{theorem:affineHS-approx-finite-dim}
	Assume that $(A+I), (B+I) \in \PC_2(\H)$. Let $A_N = P_NAP_N|_{\H_N}$ and $B = P_NBP_N|_{\H_N}$,
	with matrix representation $\Abf_N$ and $\Bbf_N$, in the basis $\{e_k\}_{k=1}^N$, respectively. 
	 Then
	\begin{align}
	\lim_{N \approach \infty}||\log[(\Bbf_N + I)^{-1/2}(\Abf_N + I)(\Bbf_N + I)^{-1/2}]||_F 
	&= \lim_{N \approach \infty}||\log[(B_N+I)^{-1/2}(A_N + I)(B_N +I)^{-1/2}]||_{\HS}
	\nonumber
	\\
	& = ||\log[(B+I)^{-1/2}(A+I)(B+I)^{-1/2}]||_{\HS}.
	\end{align}
	Assume that $(A+\gamma I), (B+ \gamma I) \in \PC_2(\H)$, $\gamma \in \R, \gamma > 0$. 
	Then
	\begin{align}
			\lim_{N \approach \infty}||\log[(\Bbf_N + \gamma I)^{-1/2}(\Abf_N + \gamma I)(\Bbf_N +\gamma  I)^{-1/2}]||_F 
	&=\lim_{N \approach \infty}||\log[(B_N+\gamma I)^{-1/2}(A_N + \gamma I)(B_N + \gamma I)^{-1/2}]||_{\HS}
	\nonumber
	\\
	& = ||\log[(B+ \gamma I)^{-1/2}(A+ \gamma I)(B+ \gamma I)^{-1/2}]||_{\HS}.
	\end{align}
\end{theorem}

\section{Estimation of distances between Gaussian processes}
\label{section:estimate-distance-Gaussian-processes}

 Let $\xi^i \sim \GP(0, K^i)$, $i=1,2$, be two Gaussian processes satisfying Assumptions A1-A4, with paths in $\Lcal^2(T,\nu)$. 
 The Log-Hilbert-Schmidt and affine-invariant Riemannian distances between $\xi^1$ and $\xi^2$ are defined via their corresponding centered Gaussian measures with covariance operators $C_{K^i}$, as follows 
 \begin{align}
&D^{\gamma}_{\logHS}(\xi^1||\xi^2) = D^{\gamma}_{\logHS}[\Ncal(0,C_{K^1}),\Ncal(0,C_{K^2})]
 	= ||\log(\gamma I + C_{K^1}) - \log(\gamma I + C_{K^2})||_{\HS(\Lcal^2(T,\nu))},
 	\\
&D^{\gamma}_{\aiHS}(\xi^1||\xi^2)=	D^{\gamma}_{\aiHS}[\Ncal(0,C_{K^1}),\Ncal(0,C_{K^2})]
 	= ||\log[(\gamma I + C_{K^1})^{-1/2}(\gamma I + C_{K^2})(\gamma I + C_{K^1})^{-1/2}]||_{\HS(\Lcal^2(T,\nu))}.
 \end{align}
In the following, we aim  to estimate
	$D^{\gamma}_{\logHS}[\Ncal(0, C_{K^1}), \Ncal(0, C_{K^2})]$ and
	$D^{\gamma}_{\aiHS}[\Ncal(0, C_{K^1}), \Ncal(0, C_{K^2})]$ given finite samples
	$\{\{\xi^1_i(x_j)\}_{i=1}^{N_1}, \{\xi^2_i(x_j)\}_{i=1}^{N_2}\}_{j=1}^m$
	from $\xi^1,\xi^2$ on 
	a set of points $\Xbf=(x_j)_{j=1}^m$ in $T$.
	These correspond to $N_i$ realizations of process $\xi^i$, $i=1,2$, sampled at the $m$ points in $T$ given by $\Xbf$.
	{\it For simplicity and without loss of generality, in the theoretical analysis we let $N_1 = N_2 = N$}.
We carry out the analysis in three scenarios: (i) using finite-rank sample covariance operators, (ii) using
finite covariance matrices, (iii) using finite samples, with the last being the most practical.

\subsection{Estimation of distances from sample covariance operators}
\label{section:estimate-distance-sample-cov-operators}
Consider the first scenario, where we have access to samples $\{\xi^i_j(t) = \xi^i(\omega_j, t)\}_{j=1}^N$, $i=1,2$, and the corresponding sample covariance operators.
For $\xi \sim \GP(0,K)$ on the probability space $(\Omega, \Fcal, P)$, define the rank-one operator $\xi(\omega,.) \otimes \xi(\omega,.) \in \Lcal(\Lcal^2(T,\nu))$
by 	$[\xi(\omega,.) \otimes \xi(\omega,.)]f(x) = \xi(\omega,x)\int_{T}\xi(\omega,t)f(t)d\nu(t)$, 
$\omega \in \Omega$, $f \in \Lcal^2(T,\nu)$. Then $[\xi(\omega,.) \otimes \xi(\omega,.)] \in \HS(\Lcal^2(T,\nu))$ $P$-almost surely, with
\begin{align}
	||\xi(\omega,.)\otimes \xi(\omega,.)||_{\HS(\Lcal^2(T,\nu))} &= \int_{T}\xi(\omega,t)^2d\nu(t) < \infty \;\;\text{$P$-almost surely}.
\end{align}
The covariance operator $C_K$ as defined in Eq. \eqref{equation:CK} can be expressed as 
\begin{align}
	C_K &= \bE[\xi \otimes \xi],\;\;
	C_Kf(x) = \bE\int_{T}\xi(\omega,x)\xi(\omega, t)f(t)d\nu(t)= \int_{T}K(x,t)f(t)d\nu(t).
\end{align}
Let $\Wbf = (\omega_j)_{j=1}^N$ be independently sampled from $(\Omega,P)$, corresponding to the samples
$\{\xi_j(t) = \xi(\omega_j,t)\}_{j=1}^N$ from $\xi$. It defines the {\it sample covariance function} $K_{\Wbf}$ and corresponding {\it sample covariance operator} $C_{K,\Wbf}$ by
\begin{align}
	K_{\Wbf}(x,y) &= \frac{1}{N}\sum_{i=1}^N\xi(\omega_i,x)\xi(\omega_i,y),	\;\;C_{K,\Wbf} = \frac{1}{N}\sum_{i=1}^N\xi(\omega_i, .) \otimes \xi(\omega_i, .),
	\\
	C_{K,\Wbf}f(x) &= \frac{1}{N}\sum_{i=1}^N\int_{T}\xi(\omega_i, x)\xi(\omega_i,t)f(t)d\nu(t) = \int_{T}K_{\Wbf}(x,t)f(t)d\nu(t).
\end{align}
For each fixed $\Wbf$,
$K_{\Wbf}$ is a positive definite kernel on $T \times T$. 
It is continuous if the sample paths $\xi(\omega, .)$ are continuous $P$-almost surely.
%
For $\H_{K_{\Wbf}} = \myspan\{\xi(\omega_j,.)\}_{j=1}^N\subset \Lcal^2(T,\nu)$,
$\dim(\H_{K_{\Wbf}}) \leq N$ and
thus $C_{K,\Wbf}$ has rank at most $N$.
The convergence of $C_{K,\Wbf}$ to $C_K$ is obtained given the following additional assumption, which implies A3

%
\begin{align}
\text{\bf (A5)}\;\; \int_{T}[K(x,x)]^2d\nu(x) \leq \kappa^4, \;\; \int_{T}[K^i(x,x)]^2d\nu(x) \leq \kappa_i^4, \; i=1,2.
\end{align}
\begin{proposition}
	[\cite{Minh2021:FiniteEntropicGaussian}]
	\label{proposition-concentration-sample-cov-operator}
	Assume Assumptions A1-A5. Let $\Wbf = (\omega_j)_{j=1}^N$ be independently sampled from $(\Omega,P)$. 
	For any $0 < \delta < 1$, with probability at least $1-\delta$,
	\begin{align}
		||C_{K,\Wbf}||_{\HS(\Lcal^2(T,\nu))} \leq \frac{2\kappa^2}{\delta},\;\;\;||C_{K,\Wbf} - C_K||_{\HS(\Lcal^2(T,\nu))} \leq \frac{2\sqrt{3}\kappa^2}{\sqrt{N}\delta}.
	\end{align}
\end{proposition}
Combining Proposition \ref{proposition-concentration-sample-cov-operator} with 
Theorem \ref{theorem:logHS-approx-sequence-0},
we obtain the following estimate for
$D^{\gamma}_{\logHS}[\Ncal(0,C_{K^1}),\Ncal(0,C_{K^2})]$.
\begin{theorem}
	[\textbf{Estimation of Log-Hilbert-Schmidt distance from sample covariance operators}]
	\label{theorem:logHS-convergence-sample-covariance-operator}
	Under Assumptions A1-A5, let $\Wbf^i = (\omega^i_j)_{j=1}^N$, $i=1,2$, be independently sampled from 
	$(\Omega_i, P_i)$.
	$\forall 0 < \delta < 1$, with probability at least $1-\delta$,
	\begin{align}
		\left|D^{\gamma}_{\logHS}[\Ncal(0,C_{K^1,\Wbf^1}), \Ncal(0,C_{K^2,\Wbf^2})] - D^{\gamma}_{\logHS}[\Ncal(0,C_{K^1}),\Ncal(0,C_{K^2})]\right|
		&\leq 
		\frac{4\sqrt{3}(\kappa_1^2+\kappa_2^2)}{\gamma \sqrt{N}\delta}.
	\end{align}
\end{theorem}
Combining Propositions \ref{proposition-concentration-sample-cov-operator}
with Theorem \ref{theorem:affineHS-approx-sequence-0},
we obtain the following estimate for 
$D^{\gamma}_{\aiHS}[\Ncal(0,C_{K^1}),\Ncal(0,C_{K^2})]$.
\begin{theorem}
[\textbf{Estimation of affine-invariant Riemannian distance from sample covariance operators}]
\label{theorem:aiHS-convergence-sample-covariance-operator}
Assume Assumptions A1-A5. Let $\gamma \in \R, \gamma > 0$ be fixed. 
Let $\Wbf^i = (\omega^i_j)_{j=1}^N$, $i=1,2$, be independently sampled from 
$(\Omega_i, P_i)$.
For $0 < \ep < \gamma$, $0 < \delta < 1$,
let $N(\ep)\in \Nbb$, $N(\ep) \geq 1+ \max\left\{\frac{48\kappa_1^4}{\ep^2\delta^2}, 
\frac{48\kappa_2^4}{\ep^2\delta^2}\right\}$, then $\forall N \geq N(\ep)$, with probability at least $1-\delta$,
\begin{align}
	&\left|D^{\gamma}_{\aiHS}[\Ncal(0,C_{K^1,\Wbf^1}), \Ncal(0,C_{K^2,\Wbf^2})] - D^{\gamma}_{\aiHS}[\Ncal(0,C_{K^1}),\Ncal(0,C_{K^2})]\right|
	 \leq \frac{4\sqrt{3}(\kappa_1^2 + \kappa_2^2)}{(\gamma - \ep)\sqrt{N}\delta}.
\end{align}
\end{theorem}


If $\kappa_1,\kappa_2$ are absolute constants, then the convergence in both Theorems \ref{theorem:logHS-convergence-sample-covariance-operator} and \ref{theorem:aiHS-convergence-sample-covariance-operator} is {\it dimension-independent}.

\subsection{Estimation of distances from finite covariance matrices}
\label{section:estimation-distance-covariance-matrices}
Consider the second scenario, where we have access to finite covariance matrices
associated with the covariance functions $K^1,K^2$.
Let $\Xbf = (x_i)_{i=1}^m$ be independently sampled from $(T,\nu)$. 
The Gaussian process assumption $\xi^i \sim \GP(0,K^i)$ means that $(\xi^i(., x_j))_{j=1}^m$ are $m$-dimensional Gaussian random variables,
with $(\xi^i(., x_j))_{j=1}^m \sim \Ncal(0, K^i[\Xbf])$, where
$(K^i[\Xbf])_{jk} = K^i(x_j, x_k)$, $1 \leq j,k \leq m$.
Assuming that the covariance matrices $K^i[\Xbf]$ are {\it known}.
Let $\gamma\in \R, \gamma > 0$ be fixed, $D^{\gamma}_{\logE}[\Ncal(0,A),\Ncal(0,B)] = d_{\logE}(A+\gamma I, B+\gamma I)$, $D^{\gamma}_{\aiE}[\Ncal(0,A),\Ncal(0,B)] = d_{\aiE}(A+\gamma I, B+\gamma I)$, we show that 
%
%
\begin{align}
	\label{equation:consistent-estimate-logE}
	D^{\gamma}_{\logE}\left[\Ncal\left(0, \frac{1}{m}K^1[\Xbf]\right), \Ncal\left(0, \frac{1}{m}K^2[\Xbf]\right)\right]
	\;\text{\it consistently estimates}\;
	D^{\gamma}_{\logHS}[\Ncal(0, C_{K^1}), \Ncal(0,C_{K^2})],
	\\
	D^{\gamma}_{\aiE}\left[\Ncal\left(0, \frac{1}{m}K^1[\Xbf]\right), \Ncal\left(0, \frac{1}{m}K^2[\Xbf]\right)\right]
	\;\text{\it consistently estimates}\;
	D^{\gamma}_{\aiHS}[\Ncal(0, C_{K^1}), \Ncal(0,C_{K^2})].
	\label{equation:consistent-estimate-aiE}
\end{align}
Since
$\frac{1}{m}K^i[\Xbf]:\R^m \mapto \R^m$ and $C_{K^i}: \Lcal^2(T,\nu) \mapto \Lcal^2(T,\nu)$
operate on two different Hilbert spaces, namely $\R^m$ and $\Lcal^2(T,\nu)$, 
we express the quantities in Eqs.\eqref{equation:consistent-estimate-logE},\eqref{equation:consistent-estimate-aiE} 
via {\it RKHS covariance and cross-covariance operators}
on the same RKHS induced by the kernels $K^i$'s. 
The convergence analysis is then carried out entirely via RKHS methodology.



{\bf RKHS covariance and cross-covariance operators}.
Let $K^1,K^2$ be two kernels satisfying Assumptions A1-A4, and $\H_{K^1}, \H_{K^2}$ the corresponding
RKHS. Let $R_{K^i}:\Lcal^2(T,\nu) \mapto \H_{K^i}$, $i=1,2$ be as defined in Eq.\eqref{equation:RK}.
Together, they define the following {\it RKHS cross-covariance operators}
%
%
\begin{align}
	\label{equation:R12-operator}
	R_{ij}&= R_{K^i}R_{K^j}^{*}: \H_{K^j} \mapto \H_{K^i},\;\;\;i,j=1,2,
	R_{ji} = R_{K^j}R_{K^i}^{*}: \H_{K^i}\mapto \H_{K^j} = R_{ij}^{*},
	\\
		R_{ij} & = \int_{T}(K^i_t \otimes K^j_t)d\nu(t),\;\;
		R_{ij}f = \int_{T}K^i_t \la f, K^j_t\ra_{\H_{K^j}}d\nu(t),  
		\\
		R_{ij}f(x) &= \int_{T}K^i_t(x)f(t)d\nu(t) = \int_{T}K^i(x,t)f(t)d\nu(t),\;\; f\in \H_{K^j}.
	\end{align}
In particular, $R_{ii} = L_{K^{i}}$, with the {\it RKHS covariance operator} $L_{K}$ defined by
\begin{align}
	\label{equation:LK}
	L_K &= R_KR_K^{*}: \H_K \mapto \H_K, \;\;L_K = \int_{T}(K_t \otimes K_t) d\nu(t),\;
	\\
	L_Kf(x) &= \int_{T}K_t(x)\la f, K_t\ra_{\H_K}d\nu(t) = \int_{T}K(x,t)f(t)d\nu(t),\;\; f\in \H_K.
\end{align}
$L_K$ has the same nonzero eigenvalues as $C_K$ and thus $L_K \in \Sym^{+}(\H_K) \cap \Tr(\H_K)$, with 
\begin{align}
	\trace(L_K) = \trace(C_K) \leq \kappa^2,\;\;||L_K||_{\HS(\H_K)} = ||C_K||_{\HS(\Lcal^2(T,\nu))} \leq \kappa^2.
\end{align}
\begin{lemma}
	[\cite{Minh2021:FiniteEntropicGaussian}]
	\label{lemma:R12-HS}
	Under Assumptions A1-A3,
	$R_{ij} \in \HS(\H_{K^j}, \H_{K^i})$, with
	$||R_{ij}||_{\HS(\H_{K^j}, \H_{K^i})} \leq \kappa_i\kappa_j$, $i,j=1,2$.
\end{lemma}

{\bf Empirical RKHS covariance and cross-covariance operators}. Let $\Xbf = (x_i)_{i=1}^m$ be 
independently sampled from $T$ according to $\nu$. 
It defines the following {\it sampling operator}
(see e.g. \cite{SmaleZhou2007})
\begin{align}
	&S_{\Xbf}: \H_{K} \mapto \R^m, \;\;\; S_{\Xbf}f = (f(x_i))_{i=1}^m = (\la f, K_{x_i}\ra)_{i=1}^m
	\;\;
	\text{with adjoint }
	&S_{\Xbf}^{*}: \R^m \mapto \H_{K},\;\;\; S_{\Xbf}^{*}\b = \sum_{i=1}^mb_iK_{x_i}.
\end{align}
The sampling operators $S_{i,\Xbf}:\H_{K^i}\mapto \R^m$, $i=1,2$, together define the following empirical version
of $R_{ij}$
\begin{align}
	R_{ij,\Xbf} & = \frac{1}{m}S_{i,\Xbf}^{*}S_{j, \Xbf} = \frac{1}{m}\sum_{k=1}^m(K^i_{x_k} \otimes K^j_{x_k}): \H_{K^j} \mapto \H_{K^i},
	\\
	\;\;
	R_{ij,\Xbf}f &= \frac{1}{m}\sum_{k=1}^mK^i_{x_k}\la f, K^j_{x_k}\ra_{\H_{K^j}} = \frac{1}{m}\sum_{k=1}^mf(x_k)K^i_{x_k}, \; f\in \H_{K^j}.
	%
\end{align}
In particular, $R_{ii,\Xbf} = L_{K^i,\Xbf}$, with the empirical RKHS covariance operator $L_{K,\Xbf}:\H_K \mapto \H_K$ defined by
\begin{align}
	&L_{K,\Xbf} = \frac{1}{m}S_{\Xbf}^{*}S_{\Xbf} =\frac{1}{m}\sum_{i=1}^m (K_{x_i} \otimes K_{x_i}):\H_K \mapto \H_K, \\
	&L_{K,\Xbf}f = \frac{1}{m}S_{\Xbf}^{*}(f(x_i))_{i=1}^m = \frac{1}{m}\sum_{i=1}^mf(x_i)K_{x_i} = \frac{1}{m}\sum_{i=1}^m\la f, K_{x_i}\ra_{\H_K}K_{x_i}.
\end{align}
Furthermore, the operator $S_{\Xbf}S_{\Xbf}^{*}:\R^m \mapto \R^m$ is given by
\begin{align}
	&S_{\Xbf}S_{\Xbf}^{*}:\R^m \mapto \R^m,\;\;	S_{\Xbf}S_{\Xbf}^{*}\b = S_{\Xbf}\sum_{i=1}^mb_iK_{x_i} = (\sum_{i=1}^mb_iK(x_i, x_1), \ldots, \sum_{i=1}^mb_iK(x_i, x_m)) = K[\Xbf]\b.
\end{align}
In particular,
the nonzero eigenvalues of $L_{K,\Xbf}$ are precisely those of $\frac{1}{m}K[\Xbf]$, corresponding to 
eigenvectors that must lie in $\H_{K,\Xbf}=\myspan\{K_{x_i}\}_{i=1}^m$.
Thus, the nonzero eigenvalues of $C_K:\Lcal^2(T, \nu) \mapto \Lcal^2(T,\nu)$, $\trace(C_K)$, $||C_K||_{\HS}$,
which are the same as those of $L_K:\H_K \mapto \H_K$, can be empirically estimated from those of the $m \times m$ matrix $\frac{1}{m}K[\Xbf]$ (see \cite{Rosasco:IntegralOperatorsJMLR2010}).
The representations of
$||\log(\gamma I+C_{K^1}) - \log(\gamma I+C_{K^2})||_{\HS(\Lcal^2(T,\nu))}$ and
$\left\|\log\left(\gamma I+\frac{1}{m}K^1[\Xbf]\right) - \log\left(\gamma I+\frac{1}{m}K^2[\Xbf]\right)\right\|_F$
in terms of RKHS covariance and cross-covariance operators and their empirical versions, respectively, are as follows.
\begin{proposition}
	[\textbf{Log-Hilbert-Schmidt distance via RKHS operators}]
	\label{proposition:logHS-Gram-matrices-LKX}
	Let $\gamma \in \R, \gamma > 0$ be fixed.  
	Assume A1-A4, then
	\begin{align}
		&||\log(\gamma I+C_{K^1}) - \log(\gamma I+C_{K^2})||^2_{\HS(\Lcal^2(T,\nu))} 
		\nonumber
		\\
		&=\left\|\log\left(I+\frac{1}{\gamma}L_{K^1}\right)\right\|_{\HS(\H_{K^1})}^2 + \left\|\log\left(I+\frac{1}{\gamma}L_{K^2}\right)\right\|_{\HS(\H_{K^2})}^2  - \frac{2}{\gamma^2}\trace\left[R_{12}^{*}h\left(\frac{1}{\gamma}L_{K^1}\right)R_{12}
		h\left(\frac{1}{\gamma}L_{K^2}\right)\right].
	\end{align}
Here $h(A) = A^{-1}\log(I+A)$ for $A$ compact, positive, as in Lemma \ref{lemma:hA-positive}, with $h(0) = I$.
	The corresponding empirical version is
	\begin{align}
		&\left\|\log\left(\gamma I+\frac{1}{m}K^1[\Xbf]\right) - \log\left(\gamma I+\frac{1}{m}K^2[\Xbf]\right)\right\|_F^2
		\nonumber
		\\
		& = \left\|\log\left(I+\frac{1}{\gamma}L_{K^1,\Xbf}\right)\right\|^2_{\HS(\H_{K^1})} + \left\|\log\left(I+\frac{1}{\gamma}L_{K^2,\Xbf}\right)\right\|^2_{\HS(\H_{K^2})} -\frac{2}{\gamma^2}\trace\left[R_{12,\Xbf}^{*}h\left(\frac{1}{\gamma}L_{K^1,\Xbf}\right)R_{12,\Xbf}h\left(\frac{1}{\gamma}L_{K^2,\Xbf}\right)\right].
	\end{align}
\end{proposition}
Similarly, 
$||\log[(\gamma I+C_{K^1})^{-1/2}(\gamma I+C_{K^2})(\gamma I+C_{K^1})^{-1/2}]||_{\HS(\Lcal^2(T,\nu))}$,
{\small$\left\|\log\left[\left(\gamma I+\frac{1}{m}K^1[\Xbf]\right)^{-1/2} \left(\gamma I+\frac{1}{m}K^2[\Xbf]\right)
\left(\gamma I+\frac{1}{m}K^1[\Xbf]\right)^{-1/2}\right]
\right\|_F$}
in terms of RKHS covariance and cross-covariance operators and their empirical versions, respectively, are as follows.
\begin{proposition}
	[\textbf{Affine-invariant distance via RKHS operators}]
	\label{proposition:aiHS-RKHS-representation}
	Let $\gamma \in \R, \gamma > 0$ be fixed.
	Under Assumptions A1-A4,
	\begin{align}
		&||\log[(\gamma I+C_{K^1})^{-1/2}(\gamma I+C_{K^2})(\gamma I+C_{K^1})^{-1/2}]||_{\HS(\Lcal^2(T,\nu))}^2
		\nonumber
		\\
		&=\trace\left[\log\left[I+\begin{pmatrix}
			(I+\frac{1}{\gamma}L_{K^1})^{-1}-I	 & \frac{1}{\gamma}(I+\frac{1}{\gamma}L_{K^1})^{-1}R_{12}\\
			-\frac{1}{\gamma}R_{12}^{*}(I+\frac{1}{\gamma}L_{K^1})^{-1}	 & \frac{1}{\gamma}L_{K^2} - \frac{1}{\gamma^2}R_{12}^{*}(I+\frac{1}{\gamma}L_{K^1})^{-1}R_{12}
		\end{pmatrix}
		\right]\right]^2.
	\end{align}
	The corresponding empirical version is
	\begin{align}
		&\left\|\log\left[\left(\gamma I+\frac{1}{m}K^1[\Xbf]\right)^{-1/2}\left(\gamma I+\frac{1}{m}K^2[\Xbf]\right)\left(\gamma I+ \frac{1}{m}K^1[\Xbf]\right)^{-1/2}\right]\right\|_{F}^2
		\nonumber
		\\
		&=\trace\left[\log\left[I+\begin{pmatrix}
			(I+\frac{1}{\gamma}L_{K^1,\Xbf})^{-1}-I	 & \frac{1}{\gamma}(I+\frac{1}{\gamma}L_{K^1,\Xbf})^{-1}R_{12,\Xbf}\\
			-\frac{1}{\gamma}R_{12,\Xbf}^{*}(I+\frac{1}{\gamma}L_{K^1,\Xbf})^{-1}	 & \frac{1}{\gamma}L_{K^2,\Xbf} - \frac{1}{\gamma^2}R_{12,\Xbf}^{*}(I+\frac{1}{\gamma}L_{K^1,\Xbf})^{-1}R_{12,\Xbf}
		\end{pmatrix}
		\right]\right]^2.
	\end{align}
\end{proposition}
For simplicity, to estimate the convergence of $R_{ij,\Xbf}$ towards $R_{ij}$, in the following we assume that
$K, K^1, K^2$ are {\it bounded} (the unbounded kernel case leads to looser convergence bounds, see \cite{Minh2021:FiniteEntropicGaussian}). Thus, assume 
$\exists \kappa, \kappa_1, \kappa_2 > 0$ such that
\begin{align}
\text{\bf (A6)}\;\;\;	\sup_{x \in T}K(x,x) \leq \kappa^2, \;\; \sup_{x \in T}K^i(x,x) \leq \kappa_i^2, i=1,2.
\end{align}
\begin{proposition}
	[\textbf{Convergence of RKHS empirical covariance and cross-covariance operators} \cite{Minh2021:FiniteEntropicGaussian}]
	\label{proposition:concentration-TK2K1-empirical}
	Under Assumptions A1-A6,
	$||R_{ij,\Xbf}||_{\HS(\H_{K^j}, \H_{K^i})} \leq \kappa_i \kappa_j$, $i,j=1,2$, $\forall \Xbf \in T^m$.
	Let $\Xbf = (x_i)_{i=1}^m$ be independently sampled from $(T,\nu)$.
	$\forall 0 < \delta < 1$, with probability at least $1-\delta$,
	\begin{align}
		||R_{ij,\Xbf} - R_{ij}||_{\HS(\H_{K^j}, \H_{K^i})} \leq \kappa_i\kappa_j\left[ \frac{2\log\frac{2}{\delta}}{m} + \sqrt{\frac{2\log\frac{2}{\delta}}{m}}\right].
	\end{align}	
	In particular, 
	$||L_{K^i,\Xbf}||_{\HS(\H_{K^i})} \leq \kappa_i^2$
	and 
	with probability at least $1-\delta$,
		$\left\|L_{K^i,\Xbf} - L_{K^i}\right\|_{\HS(\H_{K^i})} \leq \kappa_i^2\left(\frac{2\log\frac{2}{\delta}}{m} + \sqrt{\frac{2\log\frac{2}{\delta}}{m}}\right)$. 
\end{proposition}

Combining Propositions \ref{proposition:logHS-Gram-matrices-LKX} and \ref{proposition:concentration-TK2K1-empirical}, 
we obtain the following estimate of $D^{\gamma}_{\logHS}[\Ncal(0, C_{K^1}), \Ncal(0, C_{K^2})]$.
\begin{theorem}
	[\textbf{Estimation of Log-Hilbert-Schmidt distance from finite covariance matrices}]
	\label{theorem:logHS-approx-finite-covariance}
	Let $\gamma \in \R, \gamma > 0$ be fixed. Under Assumptions A1-A6,
	let $\Xbf = (x_i)_{i=1}^m$ be independently sampled from $(T,\nu)$.
	For any 
	$0 < \delta < 1$, with probability at least $1-\delta$,
	\begin{align}
		&\left|	\left\|\log\left(\gamma I+\frac{1}{m}K^1[\Xbf]\right) - \log\left(\gamma I+\frac{1}{m}K^2[\Xbf]\right)\right\|^2_F
		- ||\log(\gamma I+C_{K^1}) - \log(\gamma I+C_{K^2})||^2_{\HS(\Lcal^2(T,\nu))}\right|
		\nonumber
		\\
		& \leq \frac{2(\kappa_1^4 + \kappa_2^4)}{\gamma^2}
		\left(\frac{2\log\frac{6}{\delta}}{m} + \sqrt{\frac{2\log\frac{6}{\delta}}{m}}\right)
		+ \frac{2\kappa_1^2\kappa_2^2}{\gamma^2}\left(1 + \frac{\kappa_1^2+\kappa_2^2}{2\gamma}\right)\left(\frac{2\log\frac{24}{\delta}}{m} + \sqrt{\frac{2\log\frac{24}{\delta}}{m}}\right).
	\end{align}
\end{theorem}
Combining Propositions 	\ref{proposition:aiHS-RKHS-representation}
and \ref{proposition:concentration-TK2K1-empirical}, 
we obtain the following estimate of $D^{\gamma}_{\aiHS}[\Ncal(0, C_{K^1}), \Ncal(0, C_{K^2})]$.
\begin{theorem}
	[\textbf{Estimation of affine-invariant Riemannian distance from finite covariance matrices}]
	\label{theorem:aiHS-approx-finite-covariance}
	Let $\gamma \in \R, \gamma > 0$ be fixed.
	Under Assumptions 1-5, let $\Xbf = (x_j)_{j=1}^m$ be independently sampled from $(T,\nu)$.
	For any $0 < \delta < 1$, with probability at least $1-\delta$,
	\begin{align}
		&\left|\left\|\log\left[\left(\gamma I+\frac{1}{m}K^1[\Xbf]\right)^{-1/2}\left(\gamma I+\frac{1}{m}K^2[\Xbf]\right)\left(\gamma I+ \frac{1}{m}K^1[\Xbf]\right)^{-1/2}\right]\right\|^2_{F}\right.
		\nonumber
		\\
		&\left.-||\log[(\gamma I+C_{K^1})^{-1/2}(\gamma I+C_{K^2})(\gamma I+C_{K^1})^{-1/2}]||^2_{\HS(\Lcal^2(T,\nu))}\right|
		\nonumber
		\\
		&\quad \leq \frac{1}{\gamma^2}\left(1+\frac{\kappa_1^2}{\gamma}\right)^3\left[(\kappa_1 + \kappa_2)^2 + \frac{\kappa_1^2\kappa_2^2}{\gamma}\right]\left(\kappa_1 + \kappa_2 + \frac{\kappa_1^2\kappa_2}{\gamma}\right)^2\left[ \frac{2\log\frac{6}{\delta}}{m} + \sqrt{\frac{2\log\frac{6}{\delta}}{m}}\right].
	\end{align}
\end{theorem}
The convergence in Theorems \ref{theorem:logHS-approx-finite-covariance} and \ref{theorem:aiHS-approx-finite-covariance} is thus both {\it dimension-independent} if $\kappa_1,\kappa_2$ are absolute constants.

\subsection{Estimation of distances from finite samples}
\label{section:estimation-distance-finite-samples}
Consider now the most practical scenario, where
we only have access to
samples of the Gaussian processes $\xi^1,\xi^2$ on a finite set of points $\Xbf = (x_j)_{j=1}^m$
on $T$.
We can first estimate the covariance matrices $K^1[\Xbf], K^2[\Xbf]$, compute their distances, 
and apply Theorem \ref{theorem:logHS-approx-finite-covariance} from Section \ref{section:estimation-distance-covariance-matrices}.
For the Gaussian process $\xi = (\xi(\omega, t))$ defined on the probability space
$(\Omega, \Fcal, P)$, let $\Wbf = (\omega_1, \ldots, \omega_N)$ be independently sampled from $(\Omega,P)$, which corresponds to $N$ sample paths
$\xi_i(x) = \xi(\omega_i,x), 1 \leq i \leq N, x \in T$.
Let $\Xbf =(x_i)_{i=1}^m \in T^m$ be fixed. 
Consider the following $m \times N$ data matrix
\begin{align}
	\Zbf = \begin{pmatrix}
		\xi(\omega_1, x_1), \ldots, \xi(\omega_N, x_1),
		\\
		\cdots  
		\\
		\xi(\omega_1, x_m), \ldots, \xi(\omega_N, x_m)
	\end{pmatrix}
	= [\zbf(\omega_1), \ldots \zbf(\omega_N)] \in \R^{m \times N}.
\end{align}
Here
$\zbf(\omega) = (\z_i(\omega))_{i=1}^m = (\xi(\omega, x_i))_{i=1}^m \in \R^m$.
Since $(K[\Xbf])_{ij} = \bE[\xi(\omega,x_i)\xi(\omega, x_j)]$, $1\leq i,j\leq m$,
\begin{align}
	\label{equation:K-exact-W}
	K[\Xbf] = \bE[\zbf(\omega)\zbf(\omega)^T] = \int_{\Omega}\zbf(\omega)\zbf(\omega)^TdP(\omega).
\end{align}
The empirical version of $K[\Xbf]$, using the random sample $\Wbf = (\omega_i)_{i=1}^N$, is then
\begin{align}
	\label{equation:K-hat-W}
	\hat{K}_{\Wbf}[\Xbf] = \frac{1}{N}\sum_{i=1}^N \zbf(\omega_i)\zbf(\omega_i)^T = \frac{1}{N}\Zbf\Zbf^T.
\end{align}
The convergence of $\hat{K}_{\Wbf}[\Xbf]$ to $K[\Xbf]$ is given by the following.
\begin{proposition}
	[\cite{Minh2021:FiniteEntropicGaussian}]
	\label{proposition:concentration-empirical-covariance}
	Assume Assumptions A1-A6.
	Let $\xi \sim \GP(0,K)$ 
	on $(\Omega, \Fcal,P)$
	Let $\Xbf = (x_i)_{i=1}^m \in T^m$ be fixed. Then
	$||K[\Xbf]||_F \leq m\kappa^2$.
	Let $\Wbf = (\omega_1, \ldots, \omega_N)$ be independently sampled from $(\Omega,P)$. 
	For any $0 < \delta < 1$, with probability at least $1-\delta$, 
	\begin{align}
		&||\hat{K}_{\Wbf}[\Xbf] - K[\Xbf]||_F \leq \frac{2\sqrt{3}m\kappa^2}{\sqrt{N}\delta},
		\;\;
		||\hat{K}_{\Wbf}[\Xbf]||_F \leq \frac{2m\kappa^2}{\delta}.
	\end{align}
\end{proposition}
Let now $\xi^i \sim \GP(0, K^i)$, $i=1,2$, 
on the probability spaces $(\Omega_i, \Fcal_i, P_i)$, respectively.
Let $\Wbf^i = (\omega^i_j)_{j=1}^N$, be independently sampled from $(\Omega_i, P_i)$, corresponding to the sample paths
$\{\xi^i_j(t) = \xi^i(\omega_j,t)\}_{j=1}^N$, $t \in T$, from $\xi^i$, $i=1,2$.
Combining Proposition \ref{proposition:concentration-empirical-covariance} and
Theorem \ref{theorem:logHS-approx-sequence-0},
we obtain the following 
empirical estimate of 
$D^{\gamma}_{\logE}\left[\Ncal\left(0, \frac{1}{m}K^1[\Xbf]\right), \Ncal\left(0, \frac{1}{m}K^2[\Xbf]\right)\right]$
from two finite samples  of 
$\xi^1\sim \GP(0, K^1)$ and $\xi^2\sim \GP(0,K^2)$ given 
by $\Wbf^1, \Wbf^2$.
\begin{theorem}
	\label{theorem:logHS-estimate-unknown-1}
	Assume Assumptions A1-A6.
	Let $\Xbf = (x_i)_{i=1}^m \in T^m$, $m \in \Nbb$ be fixed.
	Let $\Wbf^1 = (\omega_j^1)_{j=1}^N$, $\Wbf^2 = (\omega_j^2)_{j=1}^N$ be independently sampled from
	$(\Omega_1, P_1)$ and $(\Omega_2, P_2)$, respectively.
	For any $0 < \delta<1$, with probability at least $1-\delta$,
		\begin{align}
			&\left|D^{\gamma}_{\logE}\left[\Ncal\left(0, \frac{1}{m}\hat{K}^1_{\Wbf^1}[\Xbf]\right), \Ncal\left(0, \frac{1}{m}\hat{K}^2_{\Wbf^2}[\Xbf]\right)\right]
			- D^{\gamma}_{\logE}\left[\Ncal\left(0, \frac{1}{m}K^1[\Xbf]\right), \Ncal\left(0, \frac{1}{m}K^2[\Xbf]\right)\right]
			\right|
			 \leq \frac{4\sqrt{3}(\kappa_1^2+\kappa_2^2)}{\gamma\sqrt{N}\delta}.
		\end{align}
	Here the probability is with respect to the product space $(\Omega_1,P_1)^N \times (\Omega_2,P_2)^N$.
\end{theorem}

Combing Theorems \ref{theorem:logHS-estimate-unknown-1} and \ref{theorem:logHS-approx-finite-covariance}, we are finally led to the following empirical estimate
of the theoretical Log-Hilbert-Schmidt distance $D^{\gamma}_{\logHS}[\Ncal(0,C_{K^1}), \Ncal(0,C_{K^2})]$
from two finite samples $\Zbf^1,\Zbf^2$ of $\xi^1 \sim \GP(0,K^1)$ and $\xi^2 \sim \GP(0,K^2)$.

\begin{theorem}
	[\textbf{Estimation of Log-Hilbert-Schmidt distance between Gaussian processes from finite samples}]
	\label{theorem:logHS-estimate-unknown-2}
	Assume Assumptions A1-A6.
	Let $\Xbf = (x_i)_{i=1}^m$
	be independently sampled from $(T, \nu)$.
	Let $\Wbf^1 = (\omega_j^1)_{j=1}^N$, $\Wbf^2 = (\omega_j^2)_{j=1}^N$ be independently sampled from
	$(\Omega_1, P_1)$ and $(\Omega_2, P_2)$, respectively.
	For any $0 < \delta<1$, with probability at least $1-\delta$,
		\begin{align}
			&\left|D^{\gamma}_{\logE}\left[\Ncal\left(0, \frac{1}{m}\hat{K}^1_{\Wbf^1}[\Xbf]\right), \Ncal\left(0, \frac{1}{m}\hat{K}^2_{\Wbf^2}[\Xbf]\right)\right] - D^{\gamma}_{\logHS}[\Ncal(0,C_{K^1}), \Ncal(0,C_{K^2})]\right|
			\nonumber
			\\
			& \leq \frac{8\sqrt{3}(\kappa_1^2+\kappa_2^2)}{\gamma\sqrt{N}\delta}
			+ \frac{1}{\gamma}\sqrt{2(\kappa_1^4 + \kappa_2^4)
				\left(\frac{2\log\frac{12}{\delta}}{m} + \sqrt{\frac{2\log\frac{12}{\delta}}{m}}\right)
				+ {2\kappa_1^2\kappa_2^2}\left(1 + \frac{\kappa_1^2+\kappa_2^2}{2\gamma}\right)\left(\frac{2\log\frac{48}{\delta}}{m} + \sqrt{\frac{2\log\frac{48}{\delta}}{m}}\right)}.
		\end{align}
	Here the probability is with respect to the space $(T,\nu)^m \times (\Omega_1,P_1)^N \times (\Omega_2,P_2)^N$.
\end{theorem}
Entirely similar results can be obtained for the affine-invariant Riemannian distance.
%

\section{Estimation of distances between RKHS Gaussian measures}
\label{section:distance-RKHS-covariance-operators}

We now consider the estimation of the distances between two RKHS Gaussian measures
induced by {\it two Borel probability measures} via {\it one positive definite kernel} on the same metric space.
This setting has been applied practically, see e.g. \cite{MinhSB:NIPS2014,Covariance:CVPR2014} for computer vision applications.

Throughout this section,
let {$\Xcal$} be a complete separable metric space.
Let 
{$K$} be a continuous positive definite kernel on {$\Xcal \times \Xcal$}. Then the reproducing kernel Hilbert space (RKHS) {$\H_K$} induced by {$K$} is separable (\cite{Steinwart:SVM2008}, Lemma 4.33).
Let {$\Phi: \Xcal \mapto \H_K$} be the corresponding canonical feature map, so that 
$K(x,y) = \la \Phi(x), \Phi(y)\ra_{\H_K}$ $\forall (x,y) \in \Xcal \times \Xcal$.
Let $\rho$ be a Borel probability measure on $\Xcal$ such that
\begin{align}
	\int_{\Xcal}||\Phi(x)||_{\H_K}^2d\rho(x) = \int_{\Xcal}K(x,x)d\rho(x) < \infty.
\end{align}
Then the following RKHS mean vector $\mu_{\Phi} \in \H_K$ and RKHS covariance operator {$C_{\Phi}:\H_K \mapto \H_K$} induced by the feature map $\Phi$ are both well-defined and are given by
\begin{align}
	\mu_{\Phi} &= \mu_{\Phi,\rho} = \int_{\Xcal}\Phi(x)d\rho(x) \in \H_K, \;\;\;
	C_{\Phi} = C_{\Phi,\rho} = \int_{\Xcal}(\Phi(x)-\mu_{\Phi})\otimes (\Phi(x)-\mu_{\Phi})d\rho(x).
\end{align}
Let {$\Xbf =[x_1, \ldots, x_m]$,$m \in \N$,} be a data matrix randomly sampled from {$\Xcal$} according to
 $\rho$, where {$m \in \Nbb$} is the number of observations.
The feature map {$\Phi$} on {$\Xbf$} 
defines
the bounded linear operator
$\Phi(\Xbf): \R^m \mapto \H_K, \Phi(\Xbf)\b = \sum_{j=1}^mb_j\Phi(x_j) , \b \in \R^m$.
The corresponding empirical mean vector and covariance operator for {$\Phi(\Xbf)$}
are defined to be
\begin{align}
	\mu_{\Phi(\Xbf)} &= \frac{1}{m}\sum_{j=1}^m\Phi(x_j) = \frac{1}{m}\Phi(\Xbf)\1_m,
	\;\;\;
	C_{\Phi(\Xbf)} = \frac{1}{m}\Phi(\Xbf)J_m\Phi(\Xbf)^{*}: \H_K \mapto \H_K,
	\label{equation:covariance-operator}
\end{align}
where $J_m = I_m -\frac{1}{m}\1_m\1_m^T,\1_m = (1, \ldots, 1)^T \in \R^m$, is the centering matrix
with $J_m^2 = J_m$.
The convergence of $\mu_{\Phi(\Xbf)}$ and $C_{\Phi(\Xbf)}$ towards $\mu_{\Phi}$ and $C_{\Phi}$, respectively, is quantified by the following
\begin{theorem}
	[\textbf{Convergence of RKHS mean and covariance operators - bounded kernels \cite{Minh:2021EntropicConvergenceGaussianMeasures}}]
	\label{theorem:CPhi-concentration}
	Assume that $\sup_{x \in \Xcal}K(x,x)\leq \kappa^2$.
	Let $\Xbf = (x_i)_{i=1}^m$, $m \in \Nbb$, be independently  sampled from $(\Xcal, \rho)$. Then
	$||\mu_{\Phi}||_{\H_K} \leq \kappa$, $||\mu_{\Phi(\Xbf)}||_{\H_K} \leq \kappa$ $\forall \Xbf \in \Xcal^m$,
	$||C_{\Phi}||_{\HS(\H_K)}|| \leq 2\kappa^2$,
	$||C_{\Phi(\Xbf)}||_{\HS(\H_K)}  \leq 2\kappa^2$ $\forall \Xbf \in \Xcal^m$.
	For any $0 < \delta <1$, with probability at least $1-\delta$, 
	\begin{align}
		||\mu_{\Phi(\Xbf)} - \mu_{\Phi}||_{\H_K} & \leq \kappa\left(\frac{2\log\frac{4}{\delta}}{m} + \sqrt{\frac{2\log\frac{4}{\delta}}{m}}\right),\;\;
		||C_{\Phi(\Xbf)} - C_{\Phi}||_{\HS(\H_K)} \leq 3\kappa^2\left(\frac{2\log\frac{4}{\delta}}{m} + \sqrt{\frac{2\log\frac{4}{\delta}}{m}}\right).
	\end{align}
\end{theorem}
Combining Theorem \ref{theorem:CPhi-concentration} with Theorems \ref{theorem:logHS-convergence} and \ref{theorem:affineHS-convergence}, we first obtain the following convergence
of the empirical Gaussian measure $\Ncal(0, C_{\Phi(\Xbf)})$ towards the Gaussian measure $\Ncal(0, C_{\Phi})$, which are defined on the RKHS $\H_K$.
\begin{theorem}
\label{theorem:convergence-RKHS-Gaussian}
Let $\gamma \in \R, \gamma > 0$ be fixed.
		Assume that $\sup\limits_{x \in \Xcal}K(x,x)\leq \kappa^2$.
	For any $0 < \delta < 1$, with probability at least $1-\delta$,
\begin{align}
D^{\gamma}_{\logHS}[\Ncal(0,C_{\Phi(\Xbf)}), \Ncal(0, C_{\Phi})] \leq
\frac{3\kappa^2}{\gamma}\left(\frac{2\log\frac{4}{\delta}}{m} + \sqrt{\frac{2\log\frac{4}{\delta}}{m}}\right).
\end{align}
For $0 < \ep < \gamma$,
let $N(\ep) \in \Nbb$ be such that
${3\kappa^2}\left(\frac{2\log\frac{4}{\delta}}{N(\ep)} + \sqrt{\frac{2\log\frac{4}{\delta}}{N(\ep)}}\right) < \ep$,
then $\forall m \geq N(\ep)$, with probability at least $1-\delta$,
\begin{align}
D^{\gamma}_{\aiHS}[\Ncal(0,C_{\Phi(\Xbf)}), \Ncal(0, C_{\Phi})] \leq
\frac{3\kappa^2}{\gamma-\ep}\left(\frac{2\log\frac{4}{\delta}}{m} + \sqrt{\frac{2\log\frac{4}{\delta}}{m}}\right).
\end{align}
\end{theorem}

Let now {$\Xbf^1 = [x_i^1]_{i=1}^m$, $\Xbf^2 = [x_i^2]_{i=1}^m$}, be two random data matrices sampled from {$\Xcal$} according to two Borel probability distributions $\rho_1$ and $\rho_2$ on $\Xcal$ . Let $\mu_{\Phi(\Xbf^1)}, \mu_{\Phi(\Xbf^2)}$ and $C_{\Phi(\Xbf^1)}$, $C_{\Phi(\Xbf^2)}$
be the corresponding mean vectors and covariance operators induced by
the kernel 
$K$, respectively.
Define the following $m \times m$ Gram matrices
\begin{align}
	&K[\Xbf^1] = \Phi(\Xbf^1)^{*}\Phi(\Xbf^1),\;K[\Xbf^2] = \Phi(\Xbf^2)^{*}\Phi(\Xbf^2), K[\Xbf^1,\Xbf^2] = \Phi(\Xbf^1)^{*}\Phi(\Xbf^2),
	\\
	&(K[\Xbf^1])_{jk} = K(x_j^1,x_k^1), (K[\Xbf^2])_{jk} = K(x_j^2,x_k^2),(K[\Xbf^1,\Xbf^2])_{jk} = K(x^1_j, x^2_k), \; 1\leq j,k \leq m.
\end{align}
For fixed $\gamma_i \in \R, \gamma_i > 0$, $i=1,2$, the following distances are expressed explicitly in terms of the Gram matrices (\cite{MinhSB:NIPS2014,Minh:2019AlphaBeta})
\begin{align}
	\label{equation:logHS-Gram}
&||\log(C_{\Phi(\Xbf^1)} + \gamma_1 I_{\H_K}) - \log(C_{\Phi(\Xbf^2)} + \gamma_2 I_{\H_K})||_{\eHS},\;\;\;
\\
&||\log[(C_{\Phi(\Xbf^1)} + \gamma_1 I_{\H_K})^{-1/2}(C_{\Phi(\Xbf^2)} + \gamma_2 I_{\H_K})^{-1/2}(C_{\Phi(\Xbf^1)} + \gamma_1 I_{\H_K})^{-1/2}]||_{\eHS}.
\label{equation:aiHS-Gram}
\end{align}
The following shows theoretical consistency for the empirical quantities in Eqs.\eqref{equation:logHS-Gram}, \eqref{equation:aiHS-Gram}, which are used practically.
\begin{theorem}
	\label{theorem:approx-RKHS-Gaussian}
	Let $\gamma_i \in \R, \gamma_i > 0$, $i=1,2$, be fixed.
		Assume that $\sup\limits_{x \in \Xcal}K(x,x)\leq \kappa^2$.
	For any 
	$0 < \delta < 1$, with probability at least $1-\delta$,
\begin{align}
&\left|||\log(C_{\Phi(\Xbf^1)} + \gamma_1 I) - \log(C_{\Phi(\Xbf^2)} + \gamma_2 I)||_{\eHS}
- ||\log(C_{\Phi, \rho_1} + \gamma_1I) - \log(C_{\Phi,\rho_2} + \gamma_2 I)||_{\eHS}\right|
\nonumber
\\
& 
\quad \leq 3\kappa^2\left(\frac{1}{\gamma_1} + \frac{1}{\gamma_2}\right)\left(\frac{2\log\frac{8}{\delta}}{m} + \sqrt{\frac{2\log\frac{8}{\delta}}{m}}\right).
\end{align}
Here the probability is with respect to the space $(\rho_1 \times \rho_2)^m$.
For $0 < \ep < \min\{\gamma_1,\gamma_2\}$,
let $N(\ep) \in \Nbb$ be such that
${3\kappa^2}\left(\frac{2\log\frac{8}{\delta}}{N(\ep)} + \sqrt{\frac{2\log\frac{8}{\delta}}{N(\ep)}}\right) < \ep$,
then $\forall m \geq N(\ep)$, with probability at least $1-\delta$,
\begin{align}
&\left|||\log[(C_{\Phi(\Xbf^1)} + \gamma_1 I_{\H_K})^{-1/2}(C_{\Phi(\Xbf^2)} + \gamma_2 I_{\H_K})^{-1/2}(C_{\Phi(\Xbf^1)} + \gamma_1 I_{\H_K})^{-1/2}]||_{\eHS}\right.
\nonumber
\\
&\quad \left.- ||\log[(C_{\Phi, \rho_1} + \gamma_1 I_{\H_K})^{-1/2}(C_{\Phi, \rho_2} + \gamma_2 I_{\H_K})^{-1/2}(C_{\Phi, \rho_1} + \gamma_1 I_{\H_K})^{-1/2}]||_{\eHS}
\right|
\nonumber
\\
& \leq 3\kappa^2\left(\frac{1}{\gamma_1-\ep} + \frac{1}{\gamma_2-\ep}\right)\left(\frac{2\log\frac{8}{\delta}}{m} + \sqrt{\frac{2\log\frac{8}{\delta}}{m}}\right).
\end{align}
\end{theorem}



\section{Numerical experiments on Gaussian processes}
\label{section:experiment}

In this section, we illustrate the theoretical results above with several experiments on Gaussian processes.

{\bf Estimation of distances}.
In Figures \ref{figure:Gaussian-process-sample-1} and \ref{figure:Gaussian-process-sample-2}, we illustrate the convergence behavior studied in Section \ref{section:estimation-distance-finite-samples}. Here we show
the estimation of the Log-Hilbert-Schmidt and affine-invariant Riemannian distances between 
covariance operators of two Gaussian processes $\GP(0,K^1)$, $\GP(K^2)$, on $T = [0,1]$, with increasing number of sample paths
$N =10, 20, \ldots, 1000$. In Figure \ref{figure:Gaussian-process-sample-1}, $K^1 = \exp(-a||x-y||), K^2(x,y) = \exp\left(-\frac{||x-y||}{\sigma^2}\right)$, where $a = 1$ and $\sigma = 0.1$.
In Figure \ref{figure:Gaussian-process-sample-2}, $K^i(x,y) = \exp(-a_i||x-y||)$, where $a_1 =1$, $a_2 = 1.2$.
In both cases, the set $\Xbf = (x_j)_{j=1}^m$ is chosen randomly from $T$ using the uniform distribution, with $m = 500$. The regularization parameter is fixed at $\gamma = 10^{-9}$.

\begin{figure}
	\begin{subfigure}{0.3\textwidth}
		\includegraphics[width=\textwidth]{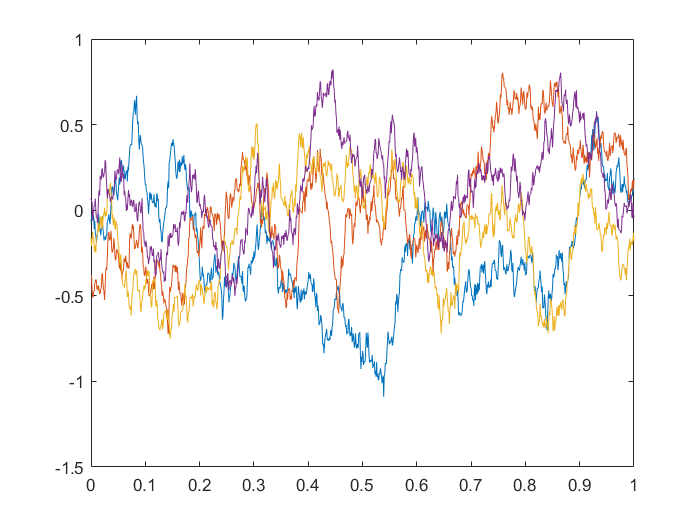}
	\end{subfigure}
	\begin{subfigure}{0.3\textwidth}
		\includegraphics[width=\textwidth]{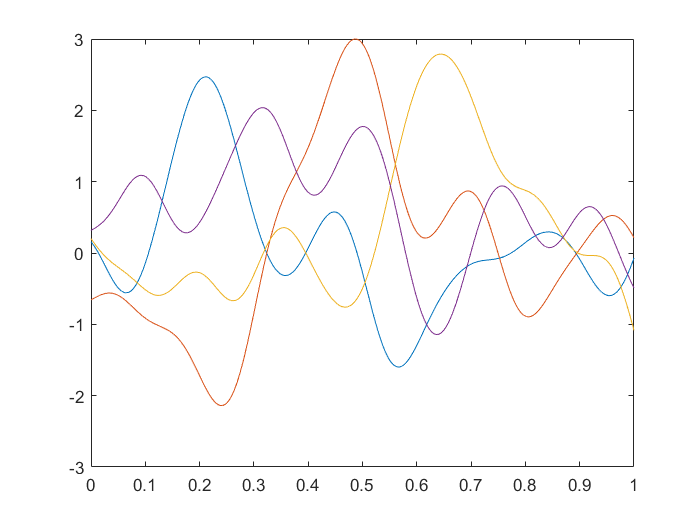}
	\end{subfigure}
	\begin{subfigure}{0.3\textwidth}
		\includegraphics[width =\textwidth]{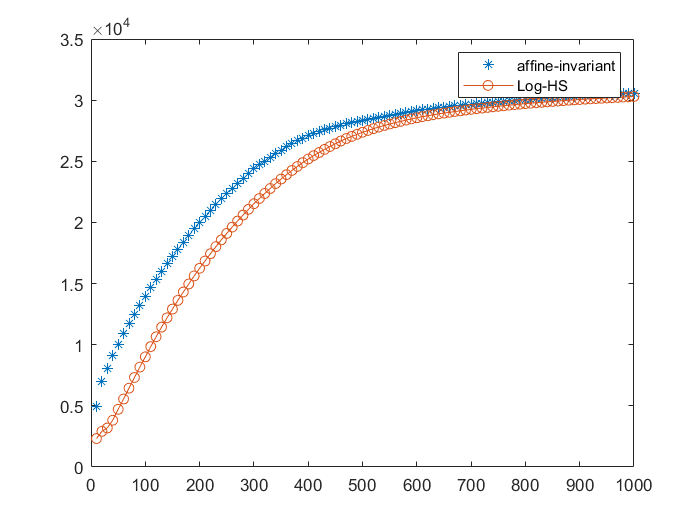}
	\end{subfigure}
	\caption{Samples of the centered Gaussian processes {$\GP(0,K^1)$, $\GP(0,K^2)$} 
		on {$T = [0,1]$} and approximations of squared distances between them.
		Left: {$K^1(x,y) = \exp(-a||x-y||)$, $a=1$}. Middle: {$K^2(x,y) = \exp(-||x-y||^2/\sigma^2)$, $\sigma = 0.1$}.
		Here the number of sample paths is {$N =10,20,\ldots, 1000$}, number of sample points is $m = 500$, and regularization parameter is {$\gamma = 10^{-9}$}}
		\label{figure:Gaussian-process-sample-1}
\end{figure}

	\begin{figure}
	\begin{subfigure}{0.3\textwidth}
		\includegraphics[width=\textwidth]{LaplacianKernel_GaussianProcess.png}
	\end{subfigure}
	\begin{subfigure}{0.3\textwidth}
		\includegraphics[width=\textwidth]{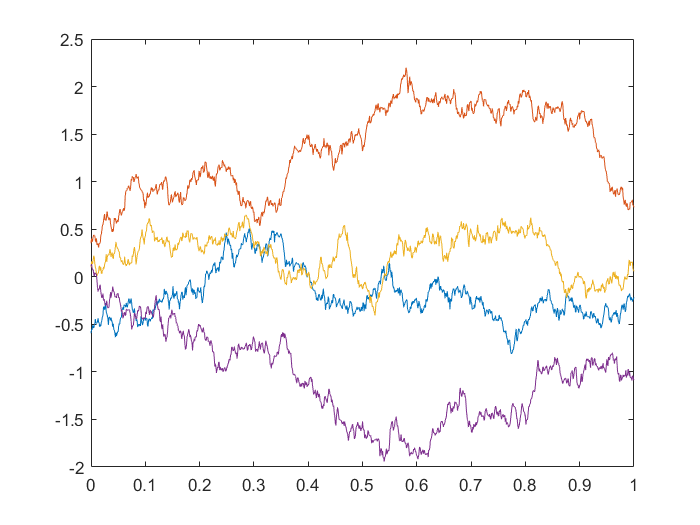}
	\end{subfigure}
	\begin{subfigure}{0.3\textwidth}
		\includegraphics[width =\textwidth]{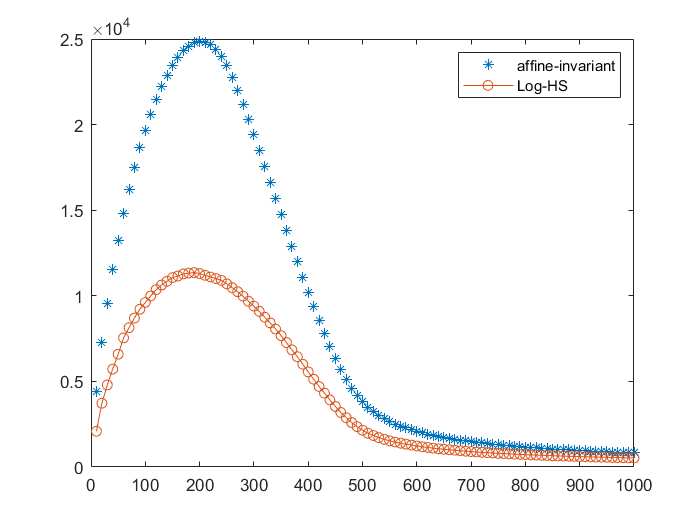}
	\end{subfigure}
	\caption{Samples of the centered Gaussian processes {$\GP(0,K^1)$, $\GP(0,K^2)$} 
		on {$T = [0,1]$} and approximations of squared distances between them.
		Left: {$K^1(x,y) = \exp(-a||x-y||)$, $a=1$}. Middle: {$K^2(x,y) = \exp(-a||x-y||)$, $a = 1.2$}.
		Here the number of sample paths is {$N =10,20,\ldots, 1000$},  number of sample points is $m = 500$, and regularization parameter is {$\gamma = 10^{-9}$}}
	\label{figure:Gaussian-process-sample-2}
\end{figure}

{\bf Classification of covariance operators}.
We carry out the following binary classification of covariance operators corresponding to two centered Gauss-Markov processes
$\GP(0,K^i)$ with
$K^i(x,y) = \exp(-\sigma_i||x-y||)$, $\sigma_i > 0$, $i=1,2$, on 
$T=[0,1]^d$ for $d=1,5$.
For each process, we generated a set of empirical covariance matrices, each defined by a set of finite samples as in Eq.\eqref{equation:K-hat-W}, using $N = 500$ sample paths on $m = 200$ points
randomly chosen by the uniform distribution on $T$.
The training and testing sets contain $10$ and $100$ empirical covariance matrices, respectively, 
split equally between the two classes. For classification, we utilized
the nearest neighbor approach. 
For the affine-invariant Riemannian and Log-Hilbert-Schmidt distances, we fixed $\gamma = 10^{-9}$. 
The experiments are repeated 5 times.

We report the average classification errors on the test set, along with standard deviations,  in four different scenarios
in Table \ref{table:binary-classify}, with examples of confusion matrices in Figure
\ref{figure:confusion-matrices}.  For the setting ($\sigma_1 = 1,\sigma_2 = 1.3$), the two Gaussian processes are easily distinguishable and perfect classification is achieved in almost all cases.
For the case ($\sigma_1 = 1,\sigma_2 = 1.1$), the two Gaussian processes are clearly much closer to each other
and the distances performed differently. The Log-Hilbert-Schmidt and affine-invariant Riemannian distances
perform consistently across different scenarios, with almost perfect classification in all settings.
The Hilbert-Schmidt distance, which does not take into account the geometrical structures of covariance matrices/operators, incurs considerable error in this case.
The Bures-Wasserstein distance, which does not possess dimension-independent convergence,
performs much worse on average in the case $d=5$ compared to $d=1$.

{\small
	\begin{table}[!t]		
		\caption{Classification errors on the test set ($\sigma_1 = 1$ in all cases)}
		\label{table:binary-classify}	
		\begin{tabular}{|c|c|c|c|c|}
			\hline
			Distance & ($\sigma_2 = 1.1$,  $d=1$) &($\sigma_2 = 1.3$, $d= 1$) &($\sigma_2 = 1.1$, $d=5$) & ($\sigma_2 = 1.3$, $d=5$)
			\\
			\hline & & & &
			\\ 
			{\it Hilbert-Schmidt} &    $33\% (7.97\%)$ & $4.60\% (2.51\%)$ & $15.00\% (5.52\%)$ &$ 1.00\% (0.71\%)$
			\\
			\hline & & & &
			\\
			{\it Bures-Wasserstein} & $8.80\% (5.54\%)$ & $0\%$ & $29.20\% (6.14\%)$ & $0\%$
			\\
			\hline & & & &
			\\
			{\it Sinkhorn ($\ep =0.1$)} & $17.40\% (7.83\%)$ & $0\%$  &$13.80\% (8.4\%)$ & $0.40\% (0.55\%)$
			\\ 
			\hline & & & &
			\\
			{\it Log-Hilbert-Schmidt} & $0\%$ & $0\%$ & $0\%$ & $0\%$
			\\
			\hline & & & &
			\\
			{\it Affine-invariant} & $0\%$ & $0\%$ & $0.2\% (0.45\%)$ & $0\%$
			\\
			\hline
		\end{tabular}	
	\end{table}
}

\begin{figure}
\centering
\begin{subfigure}[b]{0.2\textwidth}
	\centering\includegraphics[width=\textwidth]{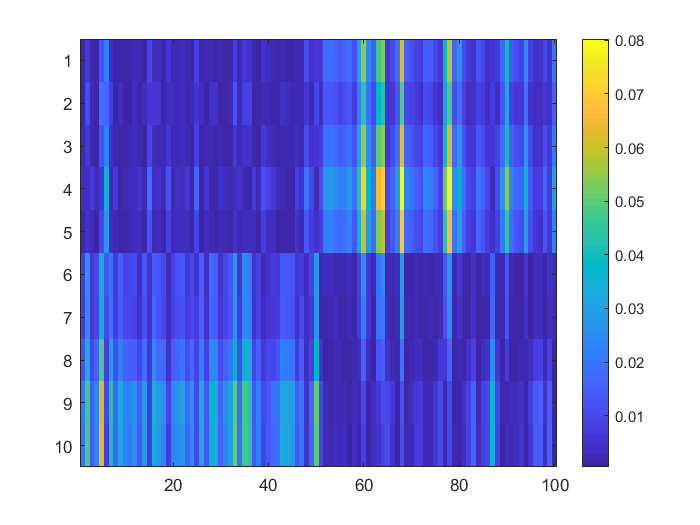}
	\caption{Hilbert-Schmidt}
\end{subfigure}
\begin{subfigure}[b]{0.2\textwidth}
	\centering
	\includegraphics[width=\textwidth]{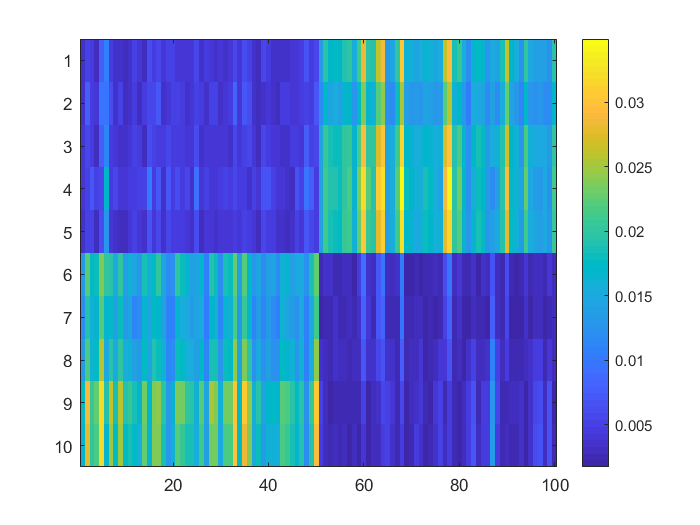}
	\caption{Bures-Wasserstein}
\end{subfigure}
\begin{subfigure}[b]{0.2\textwidth}
	\centering
	\includegraphics[width=\textwidth]{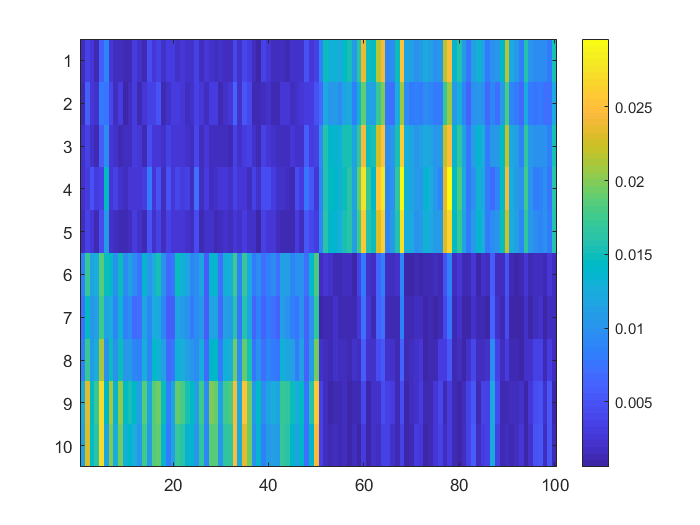}
	\caption{Sinkhorn ($\ep =0.1$)}
\end{subfigure}
\begin{subfigure}[b]{0.2\textwidth}
	\centering
	\includegraphics[width=\textwidth]{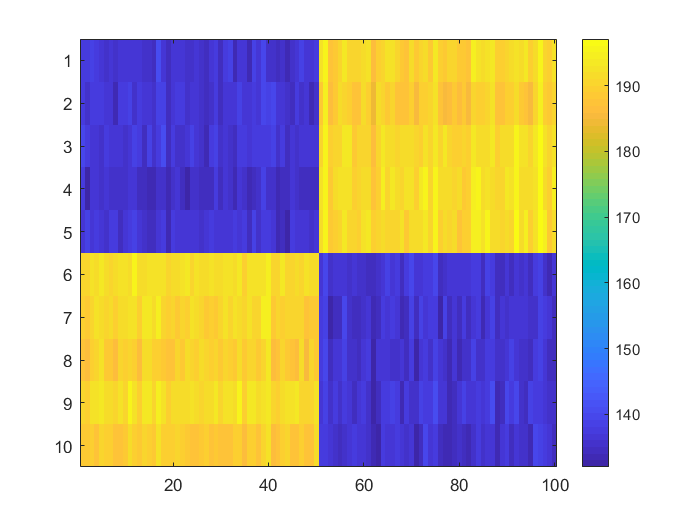}
	\caption{Log-Hilbert-Schmidt}
\end{subfigure}
	\begin{subfigure}[b]{0.2\textwidth}
		\centering\includegraphics[width=\textwidth]{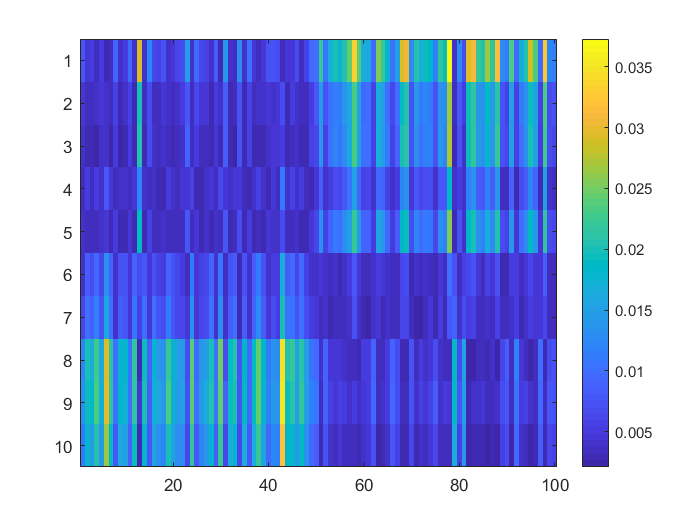}
		\caption{Hilbert-Schmidt}
	\end{subfigure}
\begin{subfigure}[b]{0.2\textwidth}
	\centering
	\includegraphics[width=\textwidth]{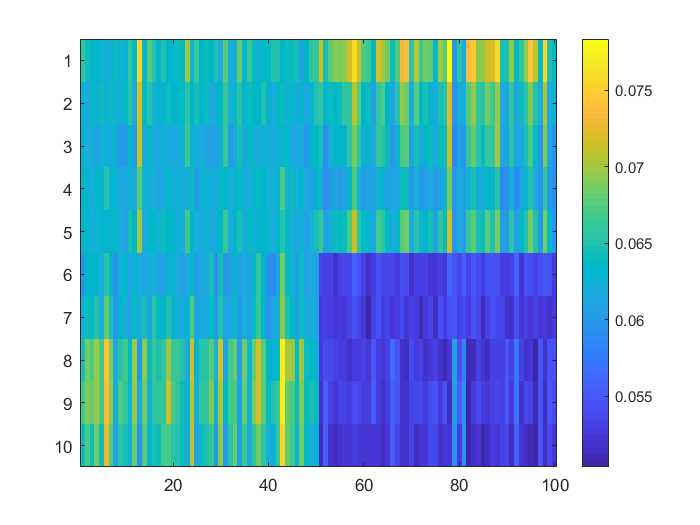}
	\caption{Bures-Wasserstein}
\end{subfigure}
\begin{subfigure}[b]{0.2\textwidth}
	\centering
	\includegraphics[width=\textwidth]{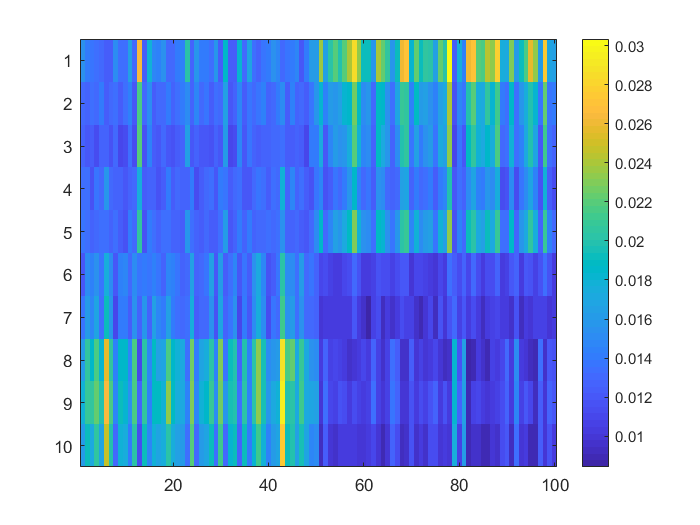}
	\caption{Sinkhorn ($\ep =0.1$)}
\end{subfigure}
\begin{subfigure}[b]{0.2\textwidth}
	\centering
	\includegraphics[width=\textwidth]{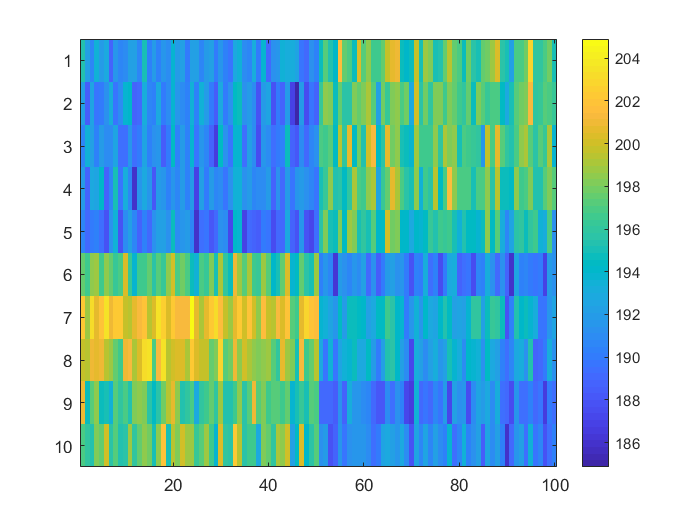}
	\caption{Log-Hilbert-Schmidt}
\end{subfigure}
\caption{Examples of confusion matrices associated with different distances in two different scenarios. Top: $\sigma_1=1, \sigma_2 = 1.3, d=1$. Bottom: $\sigma_1 =1, \sigma_2 =1.1, d=5$.}
\label{figure:confusion-matrices}
\end{figure}


\section{Proofs of main results}
\label{section:proofs}

\subsection{Proofs for the convergence of the Log-Hilbert-Schmidt distance}
We first prove Theorems \ref{theorem:logHS-convergence} and \ref{theorem:logHS-approx-sequence-0}.
In the following, let $\Csc_p(\H)$ denote the set of $p$th Schatten class operators on $\H$, under the norm $||\;||_p$, 
where $||A||_p = (\trace|A|^p)^{1/p}$, $1 \leq p \leq \infty$,
with $\Csc_1(\H)=\Tr(\H)$, 
$\Csc_2(\H) = \HS(\H)$, 
and $\Csc_{\infty}(\H)$ being the set of compact operators under the operator norm  $||\;||$.

\begin{lemma}
	[\textbf{Corollary 3.2 in \cite{Kitta:InequalitiesV}}]
	\label{lemma:inequality-Schatten}
	For any two positive operators $A,B$ on $\H$ such that $A \geq cI > 0$, $B \geq cI > 0$,
	for any bounded operator $X$ on $\H$,
	\begin{equation}
		||A^rX - XB^r||_p\leq rc^{r-1}||AX-XB||_p, 0 < r \leq 1, 1 \leq p \leq \infty.
	\end{equation}
\end{lemma}
\begin{corollary}
	\label{corollary:Schatten-class-rth-root}
	For two operators $A, B \in \Sym^{+}(\H) \cap \Csc_p(\H)$, $1 \leq p \leq \infty$, 
	\begin{align}
		||(I+A)^{r} - (I+B)^{r}||_{p} \leq |r|\;||A-B||_{p}, \;\;\;|r| \leq 1.
	\end{align}
	In particular, for $|r|\leq 1$, $r \neq 0$, we have
	$\frac{||(I+A)^{r} - (I+B)^{r}||_{p}}{|r|} \leq ||A-B||_p$.
\end{corollary}
\begin{proof}
	By assumption, $I+A \geq I$, $I+B \geq I$, so  for $0 < r \leq 1$, the inequality follows immediately from Lemma \ref{lemma:inequality-Schatten} with $c=1$. Consider now $r = -s$, $0 < s \leq 1$, then
	\begin{align*}
		&||(I+A)^{-s} - (I+B)^{-s}||_p = ||(I+A)^{-s}[(I+A)^s - (I+B)^s](I+B)^{-s}||_p
		\\
		&\leq ||(I+A)^{-s}||\;||(I+A)^{s} - (I+B)^s||_p||(I+B)^{-s}|| \leq 
		||(I+A)^{s} - (I+B)^s||_p 
		\leq s||A-B||_p,
	\end{align*}
	using result from the first case. \qed
\end{proof}

\begin{corollary}
	\label{corollary:limit-(I+A)r}
	Let $1 \leq p \leq \infty$.
	Let $r \in \R$, $0 < r \leq 1 $, be fixed. 
	Let $\{A_n\}_{n \in \Nbb}, A \in \Sym(\H) \cap \Csc_p(\H)$ be such that
	$I+A > 0, I+A_n > 0\forall n \in \Nbb$. 
	Let $M_A>0$ be such that $\la x, (I+A)x\ra \geq M_A||x||^2\forall x \in \H$.
	Assume that $\lim_{n \approach \infty}||A_n - A||_p = 0$. Then $\forall \epsilon,0 < \epsilon < M_A$, $\exists N(\epsilon) \in \Nbb$ such that
	\begin{align}
		\frac{||(I+A_n)^r - (I+A)^r||_{p}}{r} \leq (M_A - \epsilon)^{r-1}||A_n - A||_{p} \;\;\forall n \geq N(\epsilon).
	\end{align}
	\begin{align}
		\frac{||(I+A_n)^{-r} - (I+A)^{-r}||_{p}}{r} \leq \frac{1}{(M_A - \epsilon)M_A^r}||A_n - A||_{p} \;\;\forall n \geq N(\epsilon).
	\end{align}
\end{corollary}
\begin{proof}
	For any $\epsilon$ satisfying $0 < \epsilon < M_A$, there exists $N(\epsilon) \in \N$ such that
	$||A_n - A|| < \epsilon$ $\forall n \geq N(\epsilon)$.
	By assumption, we have $I+A \geq M_AI$, so that
	\begin{align*}
		I+A_n &= I +A + (A_n - A) \geq (M_A - \epsilon)I \;\;\; \forall n \geq N(\epsilon),
		\;\;
		(I+A_n)^{-1} \leq (M_A-\epsilon)^{-1}I\;\;\; \forall n \geq N(\epsilon).
	\end{align*}
	For $0 < r \leq 1$, applying Lemma \ref{lemma:inequality-Schatten}  gives
	\begin{align*}
		||(I+A_n)^r - (I+A)^r||_{p} \leq r (M_A - \epsilon)^{r-1}||A_n - A||_{p} \;\;\forall n \geq N(\epsilon).
	\end{align*}
	For $r = -s$, $0 < s \leq 1$, we have
	\begin{align*}
		&||(I+A_n)^{-s} - (I+A)^{-s}||_p = ||(I+A_n)^{-s}[(I+A_n)^s - (I+A)^s](I+A)^{-s}||_p
		\\
		&\leq ||(I+A_n)^{-s}||\;||(I+A_n)^{s} - (I+A)^s||_p||(I+A)^{-s}||
		\leq (M_A -\epsilon)^{-s} M_A^{-s}
		||(I+A_n)^{s} - (I+A)^s||_p 
		\\
		&
		\leq s(M_A -\epsilon)^{-s} M_A^{-s}(M_A-\epsilon)^{s-1}||A_n - A||_p = s(M_A-\epsilon)^{-1}M_A^{-s}||A_n-A||_p,
	\end{align*}
	using the result from the previous case. \qed
\end{proof}


\begin{lemma} 
	[Lemma 9 in \cite{Minh:LogDetIII2018}]
	\label{lemma:limit-log(I+A)}
	Let $A\in \Sym(\H) \cap \HS(\H)$ with $I+A > 0$. Then
	\begin{align}
		\lim_{\alpha \approach 0}\left\|\frac{(I+A)^{\alpha} - I}{\alpha} - \log(I+A)\right\|_{\HS} = 0.
	\end{align}
\end{lemma}

\begin{proof}
	[\textbf{Proof of Theorem \ref{theorem:logHS-convergence}}]
	(a) Consider first the case $\gamma = 1$.
	Let $n \in \Nbb$ be fixed.
	By Lemma \ref{lemma:limit-log(I+A)},
	\begin{align*}
		\lim_{\alpha \approach 0}\left\|\frac{(I+A_n)^{\alpha} - I}{\alpha} - \log(I+A_n)\right\|_{\HS} = 0,
		\;
		\lim_{\alpha \approach 0}\left\|\frac{(I+A)^{\alpha} - I}{\alpha} - \log(I+A)\right\|_{\HS} = 0.
	\end{align*}
	
	(i) If $A, A_n \in \Sym^{+}(\H) \cap \HS(\H)$, by Corollary \ref{corollary:Schatten-class-rth-root},
	for any $0 < \alpha \leq 1$,
	$\left\|\frac{(I+A_n)^{\alpha} - (I+A)^{\alpha}}{\alpha}\right\|_{\HS} \leq ||A_n-A||_{\HS}$.
	Thus for any fixed $0 < \alpha \leq 1$, for any fixed $n \in \Nbb$,
	\begin{align*}
		&\left\|\log(I+A_n) -\log(I+A) \right\|_{\HS}
		\leq \left\|\frac{(I+A_n)^{\alpha} - I}{\alpha} - \log(I+A_n)\right\|_{\HS}
		+\left\|\frac{(I+A_n)^{\alpha} - (I+A)^{\alpha}}{\alpha}\right\|_{\HS}
		+ \left\|\frac{(I+A)^{\alpha} - I}{\alpha} - \log(I+A)\right\|_{\HS}
		\\
		& \leq \left\|\frac{(I+A_n)^{\alpha} - I}{\alpha} - \log(I+A_n)\right\|_{\HS}
		+ ||A_n - A||_{\HS}
		+ \left\|\frac{(I+A)^{\alpha} - I}{\alpha} - \log(I+A)\right\|_{\HS}.
	\end{align*}
	Letting $\alpha \approach 0$ gives $||\log(I+A_n)- \log(I+A)||_{\HS} \leq ||A_n - A||_{\HS} \forall n \in \Nbb$.
	
	(ii) Consider now the general assumption $I+A > 0, I+A_n > 0$ $\forall n \in \Nbb$.
	By Corollary \ref{corollary:limit-(I+A)r}, for a fixed $0 < \alpha \leq 1$, $\forall \epsilon, 0 < \epsilon < M_A$,
	$\exists N(\epsilon) \in \Nbb$ such that 
	\begin{align}
		\frac{||(I+A_n)^{\alpha} - (I+A)^{\alpha}||_{\HS}}{\alpha} \leq (M_A-\epsilon)^{\alpha -1}||A_n - A||_{\HS}\;\;\forall n \geq N(\epsilon).
	\end{align}
	Thus for any fixed $\alpha \in \R$, $0 < \alpha \leq 1$, and any fixed $n \in \N$, $n \geq N(\epsilon)$,
	\begin{align*}
		&\left\|\log(I+A_n) -\log(I+A) \right\|_{\HS}
		\leq \left\|\frac{(I+A_n)^{\alpha} - I}{\alpha} - \log(I+A_n)\right\|_{\HS}
		+\left\|\frac{(I+A_n)^{\alpha} - (I+A)^{\alpha}}{\alpha}\right\|_{\HS}
		+ \left\|\frac{(I+A)^{\alpha} - I}{\alpha} - \log(I+A)\right\|_{\HS}
		\\
		& \leq \left\|\frac{(I+A_n)^{\alpha} - I}{\alpha} - \log(I+A_n)\right\|_{\HS}
		+ (M_A-\epsilon)^{\alpha -1}||A_n - A||_{\HS}
		+ \left\|\frac{(I+A)^{\alpha} - I}{\alpha} - \log(I+A)\right\|_{\HS}.
	\end{align*}
	Fixing $n$ and letting $\alpha \approach 0$ on the right hand side gives
	\begin{align*}
		\left\|\log(I+A_n) -\log(I+A) \right\|_{\HS} \leq (M_A-\epsilon)^{-1}||A_n - A||_{\HS} \;\;\forall n \geq N(\epsilon).
	\end{align*}
	It thus follows that $\lim_{n \approach \infty} \left\|\log(I+A_n) -\log(I+A) \right\|_{\HS} = 0$.	
		
		(b) Consider now the general case $\gamma > 0$. Part (i) then follows from (a) and 
	the identity
	$||\log(\gamma I+A_n) - \log(\gamma I+A)||_{\HS} = \left\|\log\left(I + \frac{A_n}{\gamma}\right) - \log\left(I + \frac{A}{\gamma}\right)\right\|_{\HS}
	$.
	For part (ii), we note that $\gamma I + A \geq M_A \equivalent I +\frac{A}{\gamma} \geq \frac{M}{\gamma}$.
	Furthermore, $\frac{||A_n-A||}{\gamma} \leq \frac{\ep}{\gamma}$ $\forall n \geq N(\ep)$.
	The result then follows similarly from (a).
	\qed
\end{proof}

\begin{proof}
	[\textbf{Proof of Theorem \ref{theorem:logHS-approx-sequence-0}}]
	Since $d_{\logHS}(\gamma_1 I + A, \gamma_2I +B) = ||\log(\gamma_1 I +A) - \log(\gamma_2I +B)||_{\eHS}$
	is a metric, by the triangle inequality and Theorem \ref{theorem:logHS-convergence}, we have in the case $A_n,B_n, A,B \in \Sym^{+}(\H) \cap \HS(\H)$,
	\begin{align*}
		|d_{\logHS}(\gamma_1 I + A_n, \gamma_2I +B_n) - d_{\logHS}(\gamma_1 I+ A, \gamma_2 I +B)|
		&\leq d_{\logHS}(\gamma_1 I + A_n, \gamma_1 I+A) + d_{\logHS}(\gamma_2 I + B_n, \gamma_2I + B)
		\\
		&\leq \frac{1}{\gamma_1}||A_n - A||_{\HS} + \frac{1}{\gamma_2}||B_n - B||_{\HS}.
	\end{align*}
	The general case follows similarly by Theorem \ref{theorem:logHS-convergence}.
	\qed
\end{proof}



\subsection{Proofs for the Log-Hilbert-Schmidt distance between Gaussian processes}

We now prove Theorems \ref{theorem:logHS-convergence-sample-covariance-operator}, \ref{theorem:logHS-approx-finite-covariance}, \ref{theorem:logHS-estimate-unknown-1}, and 
\ref{theorem:logHS-estimate-unknown-2}.

\begin{proof}
	[\textbf{Proof of Theorem \ref{theorem:logHS-convergence-sample-covariance-operator}}]
	By Theorem \ref{theorem:logHS-approx-sequence-0},
	\begin{align*}
		&\Delta = \left|D_{\logHS}^{\gamma}[\Ncal(0,C_{K^1,\Wbf^1}), \Ncal(0,C_{K^2,\Wbf^2})] - D_{\logHS}^{\gamma}[\Ncal(0,C_{K^1}),\Ncal(0,C_{K^2})]\right|
		\\
		&= \left|||\log(\gamma I+C_{K^1,\Wbf^1}) - \log(\gamma I+C_{K^2,\Wbf^2})||_{\HS} - ||\log(\gamma I+C_{K^1}) - \log(\gamma I+C_{K^2})||_{\HS}
		\right|
		\\
		& \leq \frac{1}{\gamma}||C_{K^1,\Wbf^1} - C_{K^1}||_{\HS} + \frac{1}{\gamma}||C_{K^2,\Wbf^2} - C_{K^2}||_{\HS}.
	\end{align*}
	By Proposition \ref{proposition-concentration-sample-cov-operator}, 
	for any $0 < \delta < 1$, with probability at least $1-\delta$,
	\begin{align*}
		||C_{K^1,\Wbf^1} - C_{K^1}||_{\HS} \leq \frac{4\sqrt{3}\kappa_1^2}{\sqrt{N}\delta} \;\;\text{and}\;\;
		||C_{K^2,\Wbf^2} - C_{K^2}||_{\HS} \leq \frac{4\sqrt{3}\kappa_2^2}{\sqrt{N}\delta}.
	\end{align*}
	Consequently, combing all previous expressions gives
	$\Delta \leq \frac{4\sqrt{3}(\kappa_1^2+\kappa_2^2)}{\gamma \sqrt{N}\delta}$.
	\qed
\end{proof}

\begin{lemma}
	\label{lemma:hA-positive}
	Let $A \in \Sym^{+}(\H)$ be compact, with eigenvalues 
	$\{\lambda_k\}_{k\in \Nbb}$, $\lambda_k \geq 0$ $\forall k \in \Nbb$, 
	and corresponding orthonormal eigenvectors $\{\phi_k\}_{k \in \Nbb}$.
	The following operator $h(A)$ is well-defined, with $h(A) \in \Sym^{+}(\H)$,
	\begin{align}
		h(A) = A^{-1}\log(I+A) = \sum_{k=1}^{\infty}\frac{\log(1+\lambda_k)}{\lambda_k}\phi_k \otimes \phi_k, \;\;\text{with }||h(A)||\leq 1.
	\end{align}
Here we use $\lim_{x \approach 0}\frac{\log(1+x)}{x} = 1$ and set $h(0) = I$.
\end{lemma}
\begin{proof}
By the spectral decomposition $A = \sum_{k=1}^{\infty}\lambda_k \phi_k \otimes \phi_k$, we have
$\log(I+A) = \sum_{k=1}^{\infty}\log(1+\lambda_k)\phi_k \otimes \phi_k$.
Since $\lambda_k \geq 0$ $\forall k \in \Nbb$, by the inequality $\log(1+x) \leq x$ $\forall x \geq 0$ and the limit $\lim_{\lambda_k\approach 0}\frac{\log(1+\lambda_k)}{\lambda_k} = 1$ by L'Hopital's rule, we have $0 \leq \log(1+\lambda_k) \leq \lambda_k \forall k \in \Nbb$.
Thus $h(A)\in \Sym^{+}(\H)$, with $||h(A)|| \leq 1$.
\qed
\end{proof}

\begin{lemma}
	\label{lemma:log-AAstar}
Let $\H_1, \H_2$ be two separable Hilbert spaces. Let $A:\H_1 \mapto \H_2$ be compact.
Let $\{\lambda_k(A^{*}A)\}_{k\in \Nbb}$ be the eigenvalues of 
$A^{*}A \in \Sym^{+}(\H_1)$, with
corresponding orthonormal eigenvectors $\{\phi_k(A^{*}A)\}_{k \in \Nbb}$. 
Then $AA^{*} \in \Sym^{+}(\H_2)$ and
\begin{align}
\log(I_{\H_2} + AA^{*}) &= \sum_{k=1}^{\infty}\frac{\log(1+\lambda_k(A^{*}A))}{\lambda_k(A^{*}A)}
(A\phi_k(A^{*}A)) \otimes (A\phi_k(A^*A))
= Ah(A^{*}A)A^{*},
\end{align}
where $h$ is as defined in Lemma \ref{lemma:hA-positive}. If $A^{*}A \in \HS(\H_1)$, then $\log(I_{\H_1} + A^{*}A) \in \HS(\H_1)$ and $\log(I_{\H_2} + AA^{*}) \in \HS(\H_2)$.
\end{lemma}
\begin{proof}
Let $N_A$ be the number of strictly positive eigenvalues of $A^{*}A$. Then \cite{MinhSB:NIPS2014} 
	\begin{align*}
		\log(I_{\H_2} + AA^{*}) &= \sum_{k=1}^{N_A}\frac{\log(1+\lambda_k(A^{*}A))}{\lambda_k(A^{*}A)}
		(A\phi_k(A^{*}A)) \otimes (A\phi_k(A^*A)).
	\end{align*}
If $\lambda_k(A^{*}A) = 0$, i.e. $A^{*}A\phi_k(A^{*}A) = 0$, then $||A\phi_k(A^{*}A)||^2 = \la A\phi_k(A^{*}A), A\phi_k(A^{*}A)\ra = \la \phi_k(A^{*}A), A^{*}A\phi_k(A^{*}A)\ra = 0 \equivalent A\phi_k(A^{*}A) = 0$.
Furthermore, $(A\phi_k(A^{*}A)) \otimes (A\phi_k(A^*A)) = A[\phi_k(A^{*}A) \otimes \phi_k(A^{*}A)]A^{*}$. Thus
\begin{align*}
	\log(I_{\H_2} + AA^{*}) &= \sum_{k=1}^{\infty}\frac{\log(1+\lambda_k(A^{*}A))}{\lambda_k(A^{*}A)}
	(A\phi_k(A^{*}A)) \otimes (A\phi_k(A^*A))
	= Ah(A^{*}A)A^{*}.
\end{align*}
If $A^{*}A \in \HS(\H_1)$, then it follows that $||\log(I_{\H_1} + A^{*}A)||^2_{\HS(\H_1)} = ||\log(I_{\H_2} + AA^{*})||^2_{\HS(\H_2)} = \trace[Ah(A^{*}A)A^{*}]^2 = 
\trace[A^{*}Ah(A^{*}A)]^2 < \infty$.
\qed
\end{proof}
\begin{corollary}
	\label{corollary:trace-logAAstar-logBBstar}
Let $\H_1, \H_2, \H$ be separable Hilbert spaces. Let $A:\H_1 \mapto \H$, $B:\H_2 \mapto \H$ be compact operators
such that $A^{*}A\in \Sym^{+}(\H_1) \cap \HS(\H_1)$, $B^{*}B\in \Sym^{+}(\H_2) \cap \HS(\H_2)$.
Then $AA^{*}, BB^{*} \in \Sym^{+}(\H) \cap \HS(\H)$ and
	\begin{align}
\trace[\log(I_{\H} + AA^{*})\log(I_{\H} + BB^{*})]
& = \trace[B^{*}Ah(A^{*}A)A^{*}Bh(B^{*}B)]
= ||\sqrt{h(A^{*}A)}A^{*}B\sqrt{h(B^{*}B)}||^2_{\HS(\H_2,\H_1)},
\nonumber
\end{align}
where $h$ is as defined in Lemma \ref{lemma:hA-positive}.
\end{corollary}
\begin{proof}
	By Lemma \ref{lemma:log-AAstar}, $\log(I_{\H} + AA^{*}) = Ah(A^{*}A)A^{*} \in \HS(\H)$, $\log(I_{\H} + BB^{*})  =
	Bh(B^{*}B)B^{*} \in \HS(\H)$, and
\begin{align*}
&\trace[\log(I_{\H} + AA^{*})\log(I_{\H} + BB^{*})] 
= \trace[Ah(A^{*}A)A^{*}Bh(B^{*}B)B^{*}]
\\
&= \trace[B^{*}Ah(A^{*}A)A^{*}Bh(B^{*}B)]=
||\sqrt{h(A^{*}A)}A^{*}B\sqrt{h(B^{*}B)}||^2_{\HS(\H_2,\H_1)}.
\end{align*}
\end{proof}

\begin{lemma}
	\label{lemma:h-LKX-convergence}
	Assume Assumptions A1-A6.
	Let $\gamma \in \R,\gamma > 0$ be fixed.
	Let $\Xbf$ be independently sampled from $(T,\nu)$.
	For $h$ as defined in Lemma \ref{lemma:hA-positive},
	for any $0 < \delta < 1$, with probability at least $1-\delta$,
\begin{align}
\left\|h\left(\frac{1}{\gamma}L_{K,\Xbf}\right) - h\left(\frac{1}{\gamma}L_{K}\right)\right\|_{\HS(\H_K)} \leq  \frac{\kappa^2}{2\gamma}\left(\frac{2\log\frac{2}{\delta}}{m} + \sqrt{\frac{2\log\frac{2}{\delta}}{m}}\right). 
\end{align}
\end{lemma}
\begin{proof}
	Since $L_{K,\Xbf}, L_{K} \in\Sym^{+}(\H_K)\cap \HS(\H_K)$,
 $h(\frac{1}{\gamma}L_{K,\Xbf}) -h(\frac{1}{\gamma}L_K) \in \Sym(\H_K) \cap \HS(H_K)$ 
 	by Lemma \ref{lemma:Ainv-logA-norm-p}, with $||h(\frac{1}{\gamma}L_{K,\Xbf}) - h(\frac{1}{\gamma}L_{K})||_{\HS(\H_K)} 
 	\leq \frac{1}{2\gamma}||L_{K,\Xbf} -L_{K}||_{\HS(\H_K)}$.
 	The result then follows from Proposition \ref{proposition:concentration-TK2K1-empirical}.\qed
\end{proof}

\begin{lemma}
\label{lemma:R12-hLK-convergence}
	Assume Assumptions A1-A6. 
	Let $\gamma \in \R, \gamma > 0$ be fixed.
	For all $\Xbf  = (x_i)_{i=1}^m\in T^m$,
\begin{align}
\left\|R_{12,\Xbf}h\left(\frac{1}{\gamma}L_{K^2,\Xbf}\right)\right\|_{\HS(\H_{K^2}, \H_{K^1})} \leq \kappa_1\kappa_2. 
\end{align}
Let $\Xbf = (x_i)_{i=1}^m$ be independently sampled from $(T,\nu)$.
For any $0 <\delta < 1$, with probability at least $1-\delta$,
\begin{align}
\left\|R_{12,\Xbf}h\left(\frac{1}{\gamma}L_{K^2, \Xbf}\right) - R_{12}h\left(\frac{1}{\gamma}L_{K^2}\right)\right\|_{\HS(\H_{K^2}, \H_{K^1})}
\leq \kappa_1\kappa_2\left(1+\frac{1}{2\gamma}\kappa_2^2\right)\left(\frac{2\log\frac{4}{\delta}}{m} + \sqrt{\frac{2\log\frac{4}{\delta}}{m}}\right).
\end{align}
\begin{align}
\left\|R_{12,\Xbf}^{*}h\left(\frac{1}{\gamma}L_{K^1,\Xbf}\right) - R_{12}^{*}h\left(\frac{1}{\gamma}L_{K^1}\right)\right\|_{\HS(\H_{K^1}, \H_{K^2})} \leq 
\kappa_1\kappa_2\left(1+\frac{1}{2\gamma}\kappa_1^2\right)\left(\frac{2\log\frac{4}{\delta}}{m} + \sqrt{\frac{2\log\frac{4}{\delta}}{m}}\right).
\end{align}
\end{lemma}
\begin{proof}
	The first inequality follows from $||h(\frac{1}{\gamma}L_{K^2,\Xbf})||\leq 1$ by Lemma \ref{lemma:hA-positive}
	and $||R_{12,\Xbf}||_{\HS(\H_{K^2}, \H_{K^1})} \leq \kappa_1\kappa_2$ by Proposition \ref{proposition:concentration-TK2K1-empirical}.
	Similarly, since $||R_{12}||_{\HS(\H_{K^2}, \H_{K^1})} \leq \kappa_1\kappa_2$,
\begin{align*}
&\Delta = ||R_{12,\Xbf}h(\frac{1}{\gamma}L_{K^2, \Xbf}) - R_{12}h(\frac{1}{\gamma}L_{K^2})||_{\HS(\H_{K^2}, \H_{K^1})}
\\
&\leq ||R_{12,\Xbf}-R_{12}||_{\HS(\H_{K^2}, \H_{K^1})}||h(\frac{1}{\gamma}L_{K^2,\Xbf})|| + ||R_{12}||_{\HS(\H_{K^2},\H_{K^1})}||h(\frac{1}{\gamma}L_{K^2,\Xbf}) - h(\frac{1}{\gamma}L_{K^2})||
\\
& \leq ||R_{12,\Xbf}-R_{12}||_{\HS(\H_{K^2}, \H_{K^1})} + \kappa_1 \kappa_2||h(\frac{1}{\gamma}L_{K^2,\Xbf}) - h(\frac{1}{\gamma}L_{K^2})||
\end{align*}
By Proposition \ref{proposition:concentration-TK2K1-empirical} and Lemma \ref{lemma:h-LKX-convergence},
the following sets satisfy $\nu^m(U_i) \geq 1-\frac{\delta}{2}$, $i=1,2$,
\begin{align*}
U_1 = \left\{\Xbf \in (T,\nu)^m: ||R_{12,\Xbf}-R_{12}||_{\HS(\H_{K^2}, \H_{K^1})} \leq \kappa_1\kappa_2\left[ \frac{2\log\frac{4}{\delta}}{m} + \sqrt{\frac{2\log\frac{4}{\delta}}{m}}\right]\right \},
\\
U_2 = \left\{\Xbf \in (T,\nu)^m: ||h(\frac{1}{\gamma}L_{K^2,\Xbf}) - h(\frac{1}{\gamma}L_{K^2})|| \leq  \frac{1}{2\gamma}\kappa_2^2\left(\frac{2\log\frac{4}{\delta}}{m} + \sqrt{\frac{2\log\frac{4}{\delta}}{m}}\right) \right\}.
\end{align*}
Thus on $U = U_1 \cap U_2$, with $\nu^m(U) \geq 1-\delta$,
\begin{align*}
\Delta \leq \kappa_1\kappa_2\left(1+\frac{1}{2\gamma}\kappa_2^2\right)\left(\frac{2\log\frac{4}{\delta}}{m} + \sqrt{\frac{2\log\frac{4}{\delta}}{m}}\right).
\end{align*}
The last inequality is obtained similarly.
\qed
\end{proof}

\begin{proposition}
	Assume Assumptions A1-A6.
	Let $\Xbf = (x_i)_{i=1}^m$ be independently sampled from $(T,\nu)$.
	For any $0 < \delta < 1$, with probability at least $1-\delta$,
\begin{align}
&\left\|R_{12,\Xbf}^{*}h\left(\frac{1}{\gamma}L_{K^1,\Xbf}\right)R_{12,\Xbf}h\left(\frac{1}{\gamma}L_{K^2,\Xbf}\right) - 
R_{12}^{*}h\left(\frac{1}{\gamma}L_{K^1}\right)R_{12}h\left(\frac{1}{\gamma}L_{K^2}\right)\right\|_{\tr(\H_{K^2})} 
\nonumber
\\
&\leq
\kappa_1^2\kappa_2^2\left(1 + \frac{\kappa_1^2+\kappa_2^2}{2\gamma}\right)\left(\frac{2\log\frac{8}{\delta}}{m} + \sqrt{\frac{2\log\frac{8}{\delta}}{m}}\right).
\end{align}
Consequently, with probability at least $1-\delta$,
\begin{align}
&\left|\trace\left[R_{12,\Xbf}^{*}h\left(\frac{1}{\gamma}L_{K^1,\Xbf}\right)R_{12,\Xbf}h\left(\frac{1}{\gamma}L_{K^2,\Xbf}\right)\right] - 
\trace\left[R_{12}^{*}h\left(\frac{1}{\gamma}L_{K^1}\right)R_{12}h\left(\frac{1}{\gamma}L_{K^2}\right)\right]\right|
\nonumber
\\
&\leq \kappa_1^2\kappa_2^2\left(1 + \frac{\kappa_1^2+\kappa_2^2}{2\gamma}\right)\left(\frac{2\log\frac{8}{\delta}}{m} + \sqrt{\frac{2\log\frac{8}{\delta}}{m}}\right).
\end{align}
\end{proposition}
\begin{proof}
	By Lemma \ref{lemma:R12-hLK-convergence}, $||R_{12,\Xbf}h(\frac{1}{\gamma}L_{K^2,\Xbf})||_{\HS(\H_{K^2}, \H_{K^1})}
	\leq \kappa_1\kappa_2$,
	$||R_{12}^{*}h(\frac{1}{\gamma}L_{K^1})||_{\HS(\H_{K^1}, \H_{K^2})} \leq \kappa_1\kappa_2$, thus
\begin{align*}
&\Delta = ||R_{12,\Xbf}^{*}h(\frac{1}{\gamma}L_{K^1,\Xbf})R_{12,\Xbf}h(\frac{1}{\gamma}L_{K^2,\Xbf}) - 
R_{12}^{*}h(\frac{1}{\gamma}L_{K^1})R_{12}h(\frac{1}{\gamma}L_{K^2})||_{\tr(\H_{K^2})}
\\
& \leq || [R_{12,\Xbf}^{*}h(\frac{1}{\gamma}L_{K^1,\Xbf}) - R_{12}^{*}h(\frac{1}{\gamma}L_{K^1})]R_{12,\Xbf}h(\frac{1}{\gamma}L_{K^2,\Xbf})||_{\tr(\H_{K^2})}
\\
&\quad + ||R_{12}^{*}h(\frac{1}{\gamma}L_{K^1})[R_{12,\Xbf}h(\frac{1}{\gamma}L_{K^2,\Xbf}) - R_{12}h(\frac{1}{\gamma}L_{K^2})]||_{\tr(\H_{K^2})}
\\
& \leq ||R_{12,\Xbf}^{*}h(\frac{1}{\gamma}L_{K^1,\Xbf}) - R_{12}^{*}h(\frac{1}{\gamma}L_{K^1})||_{\HS(\H_{K^1}, \H_{K^2})}||R_{12,\Xbf}h(\frac{1}{\gamma}L_{K^2,\Xbf})||_{\HS(\H_{K^2}, \H_{K^1})}
\\
& \quad + ||R_{12}^{*}h(\frac{1}{\gamma}L_{K^1})||_{\HS(\H_{K^1}, \H_{K^2})}||R_{12,\Xbf}h(\frac{1}{\gamma}L_{K^2,\Xbf}) - R_{12}h(\frac{1}{\gamma}L_{K^2})||_{\HS(\H_{K^2},\H_{K^1})}
\\
& \leq \kappa_1\kappa_2 ||R_{12,\Xbf}^{*}h(\frac{1}{\gamma}L_{K^1,\Xbf}) - R_{12}^{*}h(\frac{1}{\gamma}L_{K^1})||_{\HS(\H_{K^1}, \H_{K^2})}
+\kappa_1\kappa_2||R_{12,\Xbf}h(\frac{1}{\gamma}L_{K^2,\Xbf}) - R_{12}h(\frac{1}{\gamma}L_{K^2})||_{\HS(\H_{K^2},\H_{K^1})}.
\end{align*}
By Lemma \ref{lemma:R12-hLK-convergence},
the following sets satisfy $\nu^m(U_i) \geq 1-\frac{\delta}{2}$, $i=1,2$
\begin{align*}
	U_1 &= \left\{||R_{12,\Xbf}^{*}h(\frac{1}{\gamma}L_{K^1,\Xbf}) - R_{12}^{*}h(\frac{1}{\gamma}L_{K^1})||_{\HS(\H_{K^1}, \H_{K^2})} \leq
	\kappa_1\kappa_2\left(1+\frac{1}{2\gamma}\kappa_1^2\right)\left(\frac{2\log\frac{8}{\delta}}{m} + \sqrt{\frac{2\log\frac{8}{\delta}}{m}}\right)
	 \right\},
	\\
	U_2 &= \left\{||R_{12,\Xbf}h(\frac{1}{\gamma}L_{K^2,\Xbf}) - R_{12}h(\frac{1}{\gamma}L_{K^2})||_{\HS(\H_{K^2},\H_{K^1})} \leq 
	\kappa_1\kappa_2\left(1+\frac{1}{2\gamma}\kappa_2^2\right)\left(\frac{2\log\frac{8}{\delta}}{m} + \sqrt{\frac{2\log\frac{8}{\delta}}{m}}\right)
	\right\}.
	\end{align*}
Thus on $U = U_1 \cap U_2$, with $\nu^m(U)\geq 1-\delta$, we have
$\Delta \leq \kappa_1^2\kappa_2^2\left(1 + \frac{\kappa_1^2+\kappa_2^2}{2\gamma}\right)\left(\frac{2\log\frac{8}{\delta}}{m} + \sqrt{\frac{2\log\frac{8}{\delta}}{m}}\right)$.
\qed
\end{proof}

\begin{proposition}
\label{proposition:logHS-AAstar-BBstar-switch}
Let $\H_1, \H_2, \H$ be separable Hilbert spaces.
Let $A: \H_1 \mapto \H$, $B:\H_2 \mapto \H$ be compact operators, such that
$A^{*}A \in \Sym^{+}(\H_1) \cap \HS(\H_1)$, $B^{*}B \in \Sym^{+}(\H_2) \cap \HS(\H_2)$.
Then $AA^{*}, BB^{*} \in \Sym^{+}(\H) \cap \HS(\H)$ and 
\begin{align}
||\log(I_{\H}+AA^{*}) - \log(I_{\H}+BB^{*})||_{\HS(\H)}^2 &=
||\log(I_{\H_1}+A^{*}A)||^2_{\HS(\H_1)} + ||\log(I_{\H_2}+B^{*}B)||^2_{\HS(\H_2)}
\nonumber 
\\
&\quad - 2\trace[B^{*}Ah(A^{*}A)A^{*}Bh(B^{*}B)].
\end{align}
Here $h$ is as defined in Lemma \ref{lemma:hA-positive}.
\end{proposition}
\begin{proof}
Since $AA^{*}:\H \mapto \H$ and $A^{*}A:\H_1 \mapto \H_1$ have the same nonzero eigenvalues,
we have 
$||\log(I_{\H}+AA^{*})||^2_{\HS(\H)}
= ||\log(I_{\H_1} + A^{*}A)||^2_{\HS(\H_1)}$. Similarly, $||\log(I_{\H}+BB^{*})||^2_{\HS(\H)}
= ||\log(I_{\H_2} + B^{*}B)||^2_{\HS(\H_1)}$. Thus
\begin{align*}
&||\log(I_{\H}+AA^{*}) - \log(I_{\H}+BB^{*})||_{\HS(\H)}^2
\\
& = ||\log(I_{\H} + AA^{*})||^2_{\HS(\H)}
+||\log(I_{\H} + BB^{*})||^2_{\HS(\H)} - 2 \trace[\log(I_{\H} + AA^{*})\log(I_{\H} + BB^{*})]
\\
& = ||\log(I_{\H_1}+A^{*}A)||^2_{\HS(\H_1)} + ||\log(I_{\H_2}+B^{*}B)||^2_{\HS(\H_2)}
- 2\trace[B^{*}Ah(A^{*}A)A^{*}Bh(B^{*}B)],
\end{align*}
where the last equality follows from Corollary \ref{corollary:trace-logAAstar-logBBstar}.
\qed
\end{proof}

\begin{proof}
	[\textbf{Proof of Proposition \ref{proposition:logHS-Gram-matrices-LKX}}]
	(i) For the first identity,
		let  $A = \frac{1}{\sqrt{\gamma}}R_{K^1}^{*}:\H_{K^1}\mapto \Lcal^2(T,\nu)$, $B = \frac{1}{\sqrt{\gamma}}R_{K^2}^{*}:\H_{K^2}\mapto \Lcal^2(T,\nu)$, then
		$AA^{*} = \frac{1}{\gamma}R_{K^1}^{*}R_{K^1} = \frac{1}{\gamma}C_{K^1}:\Lcal^2(T,\nu) \mapto \Lcal^2(T,\nu)$,
		$BB^{*} = \frac{1}{\gamma}C_{K^2}$, $A^{*}A = \frac{1}{\gamma}R_{K^1}R_{K^1}^{*} = \frac{1}{\gamma}L_{K^1}:\H_{K^1} \mapto \H_{K^1}$, $B^{*}B = \frac{1}{\gamma}L_{K^2}$,
	$A^{*}B = \frac{1}{\gamma}R_{K^1}R_{K^2}^{*} = \frac{1}{\gamma}R_{12}: \H_{K^2} \mapto \H_{K^1}$,
	$B^{*}A = \frac{1}{\gamma}R_{12}^{*}$. 
	By Proposition \ref{proposition:logHS-AAstar-BBstar-switch}, 
	\begin{align*}
		&||\log(\gamma I+C_{K^1}) - \log(\gamma I+C_{K^2})||^2_{\HS(\Lcal^2(T,\nu))}
		= ||\log(I+\frac{1}{\gamma}C_{K^1}) - \log(I+\frac{1}{\gamma}C_{K^2})||^2_{\HS(\Lcal^2(T,\nu))}
		\\
		& = ||\log(I +AA^{*}) - \log(I+BB^{*})||^2_{\HS(\Lcal^2(T,\nu))}
		\\
		& =
		||\log(I_{\H_{K^1}}+A^{*}A)||^2_{\HS(\H_{K^1})} + ||\log(I_{\H_{K^2}}+B^{*}B)||^2_{\HS(\H_{K^2})}
		- 2\trace[B^{*}Ah(A^{*}A)A^{*}Bh(B^{*}B)]
		\\
		& = ||\log(I+\frac{1}{\gamma}L_{K^1})||_{\HS(K^1)}^2 + ||\log(I+\frac{1}{\gamma}L_{K^2})||_{\HS(K^2)}^2  - \frac{2}{\gamma^2}\trace[R_{12}^{*}h(\frac{1}{\gamma}L_{K^1})R_{12}h(\frac{1}{\gamma}L_{K^2})].
	\end{align*}
(ii) For the second identity,
let $A = \frac{1}{\sqrt{m\gamma}}S_{1,\Xbf}:\H_{K^1} \mapto \R^m$, $B = \frac{1}{\sqrt{m\gamma}}S_{2,\Xbf}:\H_{K^2} \mapto \R^m$, then $A^{*}A = \frac{1}{\gamma}L_{K^1,\Xbf}$, $AA^{*} = \frac{1}{m\gamma}K^1[\Xbf]$, $B^{*}B = \frac{1}{\gamma}L_{K^2}$, $BB^{*}  = \frac{1}{m\gamma}K^2[\Xbf]$,
$A^{*}B = \frac{1}{\gamma}R_{12,\Xbf}:\H_{K^2} \mapto \H_{K^1}$, $B^{*}A = \frac{1}{\gamma}R_{12,\Xbf}^{*}$.
By Proposition \ref{proposition:logHS-AAstar-BBstar-switch},
\begin{align*}
&	\left\|\log\left(\gamma I+\frac{1}{m}K^1[\Xbf]\right) - \log\left(\gamma I+\frac{1}{m}K^2[\Xbf]\right)\right\|_F^2
=\left\|\log\left(I+\frac{1}{m\gamma }K^1[\Xbf]\right) - \log\left(I+\frac{1}{m\gamma }K^2[\Xbf]\right)\right\|_F^2
\\
&=||\log(I+AA^{*}) - \log(I+BB^{*})||^2_F
\\
&=  ||\log(I_{\H_{K^1}}+A^{*}A)||^2_{\HS(\H_{K^1})} + ||\log(I_{\H_{K^2}}+B^{*}B)||^2_{\HS(\H_{K^2})}
- 2\trace[B^{*}Ah(A^{*}A)A^{*}Bh(B^{*}B)]
\\
& = \left\|\log\left(I+\frac{1}{\gamma}L_{K^1,\Xbf}\right)\right\|^2_{\HS(\H_{K^1})} + \left\|\log\left(I+\frac{1}{\gamma}L_{K^2,\Xbf}\right)\right\|^2_{\HS(\H_{K^2})} -\frac{2}{\gamma^2}\trace\left[R_{12,\Xbf}^{*}h\left(\frac{1}{\gamma}L_{K^1,\Xbf}\right)R_{12,\Xbf}h\left(\frac{1}{\gamma}L_{K^2,\Xbf}\right)\right].
\qed
\end{align*}
\end{proof}


\begin{lemma}
\label{lemma:logLKX-HS-norm-convergence}
Assume Assumptions A1-A6. 
Let $\gamma  \in \R, \gamma > 0$ be fixed.
Let $\Xbf = (x_i)_{i=1}^m$ be independently sampled from $(T,\nu)$.
For any $0 < \delta < 1$, with probability at least $1-\delta$,
\begin{align}
\left|\left\|\log\left(I + \frac{1}{\gamma}L_{K,\Xbf}\right)\right\|^2_{\HS(\H_K)} - 
\left\|\log\left(I + \frac{1}{\gamma}L_{K}\right)\right\|^2_{\HS(\H_K)}\right|
\leq \frac{2\kappa^4}{\gamma^2}
\left(\frac{2\log\frac{2}{\delta}}{m} + \sqrt{\frac{2\log\frac{2}{\delta}}{m}}\right). 
\end{align}
Equivalently, with probability at least $1-\delta$,
\begin{align}
\left|\left\|\log\left(I + \frac{1}{\gamma}K[\Xbf]\right)\right\|^2_{F} - 
\left\|\log\left(I + \frac{1}{\gamma}C_{K}\right)\right\|^2_{\HS(\Lcal^2(T,\nu))}\right|
\leq \frac{2\kappa^4}{\gamma^2}
\left(\frac{2\log\frac{2}{\delta}}{m} + \sqrt{\frac{2\log\frac{2}{\delta}}{m}}\right). 
\end{align}
\end{lemma}
\begin{proof} By Lemma \ref{lemma:logA-norm-p} and Proposition \ref{proposition:concentration-TK2K1-empirical},
\begin{align*}
&\left|\left\|\log\left(I + \frac{1}{\gamma}L_{K,\Xbf}\right)\right\|^2_{\HS(\H_K)} - 
\left\|\log\left(I + \frac{1}{\gamma}L_{K}\right)\right\|^2_{\HS(\H_K)}\right|
\\
& = \left|\left\|\log\left(I + \frac{1}{\gamma}L_{K,\Xbf}\right)\right\|_{\HS(\H_K)} - 
\left\|\log\left(I + \frac{1}{\gamma}L_{K}\right)\right\|_{\HS(\H_K)}\right|
\left[\left\|\log\left(I + \frac{1}{\gamma}L_{K,\Xbf}\right)\right\|_{\HS(\H_K)} + 
\left\|\log\left(I + \frac{1}{\gamma}L_{K}\right)\right\|_{\HS(\H_K)}\right]
\\
& \leq \left\|\log\left(I + \frac{1}{\gamma}L_{K,\Xbf}\right) - \log\left(I + \frac{1}{\gamma}L_{K}\right)\right\|_{\HS(\H_K)}\frac{1}{\gamma}[||L_{K,\Xbf}||_{\HS(\H_K)} + ||L_K||_{\HS(\H_K)}]
\\
& \leq \frac{1}{\gamma^2}||L_{K,\Xbf} -L_K||_{\HS(\H_K)}[||L_{K,\Xbf}||_{\HS(\H_K)} + ||L_K||_{\HS(\H_K)}]
\leq \frac{2\kappa^4}{\gamma^2}
\left(\frac{2\log\frac{2}{\delta}}{m} + \sqrt{\frac{2\log\frac{2}{\delta}}{m}}\right). 
\qed
\end{align*}
\end{proof}

\begin{proof}
	[\textbf{Proof of Theorem \ref{theorem:logHS-approx-finite-covariance}}]
	By Proposition \ref{proposition:logHS-Gram-matrices-LKX},
	\begin{align*}
		&\Delta = \left|	\left\|\log\left(\gamma I+\frac{1}{m}K^1[\Xbf]\right) - \log\left(\gamma I+\frac{1}{m}K^2[\Xbf]\right)\right\|^2_F
		- ||\log(\gamma I+C_{K^1}) - \log(\gamma I+C_{K^2})||^2_{\HS(\Lcal^2(T,\nu))}\right|
		\\
		&\leq \left|\left\|\log\left(I+\frac{1}{\gamma }L_{K^1,\Xbf}\right)\right\|^2_{\HS(\H_{K^1})} 
		- \left\|\log\left(I +\frac{1}{\gamma}L_{K^1}\right)\right\|^2_{\HS(\H_{K^1}))} \right|
		\\
		&\quad + \left|\left\|\log\left(I+\frac{1}{\gamma}L_{K^2,\Xbf}\right)\right\|^2_{\HS(\H_{K^2})} 
		- \left\|\log\left( I+\frac{1}{\gamma}L_{K^2}\right)\right\|^2_{\HS(\H_{K^2})} \right|
		\\
		& \quad + \frac{2}{\gamma^2}\left| \trace\left[R_{12,\Xbf}^{*}h\left(\frac{1}{\gamma}L_{K^1,\Xbf}\right)R_{12,\Xbf}h\left(\frac{1}{\gamma}L_{K^2,\Xbf}\right)\right] - \trace\left[R_{12}^{*}h\left(\frac{1}{\gamma}L_{K^1}\right)R_{12}h\left(\frac{1}{\gamma}L_{K^2}\right)\right]\right|.
	\end{align*}
	For each $0 < \delta < 1$, the following sets satisfy $\nu^m(U_i) \geq 1- \frac{\delta}{3}$, $i=1,2,3$,
	\begin{align*}
		U_1 &= \left\{\Xbf \in (T,\nu)^m: \left|\left\|\log\left(I+\frac{1}{\gamma}L_{K^1,\Xbf}\right)\right\|^2_{\HS(\H_{K^1})} - \left\|\log\left( I+\frac{1}{\gamma}L_{K^1}\right)\right\|^2_{\HS(\H_{K^1})} \right|
		\leq 
		\frac{2\kappa^4}{\gamma^2}
		\left(\frac{2\log\frac{6}{\delta}}{m} + \sqrt{\frac{2\log\frac{6}{\delta}}{m}}\right)
		\right\},
		\\
		U_2 &= \left\{\Xbf \in (T,\nu)^m: \left|\left\|\log\left(I+\frac{1}{\gamma}L_{K^2,\Xbf}\right)\right\|^2_{\HS(\H_{K^2})} - \left\|\log\left( I+\frac{1}{\gamma}L_{K^2}\right)\right\|^2_{\HS(\H_{K^2})} \right|
		\leq 
		\frac{2\kappa^4}{\gamma^2}
		\left(\frac{2\log\frac{6}{\delta}}{m} + \sqrt{\frac{2\log\frac{6}{\delta}}{m}}\right)
		\right\},
		\\
		U_3 &= \left\{\Xbf \in (T,\nu)^m:  \left|\trace\left[R_{12,\Xbf}^{*}h\left(\frac{1}{\gamma}L_{K^1,\Xbf}\right)R_{12,\Xbf}h\left(\frac{1}{\gamma}L_{K^2,\Xbf}\right)\right] - 
		\trace\left[R_{12}^{*}h\left(\frac{1}{\gamma}L_{K^1}\right)R_{12}h\left(\frac{1}{\gamma}L_{K^2}\right)\right]\right|
		\right.
		\\
		& \quad \quad \quad \quad \quad \quad\leq
		\left.\kappa_1^2\kappa_2^2\left(1 + \frac{\kappa_1^2+\kappa_2^2}{2\gamma}\right)\left(\frac{2\log\frac{24}{\delta}}{m} + \sqrt{\frac{2\log\frac{24}{\delta}}{m}}\right) \right\}.
	\end{align*}
	Thus on the set $U = U_1 \cap U_2 \cap U_3$, with $\nu^m(U) \geq 1-\delta$,
	\begin{align*}
		\Delta \leq \frac{2(\kappa_1^4 + \kappa_2^4)}{\gamma^2}
		\left(\frac{2\log\frac{6}{\delta}}{m} + \sqrt{\frac{2\log\frac{6}{\delta}}{m}}\right)
		+ \frac{2\kappa_1^2\kappa_2^2}{\gamma^2}\left(1 + \frac{\kappa_1^2+\kappa_2^2}{2\gamma}\right)\left(\frac{2\log\frac{24}{\delta}}{m} + \sqrt{\frac{2\log\frac{24}{\delta}}{m}}\right).
		\qed
	\end{align*}
\end{proof}

\begin{proof}
	[\textbf{Proof of Theorem \ref{theorem:logHS-estimate-unknown-1}}]
	By Theorem \ref{theorem:logHS-approx-sequence-0},
	\begin{align*}
		\Delta &= \left|\left\|\log\left(\gamma I + \frac{1}{m}\hat{K}^1_{\Wbf^1}[\Xbf]\right) - \log\left(\gamma I + \frac{1}{m}\hat{K}^2_{\Wbf^2}[\Xbf]\right)\right\|_F - \left\|\log\left(\gamma I + \frac{1}{m}K^1[\Xbf]\right) - \log\left(\gamma I + \frac{1}{m}K^2[\Xbf]\right)\right\|_F\right|
		\\ 
		&\leq \left\|\log\left(\gamma I + \frac{1}{m}\hat{K}^1_{\Wbf^1}[\Xbf]\right) - \log\left(\gamma I + \frac{1}{m}K^1[\Xbf]\right)\right\|_F
		+ \left\|\log\left(\gamma I + \frac{1}{m}\hat{K}^2_{\Wbf^2}[\Xbf]\right) - \log\left(\gamma I + \frac{1}{m}K^2[\Xbf]\right)\right\|_F
		\\
		& \leq \frac{1}{m\gamma}||\hat{K}^1_{\Wbf^1}[\Xbf]- K^1[\Xbf]||_F + \frac{1}{m\gamma}
		||\hat{K}^2_{\Wbf^2}[\Xbf]- K^2[\Xbf]||_F.
	\end{align*}
	By Proposition \ref{proposition:concentration-empirical-covariance}, for any $0 < \delta < 1$, with probability at least $1-\delta$,
	\begin{align*}
		||\hat{K}^1_{\Wbf^1}[\Xbf] - K^1[\Xbf]||_F \leq \frac{4\sqrt{3}m\kappa_1^2}{\sqrt{N}\delta}\;\;\;\text{and }
		||\hat{K}^2_{\Wbf^2}[\Xbf] - K^2[\Xbf]||_F \leq \frac{4\sqrt{3}m\kappa_2^2}{\sqrt{N}\delta}.
	\end{align*}
	It follows that $\Delta \leq \frac{4\sqrt{3}(\kappa_1^2+\kappa_2^2)}{\gamma\sqrt{N}\delta}$ with probability at least $1-\delta$.
	\qed
\end{proof}

\begin{proof}
	[\textbf{Proof of Theorem \ref{theorem:logHS-estimate-unknown-2}}]	
	We combine the results from Theorems \ref{theorem:logHS-approx-finite-covariance} and \ref{theorem:logHS-estimate-unknown-1}. We have
	\begin{align*}
		&\Delta = \left|D^{\gamma}_{\logE}\left[\Ncal\left(0, \frac{1}{m}\hat{K}^1_{\Wbf^1}[\Xbf]\right), \Ncal\left(0, \frac{1}{m}\hat{K}^2_{\Wbf^2}[\Xbf]\right)\right] - D^{\gamma}_{\logHS}[\Ncal(0,C_{K^1}), \Ncal(0,C_{K^2})]\right|
		\\
		&\leq \left|D^{\gamma}_{\logE}\left[\Ncal\left(0, \frac{1}{m}\hat{K}^1_{\Wbf^1}[\Xbf]\right), \Ncal\left(0, \frac{1}{m}\hat{K}^2_{\Wbf^2}[\Xbf]\right)\right]-D^{\gamma}_{\logE}\left[\Ncal\left(0, \frac{1}{m}K^1[\Xbf]\right), \Ncal\left(0, \frac{1}{m}K^2[\Xbf]\right)\right] \right|
		\\
		&\quad + \left|D^{\gamma}_{\logE}\left[\Ncal\left(0, \frac{1}{m}K^1[\Xbf]\right), \Ncal\left(0, \frac{1}{m}K^2[\Xbf]\right)\right]- D^{\gamma}_{\logHS}[\Ncal(0,C_{K^1}), \Ncal(0,C_{K^2})]\right| = \Delta_1 + \Delta_2.
	\end{align*}
	By Theorem \ref{theorem:logHS-estimate-unknown-1}, the following set $U_1 \subset (T,\nu)^m$ satisfies $\nu^m(U)\geq 1-\frac{\delta}{2}$, 
	\begin{align*}
		U_1 = \left\{\Xbf \in (T,\nu)^m: \Delta_1 \leq \frac{8\sqrt{3}(\kappa_1^2+\kappa_2^2)}{\gamma\sqrt{N}\delta}\right\}. 
	\end{align*}
	By Theorem \ref{theorem:logHS-approx-finite-covariance}, using the inequality $(a-b)^2 \leq |a^2-b^2|$ for $a \geq 0, b\geq 0$, for a fixed $\Xbf \in (T,\nu)^m$, the following set $U_2 \subset (\Omega_1, P_1)^N \times (\Omega_2, P_2)^N$ satisfies
	$(P_1 \otimes P_2)^N(U_2) \geq 1-\frac{\delta}{2}$,
	{\small
		\begin{align*}
			U_2 = \left\{(\Wbf^1, \Wbf^2): \Delta_2 \leq \sqrt{\frac{2(\kappa_1^4 + \kappa_2^4)}{\gamma^2}
				\left(\frac{2\log\frac{12}{\delta}}{m} + \sqrt{\frac{2\log\frac{12}{\delta}}{m}}\right)
				+ \frac{2\kappa_1^2\kappa_2^2}{\gamma^2}\left(1 + \frac{\kappa_1^2+\kappa_2^2}{2\gamma}\right)\left(\frac{2\log\frac{48}{\delta}}{m} + \sqrt{\frac{2\log\frac{48}{\delta}}{m}}\right)}\right\}.
		\end{align*}
	}
	Let $U = (U_1 \times (\Omega_1,P_1)^N \times (\Omega_2,P_2)^N) \cap ((T,\nu)^m \times U_2)$,
	then $(\nu^m \otimes P_1^N \otimes P_2^N)(U) \geq 1-\delta$ and
	\begin{align*}
		\Delta \leq \frac{8\sqrt{3}(\kappa_1^2+\kappa_2^2)}{\gamma\sqrt{N}\delta}
		+ \frac{1}{\gamma}\sqrt{2(\kappa_1^4 + \kappa_2^4)
			\left(\frac{2\log\frac{12}{\delta}}{m} + \sqrt{\frac{2\log\frac{12}{\delta}}{m}}\right)
			+ {2\kappa_1^2\kappa_2^2}\left(1 + \frac{\kappa_1^2+\kappa_2^2}{2\gamma}\right)\left(\frac{2\log\frac{48}{\delta}}{m} + \sqrt{\frac{2\log\frac{48}{\delta}}{m}}\right)}
	\end{align*}
	$\forall (\Xbf, \Wbf^1, \Wbf^2) \in U$.
	\qed
\end{proof}

\subsection{Proofs for the convergence of the affine-invariant Riemannian distance}

We now prove Theorems \ref{theorem:affineHS-convergence} and \ref{theorem:affineHS-approx-sequence-0}.

\begin{lemma}
	\label{lemma:logI+A-norm-p}
	Let $A \in \Csc_p(\H)$ with $I + A > 0$ and $||A|| < 1$. Then
	$\log(I+A) \in \Csc_p(\H)$, with
\begin{align}
||\log(I+A)||_p \leq \frac{||A||_p}{1-||A||}.
\end{align}
\end{lemma}
\begin{proof}
For $||A|| <1$, the following series is absolutely convergent,
\begin{align*}
\log(I+A) = A - \frac{A^2}{2} + \frac{A^3}{3} - \frac{A^4}{4} + \cdots
\end{align*}
It follows that, since $\Csc_p(\H)$ is a Banach algebra and a two-sided ideal in $\Lcal(\H)$
\begin{align*}
||\log(I+A)||_p \leq ||A||_p\left[1 +\frac{||A||}{2} + \frac{||A||^2}{3} + \cdots\right]
\leq ||A||_p[1 + ||A|| + ||A||^2 + \cdots] = 
 \frac{||A||_p}{1-||A||}.
 \qed
\end{align*}
\end{proof}

\begin{proof}
	[\textbf{Proof of Theorem \ref{theorem:affineHS-convergence}}]
	(i) Consider first the case $\gamma = 1$.
	Write $(I+A)^{-1/2}(I+A_n)(I+A)^{-1/2} = (I+A)^{-1} + (I+A)^{-1/2}A_n(I+A)^{-1/2} = I + (I+A)^{-1/2}(A_n -A)(I+A)^{-1/2}$. Since $\lim_{n \approach \infty}||A_n - A|| = 0$,
	$\forall 0 < \ep < M_A$, $\exists N(\ep)\in \Nbb$ such that $\forall n \geq N(\ep)$, $||A_n - A|| < \ep$ and
	\begin{align*}
		||(I+A)^{-1/2}(A_n -A)(I+A)^{-1/2}|| \leq ||(I+A)^{-1}||\;||A_n - A|| \leq \frac{1}{M_A}||A_n - A||< \frac{\epsilon}{M_A} < 1.
	\end{align*}
	By Lemma \ref{lemma:logI+A-norm-p}, $\forall n \geq N(\ep)$,
	\begin{align*}
		||\log[(I+A)^{-1/2}(I+A_n)(I+A)^{-1/2}]||_{\HS} &\leq \frac{||(I+A)^{-1/2}(A_n-A)(I+A)^{-1/2}||_{\HS}}{1-||(I+A)^{-1/2}(A_n-A)(I+A)^{-1/2}||} \leq  \frac{||(I+A)^{-1}||\;||A_n-A||_{\HS}}{1-(\ep/M_A)} 
		\\
		&\leq \frac{(1/M_A)}{1-(\ep/M_A)}||A_n-A||_{\HS} = \frac{1}{M_A-\ep}||A_n -A||_{\HS}.
	\end{align*}
	If $A \in \Sym^{+}(\H)\cap \HS(H)$, we can set $M_A = 1$ which gives the second bound.
	
	(ii) Consider now the general case $\gamma > 0$.
	We have $\gamma I + A \geq M_A \equivalent I + \frac{A}{\gamma} \geq \frac{M_A}{\gamma}$.
	Let $N(\ep) \in \Nbb$ be such that $||A_n - A|| < \ep \forall n \geq N(\ep)$, then $||\frac{A_n}{\gamma} - \frac{A}{\gamma}|| < \frac{\ep}{\gamma}$ $\forall n \geq \N(\ep)$. Since
	$||(\gamma I +A)^{-1/2}(\gamma I + A_n)(\gamma I + A)^{-1/2}||_{\HS} = ||(I + \frac{A}{\gamma})^{-1/2}(I + \frac{A_n}{\gamma})(I+\frac{A}{\gamma})||_{\HS}$, applying part (i) gives
	\begin{align*}
		||(\gamma I +A)^{-1/2}(\gamma I + A_n)(\gamma I + A)^{-1/2}||_{\HS} \leq \frac{1}{(M_A/\gamma) - (\ep/\gamma)}\left\|\frac{A_n}{\gamma} - \frac{A}{\gamma}\right\|_{\HS}  = \frac{1}{M_A-\ep}||A_n - A||_{\HS}.
	\end{align*}
	If $A \in \Sym^{+}(\H)\cap \HS(\H)$, setting $M_A = \gamma$ gives the last bound.\qed
\end{proof}

\begin{proof}
	[\textbf{Proof of Theorem \ref{theorem:affineHS-approx-sequence-0}}]
	Since $d_{\aiHS}(\gamma_1I+A, \gamma_2 I+B) = ||\log[(\gamma_1I+A)^{-1/2}(\gamma_2I+B)(\gamma_1I+A)^{-1/2}]||_{\eHS}$ is a metric,
	by the triangle inequality and Theorem \ref{theorem:affineHS-convergence},
	$\forall 0 < \ep < \min\{M_A,M_B\}$, $\exists N(\ep) \in \Nbb$ such that $\forall n \geq N(\ep)$,
	$||A_n - A|| < \ep, ||B_n - B|| < \ep$ and
	\begin{align*}
		&|d_{\aiHS}(\gamma_1I+A_n, \gamma_2I+B_n) - d_{\aiHS}(\gamma_1I+A,\gamma_2I+B)| \leq d_{\aiHS}(\gamma_1I+A_n,\gamma_1I+A) + d_{\aiHS}(\gamma_2I+B_n,\gamma_2I+B)
		\\
		& \leq \frac{1}{M_A - \ep}||A_n-A||_{\HS} + \frac{1}{M_B-\ep}||B_n - B||_{\HS}.
	\end{align*}
	If $A,B \in \Sym^{+}(\H)\cap \HS(\H)$, setting $M_A = \gamma_1, M_B =\gamma_2$ gives the last bound.
	\qed
\end{proof}

\subsection{Proofs for the affine-invariant Riemannian distance between Gaussian processes}

We now prove Theorems \ref{theorem:aiHS-convergence-sample-covariance-operator} and \ref{theorem:aiHS-approx-finite-covariance}.

\begin{proof}
	[\textbf{Proof of Theorem \ref{theorem:aiHS-convergence-sample-covariance-operator}}]
	By Proposition \ref{proposition-concentration-sample-cov-operator}, 
	for any $0 < \delta < 1$, with probability at least $1-\delta$,
	\begin{align*}
		\Delta_1 = ||C_{K^1,\Wbf^1} - C_{K^1}||_{\HS} \leq \frac{4\sqrt{3}\kappa_1^2}{\sqrt{N}\delta} \;\;\text{and}\;\;
		\Delta_2 = ||C_{K^2,\Wbf^2} - C_{K^2}||_{\HS} \leq \frac{4\sqrt{3}\kappa_2^2}{\sqrt{N}\delta}.
	\end{align*}
	For $0 < \ep < \gamma$,  let $N(\ep)\in \Nbb$, $N(\ep) \geq 1+ \max\left\{\frac{48\kappa_1^4}{\ep^2\delta^2}, 
	\frac{48\kappa_2^4}{\ep^2\delta^2}\right\}$,
	then
	$\Delta_1(N) < \ep, \Delta_2(N) < \ep$ $\forall N \geq N(\ep)$. By Theorem \ref{theorem:affineHS-approx-sequence-0}, $\forall N \geq N(\ep)$, with probability at least $1-\delta$,
	\begin{align*}
		&\Delta_3 = \left|D^{\gamma}_{\aiHS}[\Ncal(0,C_{K^1,\Wbf^1}), \Ncal(0,C_{K^2,\Wbf^2})] - D^{\gamma}_{\aiHS}[\Ncal(0,C_{K^1}),\Ncal(0,C_{K^2})]\right|
		\\
		&= \left|||\log[(\gamma I + C_{K^1, \Wbf^1})^{-1/2}(\gamma I + C_{K^2,\Wbf^2})(\gamma I + C_{K^1,\Wbf^1})^{-1/2}]||_{\HS} \right.
		\\
		&\left.\quad - ||\log[(\gamma I + C_{K^1})^{-1/2}(\gamma I + C_{K^2})(\gamma I + C_{K^1})^{-1/2}]||_{\HS}
		\right|
		\\
		& \leq \frac{1}{\gamma - \ep}\left[||C_{K^1,\Wbf^1}-C_{K^1}||_{\HS} + ||C_{K^2,\Wbf^2}-C_{K^2}||_{\HS}\right]
		\leq \frac{4\sqrt{3}(\kappa_1^2 + \kappa_2^2)}{(\gamma - \ep)\sqrt{N}\delta}.
		\qed
	\end{align*}
\end{proof}

\begin{proposition}
	\label{proposition:aiHS-AAstar-switch}
	Let $\H,\H_1, \H_2$ be separable Hilbert spaces.
	Let $A: \H_1\mapto \H, B:\H_2 \mapto \H$ be compact operators such that $A^{*}A \in \HS(\H_1), B^{*}B\in \HS(\H_2)$.
	Then $AA^{*}, BB^{*}\in \Sym^{+}(\H)\cap \HS(\H)$ and
\begin{align}
	\label{equation:aiHS-AAstar-switch}
&||\log[(I_{\H}+AA^{*})^{-1/2}(I_{\H}+BB^{*})(I_{\H}+AA^{*})^{-1/2}]||_{\HS(\H)}^2
\nonumber
\\
&=\trace\left[\log\left[I+\begin{pmatrix}
	(I_{\H_1}+A^{*}A)^{-1}-I	 & (I_{\H_1}+A^{*}A)^{-1}A^{*}B\\
	-B^{*}A(I_{\H_1}+A^{*}A)^{-1}	 & B^{*}B - B^{*}A(I_{\H_1}+A^{*}A)^{-1}A^{*}B
\end{pmatrix}
\right]\right]^2 = \trace[\log(I+D)]^2.
\end{align}
The operator $(I+tD):\H_1 \oplus \H_2 \mapto \H_1 \oplus \H_2$ is positive definite $\forall t \in [0,1]$, with 
 $||(I+tD)^{-1}|| \leq \frac{1+\lambda_1(A^{*}A)}{[1+(1-t)\lambda_1(A^{*}A)]}$ and $\sup_{t \in [0,1]}||(I+tD)^{-1}||
 \leq 1+\lambda_1(A^{*}A)$,
where $\lambda_1(A^{*}A)$ is the largest eigenvalue of $A^{*}A$.
\end{proposition}
We note that the operator $D$ in Eq.\eqref{equation:aiHS-AAstar-switch} has the form
$D = \begin{pmatrix}
	D_{11} & D_{12}
	\\
	-D_{12}^{*} & D_{22}
\end{pmatrix}
$ and is {\it not} self-adjoint.
\begin{proof}
	Expanding $(I+AA^{*})^{-1/2}(I+BB^{*})(I+AA^{*})^{-1/2}$ as 
\begin{align*}
&(I+AA^{*})^{-1/2}(I+BB^{*})(I+AA^{*})^{-1/2} = (I+AA^{*})^{-1} + (I+AA^{*})^{-1/2}BB^{*}(I+AA^{*})^{-1/2}
\\
& = I-A(I+A^{*}A)^{-1}A^{*} + (I+AA^{*})^{-1/2}BB^{*}(I+AA^{*})^{-1/2}
\\
& = I +
\begin{pmatrix}
-A(I+A^{*}A)^{-1/2} & (I+AA^{*})^{-1/2}B
\end{pmatrix}
\begin{pmatrix}
(I+A^{*}A)^{-1/2}A^{*}
\\
B^{*}(I+AA^{*})^{-1/2}
\end{pmatrix} = I+C.
\end{align*}
Consider the block operators
$
\begin{pmatrix}
	-A(I+A^{*}A)^{-1/2} & (I+AA^{*})^{-1/2}B
\end{pmatrix}: \H_1 \oplus \H_2 \mapto \H
$ and
$
\begin{pmatrix}
	(I+A^{*}A)^{-1/2}A^{*}
	\\
	B^{*}(I+AA^{*})^{-1/2}
\end{pmatrix}: \H \mapto \H_1 \oplus \H_2
$.
The nonzero eigenvalues of the operator
$C: \begin{pmatrix}
	-A(I+A^{*}A)^{-1/2} & (I+AA^{*})^{-1/2}B
\end{pmatrix}
\begin{pmatrix}
	(I+A^{*}A)^{-1/2}A^{*}
	\\
	B^{*}(I+AA^{*})^{-1/2}
\end{pmatrix}:\H \mapto \H$ are the same as those of
the operator $D:\H_1 \oplus \H_2 \mapto \H_1 \oplus \H_2$, where
\begin{align*}
D &= \begin{pmatrix}
	(I+A^{*}A)^{-1/2}A^{*}
	\\
	B^{*}(I+AA^{*})^{-1/2}
\end{pmatrix}
\begin{pmatrix}
	-A(I+A^{*}A)^{-1/2} & (I+AA^{*})^{-1/2}B
\end{pmatrix}  
\\
&=\begin{pmatrix}
-(I+A^{*}A)^{-1/2}A^{*}A(I+A^{*}A)^{-1/2} & (I+A^{*}A)^{-1/2}A^{*}(I+AA^{*})^{-1/2}B
\\
-B^{*}(I+AA^{*})^{-1/2}A(I+A^{*}A)^{-1/2} & B^{*}(I+AA^{*})^{-1}B
\end{pmatrix}
\\
&= \begin{pmatrix}
(I+A^{*}A)^{-1}-I	 & (I+A^{*}A)^{-1}A^{*}B\\
-B^{*}A(I+A^{*}A)^{-1}	 & B^{*}B - B^{*}A(I+A^{*}A)^{-1}A^{*}B
\end{pmatrix}.
\end{align*}
Here we have used the following identities (which are special cases of Corollary 2 in \cite{Minh:LogDetIII2018}) 
\begin{align}
(I+A^{*}A)^{1/2}A^{*} &= A^{*}(I+AA^{*})^{1/2}, \;\;\text{equivalently} \;\;\; A^{*}(I+AA^{*})^{-1/2} = (I+A^{*}A)^{-1/2}A^{*}, 
\\
A(I+A^{*}A)^{1/2} &= (I+AA^{*})^{1/2}A,\;\;\;\text{equivalently}\;\;\; (I+AA^{*})^{-1/2}A = A(I+A^{*}A)^{-1/2}.
\end{align}
Thus the nonzero eigenvalues of $\log(I+C)$ are the same as those of $\log(I+D)$.
Therefore
\begin{align*}
||\log(I+C)||^2_{\HS} = \trace[\log(I+C)]^2 = \trace[\log(I+D)]^2.
\end{align*}
The operator $D:\H_1 \oplus \H_2 \mapto \H_1 \oplus \H_2$ has the form
$D = \begin{pmatrix}
	D_{11} & D_{12}
	\\
	-D_{12}^{*} & D_{22}
\end{pmatrix}
$
and is thus not self-adjoint. For any $t \in [0,1]$ and $\forall x =(x_1, x_2) \in \H_1 \oplus \H_2$,
\begin{align*}
	&\la x , (I+tD)x\ra = \left\la \begin{pmatrix}x_1 \\x_2 \end{pmatrix},
	\begin{pmatrix}
		I+tD_{11} & tD_{12}
		\\
		-tD_{12}^{*} & I+tD_{22}
	\end{pmatrix}
	\begin{pmatrix}
		x_1 \\ x_2
	\end{pmatrix}
	\right\ra = \la x_1, (I+tD_{11})x_1\ra + \la x_2, (I+tD_{22})x_2\ra
	\\
	&= \la x_1, (1-t)I + t(I+A^{*}A)^{-1}x_1\ra + \la x_2, [I + tB^{*}(I+AA^{*})^{-1}B] x_2\ra \geq \left[(1-t) + \frac{t}{1+\lambda_1(A^{*}A)}\right]||x_1||^2 + ||x_2||^2 
	\\
	&\geq \frac{[1+(1-t)\lambda_1(A^{*}A)]||x||^2}{1+\lambda_1(A^{*}A)}.
\end{align*}
Thus $(I+tD)$ is positive definite $\forall t \in [0,1]$.
By the Cauchy-Schwarz Inequality, $||(I+tD)x|| \geq \frac{[1+(1-t)\lambda_1(A^{*}A)]}{1+\lambda_1(A^{*}A)}||x||$ $\forall x \in \H_1 \oplus \H_2$, from which
it follows that $I+tD$ is invertible, with $||(I+tD)^{-1}|| \leq \frac{1+\lambda_1(A^{*}A)}{[1+(1-t)\lambda_1(A^{*}A)]}$. It is clear then that $\sup_{t \in [0,1]}||(I+tD)^{-1}||
\leq 1+\lambda_1(A^{*}A)$.
\qed
\end{proof}

\begin{proof}
[\textbf{Proof of Proposition \ref{proposition:aiHS-RKHS-representation}}]
The first expression follows from
$||\log[(\gamma I+C_{K^1})^{-1/2}(\gamma I+C_{K^2})(\gamma I+C_{K^1})^{-1/2}]||_{\HS(\Lcal^2(T,\nu))}^2
= ||\log[(I+\frac{1}{\gamma}C_{K^1})^{-1/2}(I+\frac{1}{\gamma}C_{K^2})(I+\frac{1}{\gamma}C_{K^1})^{-1/2}]||_{\HS(\Lcal^2(T,\nu))}^2$ and Proposition \ref{proposition:aiHS-AAstar-switch}, with
$A =\frac{1}{\sqrt{\gamma}}R_{K^1}^{*}$, $B = \frac{1}{\sqrt{\gamma}}R_{K^2}^{*}$,
$AA^{*} = \frac{1}{\gamma}C_{K^1}$, $A^{*}A = \frac{1}{\gamma}L_{K^1}$, $BB^{*} = \frac{1}{\gamma}C_{K^2}$,
$B^{*}B = \frac{1}{\gamma}L_{K^2}$, $A^{*}B = \frac{1}{\gamma}R_{12}$.

Similarly, the second expression follows from
$\left\|\log\left[\left(\gamma I+\frac{1}{m}K^1[\Xbf]\right)^{-1/2}\left(\gamma I+\frac{1}{m}K^2[\Xbf]\right)\left(\gamma I+ \frac{1}{m}K^1[\Xbf]\right)^{-1/2}\right]\right\|_{F}^2
= \left\|\log\left[\left(I+\frac{1}{m\gamma }K^1[\Xbf]\right)^{-1/2}\left( I+\frac{1}{m\gamma }K^2[\Xbf]\right)\left(I+ \frac{1}{m\gamma}K^1[\Xbf]\right)^{-1/2}\right]\right\|_{F}^2$ 
and Proposition \ref{proposition:aiHS-AAstar-switch}, with
$A = \frac{1}{\sqrt{m\gamma}}S_{1,\Xbf}$, $B = \frac{1}{\sqrt{m \gamma}}S_{2,\Xbf}$,
$AA^{*} = \frac{1}{m \gamma}K^1[\Xbf]$, $A^{*}A = \frac{1}{\gamma}L_{K^1,\Xbf}$, $BB^{*} = \frac{1}{m\gamma}K^2[\Xbf]$,
$B^{*}B = \frac{1}{\gamma}L_{K^2,\Xbf}$, $A^{*}B = \frac{1}{\gamma}R_{12,\Xbf}$.
\qed
\end{proof}

\begin{lemma}
	\label{lemma:HS-norm-block-operators}
	Let $A_{11}\in \HS(\H_1), A_{12} \in \HS(\H_2,\H_1), A_{21}\in \HS(\H_1,\H_2), A_{22} \in \HS(\H_2)$. 
Consider the operator 
$A = \begin{pmatrix}
A_{11} & A_{12}
\\
A_{21} & A_{22}
\end{pmatrix}: \H_1 \oplus \H_2 \mapto \H_1 \oplus \H_2$. Then
$||A||^2_{\HS} = \sum_{i,j=1}^2||A_{ij}||^2_{\HS}$ and $||A||_{\HS} \leq \sum_{i,j=1}^2||A_{ij}||_{\HS}$.
%
\end{lemma}
\begin{proof}
	Let $\{e_{i,k}\}_{k \in \Nbb}$ be an orthonormal basis for 
	$\H_i$, $i = 1,2$, then
	$\left\{\begin{pmatrix}
	e_{1,k}
	\\0
	\end{pmatrix},
\begin{pmatrix}
0
\\
e_{2,k}
\end{pmatrix}\right\}_{k \in \Nbb}$ is an orthonormal basis for $\H_1 \oplus \H_2$.
By definition of the Hilbert-Schmidt norm,
\begin{align*}
||A||^2_{\HS} &= \sum_{k=1}^{\infty}\left(\left\|A\begin{pmatrix}
	e_{1,k}
	\\0
\end{pmatrix}\right\|^2 + \left\|A\begin{pmatrix}
0
\\
e_{2,k}
\end{pmatrix}\right\|^2 \right)
= \sum_{k=1}^{\infty}[||A_{11}e_{1,k}||^2 + ||A_{21}e_{1,k}||^2 + ||A_{12}e_{2,k}||^2 +||A_{22}e_{2,k}||^2]
\\
& = ||A_{11}||^2_{\HS} + ||A_{12}||^2_{\HS} + ||A_{21}||^2_{\HS} + ||A_{22}||^2_{\HS}.
\end{align*}
From this it follows that
$||A||_{\HS} = \sqrt{\sum_{i,j=1}^2||A_{ij}||^2_{\HS}} \leq \sum_{i,j=1}^2||A_{ij}||_{\HS}$.
\qed
\end{proof}

\begin{proof}
	[\textbf{Proof of Theorem \ref{theorem:aiHS-approx-finite-covariance}}]
Define $D = \begin{pmatrix}
	(I+\frac{1}{\gamma}L_{K^1})^{-1}-I	 & \frac{1}{\gamma}(I+\frac{1}{\gamma}L_{K^1})^{-1}R_{12}\\
	-\frac{1}{\gamma}R_{12}^{*}(I+\frac{1}{\gamma}L_{K^1})^{-1}	 & \frac{1}{\gamma}L_{K^2} - \frac{1}{\gamma^2}R_{12}^{*}(I+\frac{1}{\gamma}L_{K^1})^{-1}R_{12}
\end{pmatrix}: \H_{K^1} \oplus \H_{K^2} \mapto \H_{K^1} \oplus \H_{K^2}$
and $D_{\Xbf} = \begin{pmatrix}
	(I+\frac{1}{\gamma}L_{K^1,\Xbf})^{-1}-I	 & \frac{1}{\gamma}(I+\frac{1}{\gamma}L_{K^1,\Xbf})^{-1}R_{12,\Xbf}\\
	-\frac{1}{\gamma}R_{12,\Xbf}^{*}(I+\frac{1}{\gamma}L_{K^1,\Xbf})^{-1}	 & \frac{1}{\gamma}L_{K^2,\Xbf} - \frac{1}{\gamma^2}R_{12,\Xbf}^{*}(I+\frac{1}{\gamma}L_{K^1,\Xbf})^{-1}R_{12,\Xbf}
\end{pmatrix}: \H_{K^1,\Xbf} \oplus \H_{K^2,\Xbf} \mapto \H_{K^1,\Xbf} \oplus \H_{K^2,\Xbf}$.
By Lemma \ref{lemma:trace-log-square-HS-1},
\begin{align*}
\Delta &= \left|\left\|\log\left[\left(\gamma I+\frac{1}{m}K^1[\Xbf]\right)^{-1/2}\left(\gamma I+\frac{1}{m}K^2[\Xbf]\right)\left(\gamma I+ \frac{1}{m}K^1[\Xbf]\right)^{-1/2}\right]\right\|^2_{F}\right.
\nonumber
\\
&\left.-||\log[(\gamma I+C_{K^1})^{-1/2}(\gamma I+C_{K^2})(\gamma I+C_{K^1})^{-1/2}]||^2_{\HS(\Lcal^2(T,\nu))}\right|
\\
& = |\trace[\log(I+D_{\Xbf})]^2 - \trace[\log(I+D)]^2| \leq c_{D_{\Xbf}}c_{D}||D_{\Xbf} - D||_{\HS}[c_{D_{\Xbf}}||D_{\Xbf}||_{\HS} + c_D||D||_{\HS}]. 
\end{align*}
Here the constants $c_{D_{\Xbf}}$, $c_D$ are given by Proposition \ref{proposition:logHS-AAstar-BBstar-switch} by
\begin{align*}
c_{D_{\Xbf}} &= \sup_{t \in [0,1]}||(I+tD_{\Xbf})^{-1}||
\leq 1+\frac{1}{\gamma}\lambda_1(L_{K^1,\Xbf}) \leq 1 +\frac{\kappa_1^2}{\gamma},
\\
c_D &= \sup_{t \in [0,1]}||(I+tD)^{-1}||
\leq 1+\frac{1}{\gamma}\lambda_1(L_{K^1}) \leq 1 +\frac{\kappa_1^2}{\gamma}.
\end{align*}
By Lemma \ref{lemma:HS-norm-block-operators},
\begin{align*}
||D||_{\HS} &\leq 
\left\|\left(I+\frac{1}{\gamma}L_{K^1}\right)^{-1}-I\right\|_{\HS}+ \frac{2}{\gamma}\left\|\left(I+\frac{1}{\gamma}L_{K^1}\right)^{-1}R_{12}\right\|_{\HS}+
\frac{1}{\gamma}||L_{K^2}||_{\HS} +\frac{1}{\gamma^2}\left\|R_{12}^{*}\left(I+\frac{1}{\gamma}L_{K^1}\right)^{-1}R_{12}\right\|_{\HS}
\\
& \leq \frac{1}{\gamma}||L_{K^1}||_{\HS} + \frac{2}{\gamma}||R_{12}||_{\HS} +\frac{1}{\gamma}||L_{K^2}||_{\HS}
+ \frac{1}{\gamma^2}||R_{12}||_{\HS}^2
\\
& \leq \frac{1}{\gamma}\kappa_1^2 + \frac{2}{\gamma}\kappa_1\kappa_2 + \frac{1}{\gamma}\kappa_2^2 + \frac{1}{\gamma^2}\kappa_1^2\kappa_2^2 = \frac{1}{\gamma}(\kappa_1 + \kappa_2)^2 + \frac{\kappa_1^2\kappa_2^2}{\gamma^2}. 
\end{align*}
Similarly, $||D_{\Xbf}||_{\HS} \leq  \frac{1}{\gamma}(\kappa_1 + \kappa_2)^2 + \frac{\kappa_1^2\kappa_2^2}{\gamma^2}$.
It thus follows that
\begin{align}
	\label{equation:aiHS-Delta-1}
\Delta \leq \frac{1}{\gamma}\left(1+\frac{\kappa_1^2}{\gamma}\right)^3\left[(\kappa_1 + \kappa_2)^2 + \frac{\kappa_1^2\kappa_2^2}{\gamma}\right]||D_{\Xbf}-D||_{\HS}.
\end{align}
By Lemma \ref{lemma:HS-norm-block-operators}, $||D_{\Xbf} - D||_{\HS} \leq \Delta_1 + 2\Delta_2 + \Delta_3 + \Delta_4$, with the $\Delta_j$'s given in the following. The first term is
\begin{align*}
\Delta_1 &= \left\|\left(I+\frac{1}{\gamma}L_{K^1,\Xbf}\right)^{-1} - \left(I+\frac{1}{\gamma}L_{K^1}\right)^{-1}\right\|_{\HS}
= \left\|\left(I+\frac{1}{\gamma}L_{K^1,\Xbf}\right)^{-1}\left[\left(I+\frac{1}{\gamma}L_{K^1,\Xbf}\right) - \left(I+\frac{1}{\gamma}L_{K^1}\right)\right] \left(I+\frac{1}{\gamma}L_{K^1}\right)^{-1}\right\|_{\HS}
\\
&= \frac{1}{\gamma}\left\|\left(I+\frac{1}{\gamma}L_{K^1,\Xbf}\right)^{-1}\right\|\;\left\|L_{K^1,\Xbf} - L_{K^1}\right\|_{\HS} \left\|\left(I+\frac{1}{\gamma}L_{K^1}\right)^{-1}\right\|
\leq \frac{1}{\gamma}||L_{K^1,\Xbf} - L_{K^1}||_{\HS}.
\end{align*}
The second term is
\begin{align*}
\Delta_2 &= \frac{1}{\gamma}\left\|\left(I+\frac{1}{\gamma}{L_{K^1,\Xbf}}\right)^{-1}R_{12,\Xbf} - \left(I+\frac{1}{\gamma}{L_{K^1}}\right)^{-1}R_{12}\right\|_{\HS}
\\
&\leq \frac{1}{\gamma}\left\|\left(I+\frac{1}{\gamma}{L_{K^1,\Xbf}}\right)^{-1}[R_{12,\Xbf} -R_{12}]\right\|_{\HS}+
\frac{1}{\gamma}\left\| \left[\left(I+\frac{1}{\gamma}{L_{K^1,\Xbf}}\right)^{-1}-\left(I+\frac{1}{\gamma}{L_{K^1}}\right)^{-1}\right]R_{12}\right\|_{\HS}
\\
& \leq \frac{1}{\gamma}||R_{12,\Xbf} - R_{12}||_{\HS} + \frac{1}{\gamma^2}||L_{K^1,\Xbf} - L_{K^1}||_{\HS}||R_{12}||
\leq  \frac{1}{\gamma}||R_{12,\Xbf} - R_{12}||_{\HS} + \frac{\kappa_1\kappa_2}{\gamma^2}||L_{K^1,\Xbf} - L_{K^1}||_{\HS}.
\end{align*}
The third term is
$\Delta_3 = \frac{1}{\gamma}||L_{K^2,\Xbf} - L_{K^2}||$. The fourth term is
\begin{align*}
\Delta_4 &= \frac{1}{\gamma^2}\left\|R_{12,\Xbf}^{*}\left(I + \frac{1}{\gamma}L_{K^1,\Xbf}\right)^{-1}R_{12,\Xbf}
- R_{12}^{*}\left(I + \frac{1}{\gamma}L_{K^1}\right)^{-1}R_{12}\right\|_{\HS}
\\
&\leq \frac{1}{\gamma^2}||R_{12,\Xbf}^{*} - R_{12}^{*}||_{\HS}\left\|\left(I + \frac{1}{\gamma}L_{K^1,\Xbf}\right)^{-1}\right\|\;||R_{12,\Xbf}||_{\HS}
+\frac{1}{\gamma^2}||R_{12}||\;\left\|\left(I + \frac{1}{\gamma}L_{K^1,\Xbf}\right)^{-1}R_{12,\Xbf}- \left(I + \frac{1}{\gamma}L_{K^1}\right)^{-1}R_{12}\right\|_{\HS}
\\
& \leq \frac{\kappa_1\kappa_2}{\gamma^2}||R_{12,\Xbf}-R_{12}||_{\HS} + \frac{\kappa_1\kappa_2}{\gamma^2}[||R_{12,\Xbf}-R_{12}||_{\HS} + \frac{\kappa_1\kappa_2}{\gamma}||L_{K^1,\Xbf}-L_{K^1}||_{\HS}]
\\
& = \frac{2\kappa_1\kappa_2}{\gamma^2}||R_{12,\Xbf}-R_{12}||_{\HS} + \frac{\kappa_1^2\kappa_2^2}{\gamma^3}||L_{K^1,\Xbf}-L_{K^1}||_{\HS}.
\end{align*}
Combining the expressions for all the $\Delta_j$'s, we obtain
\begin{align*}
||D_{\Xbf} - D||_{\HS} &\leq \left(\frac{2}{\gamma} + \frac{2\kappa_1\kappa_2}{\gamma^2}\right)||R_{12,\Xbf}-R_{12}||_{\HS} + \left(\frac{1}{\gamma} + \frac{2\kappa_1\kappa_2}{\gamma^2} + \frac{\kappa_1^2\kappa_2^2}{\gamma^3}\right)||L_{K^1,\Xbf}-L_{K^1}||_{\HS} + \frac{1}{\gamma}||L_{K^2,\Xbf}- L_{K^2}||_{\HS}
\\
& = \frac{1}{\gamma}\left[2\left(1+\frac{\kappa_1\kappa_2}{\gamma}\right)||R_{12,\Xbf}-R_{12}||_{\HS} + \left(1+\frac{\kappa_1\kappa_2}{\gamma}\right)^2||L_{K^1,\Xbf}-L_{K^1}||_{\HS} + ||L_{K^2,\Xbf}- L_{K^2}||_{\HS}\right].
\end{align*}
By Proposition \ref{proposition:concentration-TK2K1-empirical}, for any $0 < \delta <1$, with probability at least
$1-\delta$, the following there inequalities hold simultaneously,
\begin{align*}
||R_{12,\Xbf} - R_{12}||_{\HS(\H_{K^2}, \H_{K^1})} &\leq \kappa_1\kappa_2\left[ \frac{2\log\frac{6}{\delta}}{m} + \sqrt{\frac{2\log\frac{6}{\delta}}{m}}\right],
\\
||L_{K^1,\Xbf} - L_{K^1}||_{\HS(\H_{K^1})} &\leq \kappa_1^2\left[ \frac{2\log\frac{6}{\delta}}{m} + \sqrt{\frac{2\log\frac{6}{\delta}}{m}}\right],\;\;
||L_{K^2,\Xbf} - L_{K^2}||_{\HS(\H_{K^2})} \leq \kappa_2^2\left[ \frac{2\log\frac{6}{\delta}}{m} + \sqrt{\frac{2\log\frac{6}{\delta}}{m}}\right].
\end{align*}
It follows that with probability at least $1-\delta$,
\begin{align*}
||D_{\Xbf} - D||_{\HS} 
&\leq
\frac{1}{\gamma}\left[2\left(1+\frac{\kappa_1\kappa_2}{\gamma}\right)\kappa_1\kappa_2 + \left(1+\frac{\kappa_1\kappa_2}{\gamma}\right)^2\kappa_1^2 + \kappa_2^2\right]
 \left[ \frac{2\log\frac{6}{\delta}}{m} + \sqrt{\frac{2\log\frac{6}{\delta}}{m}}\right]
\\
& = \frac{1}{\gamma}\left(\kappa_1 + \kappa_2 + \frac{\kappa_1^2\kappa_2}{\gamma}\right)^2\left[ \frac{2\log\frac{6}{\delta}}{m} + \sqrt{\frac{2\log\frac{6}{\delta}}{m}}\right].
\end{align*}
Combining this with Eq.\eqref{equation:aiHS-Delta-1}, we obtain
\begin{align*}
\Delta \leq \frac{1}{\gamma^2}\left(1+\frac{\kappa_1^2}{\gamma}\right)^3\left[(\kappa_1 + \kappa_2)^2 + \frac{\kappa_1^2\kappa_2^2}{\gamma}\right]\left(\kappa_1 + \kappa_2 + \frac{\kappa_1^2\kappa_2}{\gamma}\right)^2\left[ \frac{2\log\frac{6}{\delta}}{m} + \sqrt{\frac{2\log\frac{6}{\delta}}{m}}\right].
\end{align*}
with probability at least $1-\delta$.
\qed
\end{proof}

\begin{lemma}
	\label{lemma:I+tA}
	Assume that $A \in \Lcal(\H)$ with $I+A > 0$. 
	Let $M_A > 0$ be such that $\la x, (I+A)x\ra \geq M_A||x||^2$ $\forall x \in \H$.
	Then for $t \in [0,1]$, $I+tA > 0$, $||I+tA||\geq (1-t) + tM_A$, and $||(I+tA)^{-1}||\leq \frac{1}{(1-t) +tM_A}$.	
\end{lemma}
\begin{proof}
By the assumption on $M_A$,
$\la x, (I+A)x\ra = ||x||^2 + \la x, Ax\ra \geq M_A||x||^2$ $\forall x \in \H$.
Then $\la x, Ax\ra \geq (M_A-1)||x||^2 \imply t\la x, Ax\ra \geq t(M_A-1)||x||^2$ for $t \geq 0$. Thus $\forall x \in \H$,
\begin{align*}
\la x, (I+tA)x\ra = ||x||^2 + t\la x, Ax\ra \geq [(1-t)+tM_A]||x||^2.
\end{align*}
It follows that $\forall t \in [0,1]$, $(I+tA) > 0$. By the Cauchy-Schwarz Inequality, $\forall x \in \H$,
\begin{align*}
||x||\;||(I+tA)x|| \geq \la x, (I+tA)x\ra \geq [(1-t)+tM_A]||x||^2 \imply ||(I+tA)x|| \geq [(1-t)+tM_A]||x||.
\end{align*}
Thus $||I+tA|| \geq (1-t) +tM_A$ and $I+tA$ is invertible, with a bounded inverse, and
$||(I+tA)^{-1}||\leq \frac{1}{(1-t)+tMA}$.
\qed
\end{proof}

\begin{lemma}
	\label{lemma:log-I+A-integral-representation-general}
Let $A \in \Lcal(\H)$ be a compact operator, with $I+A > 0$.
Then the principal logarithm $\log(I+A)$ is well-defined, compact, and admits the following integral representation
\begin{align}
	\label{equation:logI+A-representation}
\log(I+A) = A\int_{0}^1(I+tA)^{-1}dt.
\end{align}
\end{lemma}
\begin{proof}
	By Lemma \ref{lemma:I+tA}, $\int_{0}^t(I+tA)^{-1}dt$ is bounded, thus
	if Eq.\eqref{equation:logI+A-representation} holds, then $\log(I+A)$ is compact.
	
	(i) Consider first the case $A \in \Sym(\H)$. Let $\{\lambda_k\}_{k \in \Nbb}$ be the eigenvalues of $A$, 
with corresponding normalized eigenvectors $\{u_k\}_{k \in \Nbb}$.
Then $1 +\lambda_k > 0$ $\forall k \in \Nbb$ and $A$ admits the spectral decomposition
$A = \sum_{k=1}^{\infty}\lambda_k u_k \otimes u_k$.
From the identity $\log(1+\lambda) = \int_{0}^1\frac{\lambda}{1+\lambda t}dt$, we have
\begin{align*}
	\log(I+A) &= \sum_{k=1}^{\infty}\log(1+\lambda_k)u_k \otimes u_k = \sum_{k=1}^{\infty}\left(\int_{0}^1\frac{\lambda_k}{1+\lambda_k t}dt\right) u_k \otimes u_k
	= \int_{0}^1A(I+tA)^{-1}dt.
\end{align*}
(ii) Consider now the case $A = UBU^{-1}$, with $B$ compact, $I+B > 0$ and $U \in \Lcal(\H)$ invertible. By part (i),
$\log(I+B)$ admits the representation
$	\log(I+B) = \int_{0}^1B(I+tB)^{-1}dt$.
Thus $\log(I+A)$ is well-defined by
\begin{align*}
	\log(I+A) &= \log(I+UBU^{-1}) = \log[U(I + B)U^{-1}] = U\log(I+B)U^{-1} 
	\\
	&= U\left[\int_{0}^1B(I+tB)^{-1}dt\right]U^{-1} = \int_{0}^1A(I+tA)^{-1}dt.
\end{align*}
(iii) Consider the general case.
We first briefly recall the Dunford-Riesz functional calculus (\cite{Dunford1988linearOperators}, VII.3) for bounded linear operators in the Hilbert space setting. Let $U\subset \Cbb$ be an open set, with boundary $\Gamma$ being a piecewise rectifiable curve, positively oriented. 
Let $A \in \Lcal(\H)$ with spectrum $\sigma(A)$. Assume that $U \supseteq \sigma(A)$. Let $f:\Cbb \mapto \Cbb$ be a function analytic in a domain containing $U \cup \Gamma$. Then the function $f(A)$ is well-defined by the following Bochner integral 
\begin{align}
	f(A) = \frac{1}{2\pi i}\oint_{\Gamma}f(z)(zI - A)^{-1}dz.
\end{align}
Let $\log(z)$ denote the principal logarithm of $z \in \Cbb$, then $\log(z)$ is analytic on $\Cbb - (-\infty, 0]$ and
thus $\log(1+z)$ is analytic on $\Cbb - (-\infty, -1]$.
For $A \in \Lcal(\H)$ a compact operator, with $I+ A > 0$,
the principal logarithm of $I+A$ is then well-defined by the following integral
\begin{align}
	\log(I+A) = \frac{1}{2\pi i}\oint_{\Gamma}(zI - A)^{-1}\log(1+z)dz,
\end{align}
with $\Gamma$ enclosing $\sigma(A)$
and not crossing $(-\infty, -1])$.

	For each fixed $t \in [0,1]$, for $z \neq (-\infty,-1]$, we always have $1+tz \neq 0$ and the function 
	$g(z) = z(1+tz)^{-1}$ is analytic on $\Cbb - (-\infty,-1]$. Thus for $\Gamma$ enclosing $\sigma(A)$,
	 not crossing $(-\infty, -1]$,
	\begin{align*}
	A(I+tA)^{-1} = \frac{1}{2\pi i}\oint_{\Gamma}(zI - A)^{-1}z(1+tz)^{-1}dz.
	\end{align*}
Since $\log(1+z)  = \int_{0}^1z(1+tz)^{-1}dt$, we have by Fubini's Theorem for Bochner integral (e.g. \cite{Dunford1988linearOperators}, Theorem III.11.9)
\begin{align*}
\log(I+A) &= \frac{1}{2\pi i}\oint_{\Gamma}(zI - A)^{-1}\log(1+z)dz = \frac{1}{2\pi i}
\oint_{\Gamma}(zI - A)^{-1}\left\{\int_{0}^1z(1+tz)^{-1}dt \right\}dz
\\
& = \frac{1}{2\pi i}\int_{0}^1\left\{\oint_{\Gamma}(zI-A)^{-1}z(1+tz)^{-1}dz\right\}dt = \int_{0}^1A(I+tA)^{-1}dt.
\qed
\end{align*}
\end{proof}





\begin{lemma}
	\label{lemma:logA-norm-p}
	Let $1 \leq p \leq \infty$ be fixed. 
	(i) 
	For $A,B \in \Sym^{+}(\H) \cap \Csc_p(\H)$,
	\begin{align}
	||\log(I+A)||_p \leq ||A||_p, \;\;\;||\log(I+A) - \log(I+B)||_p \leq ||A-B||_p.
	\end{align}
(ii) More generally,
for $A,B \in \Csc_p(\H)$, with $I+A > 0, I+B > 0$, 
\begin{align}
||\log(I+A)||_p &\leq ||A||_p\sup_{t\in [0,1]}||(I+tA)^{-1}||,
\\
||\log(I+A) - \log(I+B)||_p &\leq ||A-B||_p\sup_{t \in [0,1]}[||(I+tA)^{-1}||\;||(I+tB)^{-1}||].
\end{align}
\end{lemma}
\begin{proof}
By Lemma \ref{lemma:log-I+A-integral-representation-general},
$\log(I+A) = \int_{0}^1A(I+tA)^{-1}dt$, $\log(I+B) = \int_{0}^1B(I+tB)^{-1}dt$. 
Thus
\begin{align*}
	&\log(I+A) - \log(I+B) = \int_{0}^1[A(I+tA)^{-1} - B(I+tB)^{-1}]dt
	\\
	& = \int_{0}^1(I+tA)^{-1}[A(I+tB) - (I+tA)B](I+tB)^{-1}dt
	= \int_{0}^1(I+tA)^{-1}(A-B)(I+tB)^{-1}dt.
\end{align*}	
It thus follows that
\begin{align*}
	||\log(I+A) - \log(I+B)||_{p}
	& \leq \int_{0}^1||(I+tA)^{-1}||\;||A-B||_{p}||(I+tB)^{-1}||dt.
	\end{align*}
(i) For $A,B \in \Sym^{+}(\H) \cap \Csc_p(\H)$, $I+tA \geq I, I+tB \geq I$ $\forall t \in [0,1]$ and thus
\begin{align*}
	||\log(I+A) - \log(I+B)||_{p} \leq ||A-B||_{p}.
\end{align*}
(ii) More generally, since $I+A > 0, I +B > 0$, by Lemma \ref{lemma:I+tA},
$I+tA, I+tB$ are invertible
$\forall t \in [0,1]$. 
The results then follow from the integral representations of $\log(I+A)$ and $\log(I+B)$.
\qed
\end{proof}

\begin{lemma}
	\label{lemma:trace-log-square-HS-1}
	(i) For $A,B \in \Sym^{+}(\H)\cap \HS(\H)$,
\begin{align}
\left|\trace[\log(I+A)]^2 - \trace[\log(I+B)]^2\right| \leq ||A-B||_{\HS}[||A||_{\HS} + ||B||_{\HS}].
\end{align}
(ii) For $A,B \in \HS(\H)$ with 
$I+A > 0$, $I+B > 0$,
\begin{align}
	&\left|\trace[\log(I+A)]^2 - \trace[\log(I+B)]^2\right|
	 \leq c_Ac_B||A-B||_{\HS}[c_A||A||_{\HS} + c_B||B||_{\HS}],
\end{align}
where $c_A = \sup_{t \in [0,1]}||(I+tA)^{-1}||$, $c_B = \sup_{t \in [0,1]}||(I+tB)^{-1}$.
\end{lemma}
\begin{proof}
Expanding $|\trace[\log(I+A)]^2 - \trace[\log(I+B)]^2|$ and using the Cauchy-Schwarz Inequality
\begin{align*}
&\left|\trace[\log(I+A)]^2 - \trace[\log(I+B)]^2\right| =
\left|\trace[(\log(I+A) -\log(I+B))(\log(I+A) + \log(I+B))] \right| 
\\ 
&= \left|\la \log(I+A) -\log(I+B), \log(I+A) + \log(I+B)\ra_{\HS}\right|
\\
& \leq ||\log(I+A)-\log(I+B)||_{\HS}[||\log(I+A)||_{\HS} + ||\log(I+B)||_{\HS}]
\end{align*}

(i) For $A,B \in \Sym^{+}(\H) \cap \HS(\H)$, by Lemma \ref{lemma:logA-norm-p}, 
$||\log(I+A)||_{HS}\leq ||A||_{\HS}, ||\log(I+B)||_{\HS} \leq ||B||_{\HS}$,
$||\log(I+A)-\log(I+B)||_{\HS} \leq ||A-B||_{\HS}$, giving the desired bound.

(ii) For $A,B \in \HS(\H)$ with $I+A > 0, I+B > 0$, by Lemma \ref{lemma:logA-norm-p},
$||\log(I+A)||_{\HS} \leq ||A||_{\HS}\sup_{t \in [0,1]}||(I+tA)^{-1}||$, $||\log(I+A) - \log(I+B)||_{\HS}
\leq ||A-B||_{\HS}\sup_{t \in [0,1]}||(I+tA)^{-1}||\sup_{t \in [0,1]}||(I+tB)^{-1}||$, giving the desired bound.
\qed
\end{proof}


\begin{lemma}
	\label{lemma:Ainv-logA-norm-p}
	(i) For $A \in \Sym^{+}(\H)$ compact, $h(A) = A^{-1}\log(I+A)$ is  bounded and admits the integral representation
	\begin{align}
	h(A) = A^{-1}\log(I+A) = \int_{0}^1(I+tA)^{-1}dt.
	\end{align}
	(ii) Let $1 \leq p \leq \infty$ be fixed. For $A,B \in \Sym^{+}(\H) \cap \Csc_p(\H)$, $h(A) -h(B) \in \Sym^{+}(\H)\cap \Csc_p(\H)$, with
\begin{align}
||h(A) - h(B)||_p=||A^{-1}\log(I+A) - B^{-1}\log(I+B)||_p \leq \frac{1}{2}||A-B||_p.
\end{align}
\end{lemma}
\begin{proof}
(i) 
Since $\log(I+A) = A\int_{0}^1(I+tA)^{-1}dt$ by Lemma \ref{lemma:log-I+A-integral-representation-general}, 
$h(A) = A^{-1}\log(I+A) = \int_{0}^1(I+tA)^{-1}
dt$ is bounded.

(ii) For $A,B \in \Sym^{+}(\H)$ compact,
\begin{align*}
&h(A) - h(B) = A^{-1}\log(I+A) - B^{-1}\log(I+B) = \int_{0}^1[(I+tA)^{-1}-(I+tB)^{-1}]dt 
\\
&= \int_{0}^1(I+tA)^{-1}[(I+tB) - (I+tA)](I+tB)^{-1}dt
= \int_{0}^1t(I+tA)^{-1}(B-A)(I+tB)^{-1}dt.
\end{align*}
It follows that for $A,B \in \Sym^{+}(\H)\cap \Csc_p(\H)$, $h(A) - h(B)\in \Sym^{+}(\H)\cap \Csc_p(\H)$, with
\begin{align*}
||A^{-1}\log(I+A) - B^{-1}\log(I+B)||_p &\leq \int_{0}^1||(I+tA)^{-1}||\;||A-B||_p||(I+tB)^{-1}||tdt
\leq \frac{1}{2}||A-B||_p.
\qed
\end{align*}
\end{proof}

\subsection{Proofs for the RKHS Gaussian measures}

\begin{proof}
	[\textbf{Proof of Theorem \ref{theorem:convergence-RKHS-Gaussian}}]
	By Theorem \ref{theorem:logHS-convergence} 
	and Theorem \ref{theorem:CPhi-concentration}, for any $0 < \delta <1$, with probability at least $1-\delta$,
	\begin{align*}
		D^{\gamma}_{\logHS}[\Ncal(0,C_{\Phi(\Xbf)}), \Ncal(0, C_{\Phi})]
		&= ||\log(C_{\Phi(\Xbf)} + \gamma I) - \log(C_{\Phi} + \gamma I)||_{\HS}
		\leq \frac{1}{\gamma}||C_{\Phi(\Xbf)} - C_{\Phi}||_{\HS}
		\\
		& \leq \frac{3\kappa^2}{\gamma}\left(\frac{2\log\frac{4}{\delta}}{m} + \sqrt{\frac{2\log\frac{4}{\delta}}{m}}\right).
	\end{align*}
	By Theorem \ref{theorem:affineHS-convergence} 
	and Theorem \ref{theorem:CPhi-concentration}, for $0 < \ep < \gamma$,
	let $N(\ep) \in \Nbb$ be such that
	${3\kappa^2}\left(\frac{2\log\frac{4}{\delta}}{N(\ep)} + \sqrt{\frac{2\log\frac{4}{\delta}}{N(\ep)}}\right) < \ep$,
	then $\forall m \geq N(\ep)$, with probability at least $1-\delta$,
	\begin{align*}
		&D^{\gamma}_{\aiHS}[\Ncal(0,C_{\Phi(\Xbf)}), \Ncal(0, C_{\Phi})]
		= ||\log[(C_{\Phi(\Xbf)} + \gamma I)^{-1/2}(C_{\Phi} + \gamma I)(C_{\Phi(\Xbf)} + \gamma I)^{-1/2}]||_{\HS}
		\leq \frac{1}{\gamma - \ep}||C_{\Phi(\Xbf)} - C_{\Phi}||_{\HS}
		\\
		& \leq \frac{3\kappa^2}{\gamma-\ep}\left(\frac{2\log\frac{4}{\delta}}{m} + \sqrt{\frac{2\log\frac{4}{\delta}}{m}}\right).
		\qed
	\end{align*}
\end{proof}
\begin{proof}
	[\textbf{Proof of Theorem \ref{theorem:approx-RKHS-Gaussian}}]
	By the triangle inequality,
	\begin{align*}
		&\Delta = \left|||\log(C_{\Phi(\Xbf^1)} + \gamma_1 I) - \log(C_{\Phi(\Xbf^2)} + \gamma_2 I)||_{\eHS}
		- ||\log(C_{\Phi, \rho_1} + \gamma_1I) - \log(C_{\Phi,\rho_2} + \gamma_2 I)||_{\eHS}\right|
		\\
		&\leq ||\log(C_{\Phi(\Xbf^1)} + \gamma_1 I) - \log(C_{\Phi, \rho_1} + \gamma_1I)||_{\eHS}
		+ ||\log(C_{\Phi(\Xbf^2)} + \gamma_2 I) - \log(C_{\Phi, \rho_2} + \gamma_2I)||_{\eHS} = \Delta_1 + \Delta_2.
	\end{align*}
	By Theorem \ref{theorem:convergence-RKHS-Gaussian}, the following sets satisfy $\rho_i(U_i)\geq 1 -\frac{\delta}{2}$, $i =1,2$,
	\begin{align*}
	U_i = \left\{\Xbf^i \in (\Xcal, \rho_i)^m: \Delta_i \leq \frac{3\kappa^2}{\gamma_i}\left(\frac{2\log\frac{8}{\delta}}{m} + \sqrt{\frac{2\log\frac{8}{\delta}}{m}}\right)\right\}.
	\end{align*}
Thus for $(\Xbf^1,\Xbf^2) \in U = (U_1 \times (\Xcal, \rho_2)^m) \cap ((\Xcal, \rho_1)^m \times U_2)$, 
with $(\rho_1 \times \rho_2)^m(U) \geq 1-\delta$, we have
\begin{align*}
\Delta \leq 3\kappa^2\left(\frac{1}{\gamma_1} + \frac{1}{\gamma_2}\right)\left(\frac{2\log\frac{8}{\delta}}{m} + \sqrt{\frac{2\log\frac{8}{\delta}}{m}}\right).
\end{align*}
	The second bound follows similarly, by applying the convergence of $d_{\aiHS}$ in Theorem \ref{theorem:convergence-RKHS-Gaussian}
	%
	and Theorem \ref{theorem:CPhi-concentration}.
	\qed
\end{proof}

\subsection{Proofs for the finite-dimensional orthogonal projections}

We now prove Theorems \ref{theorem:logHS-approx-finite-dim} and \ref{theorem:affineHS-approx-finite-dim}.
\begin{lemma}
	\label{lemma:limit-projection-HS}
	Let 
	$A \in \HS(\H)$. Let $\{e_k\}_{k=1}^{\infty}$ be any orthonormal basis in $\H$. 
	For $N \in \N$ fixed, consider the orthogonal projection operator $P_N = \sum_{k=1}^Ne_k \otimes e_k$. Then
	\begin{align}
		\lim_{N \approach \infty}||P_NAP_N - A||_{\HS} = \lim_{N \approach \infty}||AP_N-A||_{\HS} = \lim_{N\approach \infty}||P_NA - A||_{\HS} = 0.
	\end{align}
\end{lemma}
\begin{proof}
We have $P_Ne_k = e_k$ for $1 \leq k \leq N$ and $P_Ne_k = 0$ for $k \geq N+1$. 
Since $||A||^2_{\HS} = \sum_{k=1}^{\infty}||Ae_k||^2 < \infty$,
\begin{align*}
||AP_N - A||_{\HS}^2 = \sum_{k=1}^{\infty}||(AP_N-A)e_k||^2 = \sum_{k=N+1}^{\infty}||Ae_k||^2 \approach 0
\end{align*}
as $N \approach \infty$. Similarly, since $||A^{*}||_{\HS} = ||A||_{\HS}$,
\begin{align*}
||P_NA-A||^2_{\HS} = ||A^{*}P_N - A^{*}||^2_{\HS} = \sum_{k=N+1}^{\infty}||A^{*}e_k||^2 \approach 0
\end{align*} 
as $N \approach \infty$. It follows that
\begin{align*}
||P_NAP_N - A||_{\HS} &\leq ||(P_NA - A)||_{\HS}||P_N|| + ||AP_N - A||_{\HS}
\leq ||P_NA-A||_{\HS} + ||AP_N - A||_{\HS} \approach 0
\end{align*}
as $N \approach \infty$.\qed
\end{proof}

\begin{lemma}
	\label{lemma:limit-projection-trace}
	Let 
	$A \in \Tr(\H)$. Let $\{e_k\}_{k=1}^{\infty}$ be any orthonormal basis in $\H$. 
	For $N \in \N$ fixed, consider the orthogonal projection operator $P_N = \sum_{k=1}^Ne_k \otimes e_k$. Then
	\begin{align}
		\lim_{N \approach \infty}||P_NAP_N - A||_{\tr} = \lim_{N \approach \infty}||AP_N-A||_{\tr} = \lim_{N\approach \infty}||P_NA - A||_{\tr} = 0.
	\end{align}
\end{lemma}
\begin{proof}
	We recall the following properties relating the Banach space $(\Csc_1(\H), ||\;||_1) = (\Tr(\H), ||\;||_{\trace})$
	of trace class operators and the Hilbert space $(\Csc_2(\H), ||\;||_2) = (\HS(\H), ||\;||_{\HS})$
	of Hilbert-Schmidt operators on $\H$ (see e.g. \cite{ReedSimon:Functional})
	\begin{enumerate}
		\item $A \in \Csc_1(\H)$ if and only if $A = BC$, for some operators $B,C \in \Csc_2(\H)$.
		\item $||BC||_{1} \leq ||B||_2||C||_2$.
	\end{enumerate}
	Given that $A \in \Tr(\H) = \Csc_1(\H)$, we then write $A = BC$, for some operators $B, C \in \HS(\H) = \Csc_2(\H)$. Then
	\begin{align*}
		&||P_NAP_N - A||_{\trace} = ||P_NBCP_N - BC||_{\trace} 
		= ||P_NBCP_N - P_NBC + P_NBC - BC||_{\trace} 
		\\
		&= ||P_NB(CP_N -C) + (P_NB - B)C||_{\trace} 
		\leq ||P_NB||_{\HS}||CP_N-C||_{\HS} + ||P_NB-B||_{\HS}||C||_{\HS}.
	\end{align*}
	By Lemma \ref{lemma:limit-projection-HS}, 
		$\lim_{N \approach \infty}||P_NB-B||_{\HS} = 0, \;\; \lim_{N \approach \infty}||CP_N - C||_{\HS} = 0$.
	It thus follows that
$		\lim_{N \approach \infty}||P_NAP_N-A||_{\trace} = 0$.
	For the second limit, 
	\begin{align*}
		& ||P_NA - A||_{\trace} = ||P_NBC-BC||_{\trace} = ||(P_NB-B)C||_{\trace} 
		\leq ||P_NB-B||_{\HS}||C||_{\HS} \approach 0 \;\;\text{as}\;\; N \approach \infty.
	\end{align*}
	The third limit is proved similarly.\qed
\end{proof}

\begin{lemma}
	\label{lemma:finite-matrix-representation}
	Let $A\in \Sym(\H)$ be compact, with $I+A > 0$.
	Let $\{e_k\}$ be any orthonormal basis in $\H$. Let $P_N = \sum_{k=1}^Ne_k \otimes e_k$, $N \in \Nbb$ fixed.
	Let $\Abf_N$ be the matrix representation of $P_NAP_N|_{\H_N}$ in the basis $\{e_k\}_{k=1}^N$,
	where $\H_N = \myspan\{e_k\}_{k=1}^N$. Then 
	\begin{enumerate}
		\item The matrix representation of $\log(I+P_NAP_N|_{\H_N})$ in the basis $\{e_k\}_{k=1}^N$ is $\log(I+\Abf_N)$.
		\item The matrix representation of $(I+ P_NAP_N|_{\H_N})^{\alpha}$ in the basis $\{e_k\}_{k=1}^N$ is $(I+\Abf_N)^{\alpha}$ $\forall \alpha \in \R$.
	\end{enumerate}
\end{lemma}
\begin{proof}
	Since $I+A > 0$ and $||P_NAP_N||\leq ||A|||$, we have $I+P_NAP_N > 0$, so that
	$\log(I+P_NAP_N)$ and $(I+P_NAP_N)^{\alpha}$ are well-defined $\forall \alpha \in \R$.
	Since $P_NAP_N:\H \mapto \H_N$,
	it has  rank at most $N$.
	Let $\{\lambda_k^{A,N}, \phi_{k}^{A,N}\}_{k=1}^{\infty}$
	be the corresponding spectrum, with eigenvalues $\lambda_k^{A,N} = 0$, $k \geq N+1$, and normalized eigenvectors
	$\{\phi_k^{A,N}\}_{k=1}^{\infty}$ 
	forming an orthonormal basis of $\H$, 
	with $\{\phi_k^{A,N}\}_{k=1}^{N}$
	forming an orthonormal basis in the subspace $\H_N$.
	Then
	\begin{align*}
		&P_NAP_N = \sum_{k=1}^N\lambda_k^{A,N}\phi_k^{A,N} \otimes \phi_k^{A,N}, \;\;
		\\
		&\log(I+ P_NAP_N) = \sum_{k=1}^N\log(1+\lambda_k^{A,N})\phi_k^{A,N} \otimes \phi_k^{A,N}: \H \mapto \H_N.
	\end{align*}
The operator $\log(I+P_NAP_N)$ also has rank at most $N$, with all non-zero eigenvalues corresponding to 
eigenvectors lying in $\H_N$.
	Since $\Abf_N$ is the matrix representation of $P_NAP_N|_{\H_N}$ in the basis $\{e_k\}_{k=1}^N$, the eigenvalues of $\Abf_N$ are precisely
	$\{\lambda_k^{A,N}\}_{k=1}^N$, and
	for $1 \leq j,k \leq N$,
	\begin{align*}
		(\Abf_N)_{kj} &= \la Ae_k,e_j\ra = \la e_k, P_NAP_N e_j\ra = \sum_{l=1}^N\lambda_l^{A,N}\la \phi_l^{A,N}, e_k\ra \la \phi_l^{A,N}, e_j\ra,\;\;
		\Abf_N = U^{*}\diag(\lambda_1^{A,N}, \ldots, \lambda_N^{A,N})U,
	\end{align*}
	where $U$ is the $N \times N$ matrix with $U_{ij} = \la \phi^{A,N}_i, e_j\ra$. $U$ is orthonormal, since
	\begin{align*}
		\sum_{l=1}^NU_{li}U_{lj} = \sum_{l=1}^N\la \phi^{A,N}_l, e_i\ra\la \phi^{A,N}_l, e_j\ra = \la e_i, e_j\ra = \delta_{ij}, \;\;\; 1 \leq i,j \leq N.
	\end{align*}
	Thus $\log(I+\Abf_N)$ is the $N \times N$ symmetric matrix given by
	\begin{align*}
		\log(I+\Abf_N) = U^{*}\diag(\log(1+\lambda_1^{A,N}), \ldots, \log(1+\lambda_N^{A,N}))U.
	\end{align*}
	On the other hand, from the spectral representation of $P_NAP_N$,
	\begin{align*}
		\la e_k, \log(I+ P_NAP_N)e_j\ra &= \sum_{l=1}^N\log(1+\lambda_l^{A,N})\la \phi_l^{A,N}, e_k\ra \la \phi_l^{A,N}, e_j\ra
		= [\log(I+\Abf_N)]_{kj}.
	\end{align*}
	It follows that in the basis $\{e_k\}_{k=1}^N$ of $\H_N$, the matrix representation of $\log(I+P_NAP_N|_{\H_N})$ is $\log(I+\Abf_N)$.
	Similarly, the matrix representation of $(I+P_NAP_N|_{\H_N})^{\alpha}$ is $(I+\Abf_N)^{\alpha}$
	$\forall \alpha \in \R$.
	\qed
\end{proof}

\begin{proof}
	[\textbf{Proof of Theorem \ref{theorem:logHS-approx-finite-dim}}]
	Let $N \in \N$ be fixed. By Lemma \ref{lemma:limit-projection-HS},
	\begin{align*}
		&\lim_{N \approach \infty}||P_NAP_N -A||_{\HS} = 0,
		\lim_{N \approach \infty}||P_NBP_N - B||_{\HS} = 0.
	\end{align*} 
With $I+A >0, I+B > 0$ we have $I+P_NAP_N > 0, I+P_NBP_N > 0$ $\forall N \in \Nbb$. By Theorem \ref{theorem:logHS-convergence},
	\begin{align*}
&	\lim_{N \approach \infty}||\log(I+P_NAP_N) - \log(I+A)||_{\HS} =0,
\lim_{N \approach \infty}||\log(I+P_NBP_N) - \log(I+B)||_{\HS} = 0.
	\end{align*}
By the triangle inequality,
\begin{align*}
&\left|||\log(I+P_NAP_N) - \log(I+P_NBP_N)||_{\HS} - ||\log(I+A) - \log(I+B)||_{\HS}\right|
\\
& \leq ||\log(I+P_NAP_N) - \log(I+A)||_{\HS} + ||\log(I+P_NBP_N) - \log(I+B)||_{\HS}
\approach 0 \;\;\;\text{as } N \approach \infty.
\end{align*}
	By Lemma \ref{lemma:finite-matrix-representation},
	in the basis $\{e_k\}_{k=1}^N$ of the subspace $\H_N$,  the matrix representations of 
	$\log(I+P_NAP_N|_{\H_N})$ and $\log(I+P_NBP_N|_{\H_N})$ are $\log(I+\Abf_N)$ and $\log(I+\Bbf_N)$, respectively.
	From the integral representation $\log(I+A) = A\int_{0}^1(I+tA)^{-1}dt$ in Lemma \ref{lemma:log-I+A-integral-representation-general}, we have $P_NAP_Ne_k = 0\imply
	\log(I+P_NAP_N)e_k = 0$ $\forall k \geq N+1$. 
	Thus
	\begin{align*}
	&||\log(I+P_NAP_N) - \log(I+P_NBP_N)||^2_{\HS}
	= \sum_{k=1}^{\infty}||[\log(I+P_NAP_N) - \log(I+P_NBP_N)]e_k||^2 
	\\
	&=  \sum_{k=1}^{N}||[\log(I+P_NAP_N) - \log(I+P_NBP_N)]e_k||^2
	= ||\log(I+P_NAP_N)|_{\H_N} - \log(I+P_NBP_N)|_{\H_N}||^2_{\HS}
	\\
	& = ||\log(I+\Abf_N) - \log(I+\Bbf_N)||^2_{F}.
	\end{align*}
It follows that $\lim_{N \approach \infty}||\log(I+\Abf_N) - \log(I+\Bbf_N)||_F = 
||\log(I+A) - \log(I+B)||_{\HS}$. For $\gamma \in \R$, $\gamma > 0$, we note that
\begin{align*}
||\log(A+\gamma I) - \log(B + \gamma I)||_{\HS} = \left\|\log\left(I+\frac{A}{\gamma}\right) - \log\left(I + \frac{B}{\gamma}\right)\right\|_{\HS}.
\end{align*}
Thus this case reduces to the previous case.
\qed
\end{proof}

\begin{proof}
[\textbf{Proof of Theorem \ref{theorem:affineHS-approx-finite-dim}}]
Let $N \in \Nbb$ fixed. By Lemma \ref{lemma:limit-projection-HS},
\begin{align*}
\lim_{N \approach \infty}||P_NAP_N -A||_{\HS} =0, \lim_{N \approach \infty}||P_NBP_N - B||_{\HS} = 0.
\end{align*}
With $I+A>0, I+B >0$, we have $I+P_NAP_N >0, I+P_NBP_N>0$ $\forall N \in \Nbb$. Thus by Theorem 
\ref{theorem:affineHS-convergence},
\begin{align*}
&\lim_{N \approach \infty}||\log[(I+P_NBP_N)^{-1/2}(I+P_NAP_N)(I+P_NBP_N)^{-1/2}]||_{\HS}
= ||\log[(I+B)^{-1/2}(I+A)(I+B)^{-1/2}]||_{\HS}.
\end{align*}
By Lemma \ref{lemma:finite-matrix-representation}, 
in the basis $\{e_k\}_{k=1}^N$ of the subspace $\H_N = \myspan\{e_k\}_{k=1}^{N}$,
the matrix representation of the operator $\log[(I+P_NBP_N|_{\H_N})^{-1/2}(I+P_NAP_N|_{\H_N})(I+P_NBP_N|_{\H_N})^{-1/2}]
:\H_N \mapto \H_N$ is $\log[(I+\Bbf_N)^{-1/2}(I+\Abf_N)(I+\Bbf_N)^{-1/2}]$.

Write $(I+P_NBP_N)^{-1/2}(I+P_NAP_N)(I+P_NBP_N)^{-1/2} = I + (I+P_NBP_N)^{-1/2}(P_NAP_N - P_NBP_N)
(I+P_NBP_N)^{-1/2}$. 
Since $P_NAP_Ne_k = P_NBP_Ne_k = 0$ $\forall k \geq N+1$, 
we have $\forall k \geq N+1$, $(I + P_NBP_N)e_k = e_k \imply (I+P_NBP_N)^{-1/2}e_k = e_k$ $\imply 
(P_NAP_N - P_NBP_N)
(I+P_NBP_N)^{-1/2}e_k = 0$. Thus $\forall k \geq N+1$,
\begin{align*}
(I+P_NBP_N)^{-1/2}(P_NAP_N - P_NBP_N)
(I+P_NBP_N)^{-1/2}e_k = 0.
\end{align*}
From the integral representation $\log(I+A) = A\int_{0}^1(I+tA)^{-1}dt$ in Lemma \ref{lemma:log-I+A-integral-representation-general}, we then have
$(\log[(I+P_NBP_N)^{-1/2}(I+P_NAP_N)(I+P_NBP_N)^{-1/2}])e_k = 0$ $\forall k \geq N+1$. Thus
\begin{align*}
&||\log[(I+P_NBP_N)^{-1/2}(I+P_NAP_N)(I+P_NBP_N)^{-1/2}]||_{\HS}^2
\\
&= \sum_{k=1}^{\infty}||(\log[(I+P_NBP_N)^{-1/2}(I+P_NAP_N)(I+P_NBP_N)^{-1/2}])e_k||^2
\\
&= \sum_{k=1}^N||(\log[(I+P_NBP_N)^{-1/2}(I+P_NAP_N)(I+P_NBP_N)^{-1/2}])e_k||^2
\\
& = ||\log[(I+P_NBP_N|_{\H_N})^{-1/2}(I+P_NAP_N|_{\H_N})(I+P_NBP_N|_{\H_N})^{-1/2}]||_{\HS}^2
\\
& = ||\log[(I+\Bbf_N)^{-1/2}(I+\Abf_N)(I+\Bbf_N)^{-1/2}]||^2_F.
\end{align*}
Combining this with the previous limit gives
\begin{align*}
&\lim_{N \approach \infty}||\log[(I+\Bbf_N)^{-1/2}(I+\Abf_N)(I+\Bbf_N)^{-1/2}]||_F
= ||\log[(I+B)^{-1/2}(I+A)(I+B)^{-1/2}]||_{\HS}.
\end{align*}
For $\gamma \in \R$, $\gamma > 0$, we note that
\begin{align*}
&||\log[(B+\gamma I)^{-1/2}(A+\gamma I)(B+\gamma I)^{-1/2}]||_{\HS}
= \left\|\log\left[\left(\frac{B}{\gamma}+ I\right)^{-1/2}\left(\frac{A}{\gamma}+ I\right)\left(\frac{B}{\gamma}+ I\right)^{-1/2}\right]\right\|_{\HS}.
\end{align*}
Thus this case reduces to the previous case.
\qed
\end{proof}

\subsection{Further technical results}

\begin{theorem}
	\label{theorem:affineHS-sequence-approx-1}
	Let
	$A,B, \{A_n\}_{n \in \N}$, $\{B_n\}_{n \in \N} \in \Sym(\H) \cap \HS(\H)$
	be such that $\lim_{n \approach \infty}||A_n - A||_{\HS} = 0$, 
	$\lim_{n \approach \infty}||B_n - B||_{\HS} = 0$.
	Assume that $(I+A) > 0, (I+B) > 0, I+A_n > 0, I+B_n > 0$ $\forall n \in \N$. Then
	$(I+B_n)^{-1/2}(I+A_n)(I+B_n)^{-1/2}  - I$,$(I+B)^{-1/2}(I+A)(I+B)^{-1/2} - I$ are in $\Sym(\H) \cap \HS(\H)$.
	
	(i) If $A,B,A_n,B_n \in \Sym^{+}(\H) \cap \HS(\H)$ $\forall n \in \Nbb$, 
	let $M_{AB} > 0$ be such that $\la x, (I+B)^{-1/2}(I+A)(I+B)^{-1/2} x\ra \geq M_{AB}||x||^2$ $\forall x \in \H$.
	Then $\forall \epsilon, 0 < \epsilon < M_{AB}$, $\exists N_{AB}(\epsilon) \in \Nbb$ such that
	$\forall n \geq N_{AB}$,
	\begin{align}
		&||\log[(I+B_n)^{-1/2}(I+A_n)(I+B_n)^{-1/2}]  
		-\log[(I+B)^{-1/2}(I+A)(I+B)^{-1/2}]||_{\HS} 
		\nonumber
		\\
		&\leq \frac{1}{M_{AB} - \epsilon}[||A_n - A||_{\HS} + (1+||A||_{\HS})||B_n - B||_{\HS}].
	\end{align}
	
	(ii) In general, let $M_B > 0$ be such that $\la x, (I+B)x\ra \geq M_B||x||^2$ $\forall x \in \H$,
	then $\forall \epsilon, 0 < \epsilon < \min\{M_B, M_{AB}\}$, $\exists N(\epsilon)$ such that $\forall n \geq N(\epsilon)$,
	\begin{align}
		&||\log[(I+B_n)^{-1/2}(I+A_n)(I+B_n)^{-1/2}]  
		-\log[(I+B)^{-1/2}(I+A)(I+B)^{-1/2}]||_{\HS} 
		\\
		&\leq \frac{1}{(M_{AB} - \epsilon)(M_B-\epsilon)}\left[||A_n - A||_{\HS} + \frac{1}{M_B-\epsilon}\left(1 + ||A||_{\HS}\right)||B_n - B||_{\HS} \right].
		\nonumber
	\end{align}
	In both cases, the following convergence holds
	\begin{align}
		\lim_{n \approach \infty}||&\log[(I+B_n)^{-1/2}(I+A_n)(I+B_n)^{-1/2}]  
		-\log[(I+B)^{-1/2}(I+A)(I+B)^{-1/2}]||_{\HS} = 0.
	\end{align}
\end{theorem}

\begin{corollary}
	\label{corollary:affineHS-sequence-approx-1}
	Let
	$A,B, \{A_n\}_{n \in \N}$, $\{B_n\}_{n \in \N} \in \Sym^{+}(\H) \cap \HS(\H)$
	be such that $\lim_{n \approach \infty}||A_n - A||_{\HS} = 0$, 
	$\lim_{n \approach \infty}||B_n - B||_{\HS} = 0$.
	Let $\gamma \in \R, \gamma > 0$ be fixed.
	Then
	$(\gamma I+B_n)^{-1/2}(\gamma I+A_n)(\gamma I+B_n)^{-1/2}  - I$,$(\gamma I+B)^{-1/2}(\gamma I+A)(\gamma I+B)^{-1/2} - I$ are in $\Sym(\H) \cap \HS(\H)$.
	Let $M_{AB} > 0$ be such that $\la x, (\gamma I+B)^{-1/2}(\gamma I+A)(\gamma I+B)^{-1/2} x\ra \geq M_{AB}||x||^2$ $\forall x \in \H$.
	Then $\forall \epsilon, 0 < \epsilon < M_{AB}$, $\exists N_{AB}(\epsilon) \in \Nbb$ such that
	$\forall n \geq N_{AB}$,
	\begin{align}
		&||\log[(\gamma I+B_n)^{-1/2}(\gamma I+A_n)(\gamma I+B_n)^{-1/2}]  
		-\log[(\gamma I+B)^{-1/2}(\gamma I+A)(\gamma I+B)^{-1/2}]||_{\HS} 
		\nonumber
		\\
		&\leq \frac{1}{(M_{AB} - \epsilon)\gamma}\left[||A_n - A||_{\HS} + \left(1+\frac{1}{\gamma}||A||_{\HS}\right)||B_n - B||_{\HS}\right].
	\end{align}
\end{corollary}

\begin{lemma}
	\label{lemma:HS-inverse}
	Let $\{A_n\}_{n\in \N}$, $A$ $\in \Sym(\H) \cap \HS(\H)$ such that
	$(I+A) > 0$, $(I+A_n) > 0$ 
	$\forall n \in \N$. 
	Then $A_n(I+A_n)^{-1}$ and $A(I+A)^{-1}$ are 
	in $\Sym(\H) \cap \HS(\H)$.
	
	(i) If $A, A_n \in \Sym^{+}(\H) \cap \HS(\H)$ $\forall n \in \Nbb$, then
	\begin{align}
		||A_n(I+A_n)^{-1} - A(I+A)^{-1}||_{\HS} \leq ||A_n -A||_{\HS} \;\;\;\forall n \in \Nbb.
	\end{align}	
	
	(ii) Assume that $\lim_{n \approach \infty}||A_n - A|| = 0$. 
	Let $M_A >0$ be such that $\la x, (I+A)x\ra \geq  M_A||x||^2$ $\forall x \in \H$,
	then $\forall \epsilon$, $0 < \epsilon < M_A$, $\exists N(\epsilon) \in \Nbb$ such that
	\begin{align}
		||A_n(I+A_n)^{-1} - A(I+A)^{-1}||_{\HS} \leq \frac{1}{M_A(M_A-\epsilon)}||A_n - A||_{\HS} \;\;\forall n \geq N(\epsilon).
	\end{align}
	In both cases, $\lim_{n\approach \infty}||A_n - A||_{\HS} = 0$ implies
	$\lim_{n \approach \infty}||A_n(I+A_n)^{-1} - A(I+A)^{-1}||_{\HS} = 0$. 
\end{lemma}
\begin{proof} 
	Since $A_n, A \in \Sym(\H)\cap \HS(\H)$, both $A_n(I+A_n)^{-1}$ and  $A(I+A)^{-1}$ are in $\Sym(\H) \cap \HS(\H)$.
	We have
	\begin{align*}
		&||A_n(I+A_n)^{-1} - A(I+A)^{-1}||_{\HS} = ||(I+A_n)^{-1}A_n - A(I+A)^{-1}||_{\HS}
		\\
		&= ||(I+A_n)^{-1}[A_n(I+A) - (I+A_n)A](I+A)^{-1}||_{\HS}
		\\
		& = ||(I+A_n)^{-1}[A_n - A](I+A)^{-1}||_{\HS} \leq ||(I+A_n)^{-1}||\;||A_n - A||_{\HS}\;||(I+A)^{-1}||.
	\end{align*}
	
	(i) If $A_n, A\in \Sym^{+}(\H)\cap \HS(\H)$, then $||(I+A_n)^{-1}||\leq 1$, $||(I+A)^{-1}||\leq 1$, so that
	\begin{align*}
		&||A_n(I+A_n)^{-1} - A(I+A)^{-1}||_{\HS} \leq ||A_n - A||_{\HS}.
	\end{align*}
	
	(ii) In the general case, 
	by the assumption $\lim_{n \approach \infty}||A_n -A|| = 0$, for any $\epsilon$ satisfying $0 < \epsilon < M_A$, there
	exists $N = N(\epsilon) \in \N$ such that $||A_n - A||  < \epsilon$ $\forall n \geq N$. Therefore,
	$I+A_n = I+A + A_n - A \geq (M_A -\epsilon)I \;\;\;\forall n \geq N$.
	Thus we have $I+A \geq M_A > 0$, $I+A_n \geq M_A-\epsilon > 0$ $\forall n \geq N = N(\epsilon)$, from which it follows that
	\begin{align*}
		||(I+A_n)^{-1}|| \leq \frac{1}{M_A -\epsilon} \forall N \geq N(\epsilon), \;\; ||(I+A)^{-1}|| \leq \frac{1}{M_A}.
	\end{align*}
	Combining this with the first inequality, we have 
	\begin{align*}
		&||A_n(I+A_n)^{-1} - A(I+A)^{-1}||_{\HS} \leq \frac{1}{M_A(M_A-\epsilon)}||A_n -A||_{\HS} \forall n \geq N,
	\end{align*}
	which implies that
	$\lim_{n \approach \infty}||A_n(I+A_n)^{-1} - A(I+A)^{-1}||_{\HS} = 0$.
	\qed
\end{proof}

\begin{proposition}
	\label{proposition:HS-inverse-2}
	Let $\{A_n\}_{n \in \N}$, $A$, $\{B_n\}_{n \in \N}$, $B$ $\in \Sym(\H)\cap \HS(\H)$.
	Assume that $I+A > 0, I+B > 0, I+A_n > 0, I+ B_n > 0$ $\forall n \in \N$. Then
	$(I+B_n)^{-1/2}(I+A_n)(I+B_n)^{-1/2}-I$ and $(I+B)^{-1/2}(I+A)(I+B)^{-1/2}-I$ are in $\Sym(\H)\cap \HS(\H)$.
	
	(i) If $A,B, A_n , B_n \in \Sym^{+}(\H) \cap \HS(\H)$ $\forall n \in \Nbb$, then
	\begin{align}
		\label{equation:inequality-ABA-1}
		&||(I+B_n)^{-1/2}(I+A_n)(I+B_n)^{-1/2} - (I+B)^{-1/2}(I+A)(I+B)^{-1/2}||_{\HS}
		\nonumber 
		\\
		&\leq ||A_n - A||_{\HS} + (1+||A||_{\HS})||B_n - B||_{\HS}.
	\end{align}
	
	(ii) If 
	$\lim_{n \approach \infty}||B_n - B|| = 0$, let $M_B > 0$ be such that
	$\la x, (I+B)x\ra \geq M_B||x||^2$ $\forall x \in \H$. Then
	$\forall \epsilon > 0$, $0 < \epsilon < M_B$, $\exists N_B(\epsilon) \in \Nbb$ such that $\forall n \geq N_B$,
	\begin{align}
		\label{equation:inequality-ABA-2}
		&||(I+B_n)^{-1/2}(I+A_n)(I+B_n)^{-1/2} - (I+B)^{-1/2}(I+A)(I+B)^{-1/2}||_{\HS}
		\nonumber
		\\
		& \leq \frac{1}{M_B-\epsilon}\left[||A_n - A||_{\HS} + \frac{1}{M_B-\epsilon}\left(1 + ||A||_{\HS}\right)||B_n - B||_{\HS} \right].	
	\end{align}
	In both cases, $\lim_{n \approach 0}||A_n - A||_{\HS} = \lim_{n \approach 0}||B_n - B||_{\HS} = 0$ implies
	\begin{align}
		\label{equation:limit-ABA-1}
		&\lim_{n \approach \infty}||(I+B_n)^{-1/2}(I+A_n)(I+B_n)^{-1/2} - (I+B)^{-1/2}(I+A)(I+B)^{-1/2}||_{\HS}
		= 0.
	\end{align}
\end{proposition}
\begin{proof}
	Let $C_n = - B_n(I+B_n)^{-1} + (I+B_n)^{-1/2}A_n(I+B_n)^{-1/2}$, $C= - B(I+B)^{-1} + (I+B)^{-1/2}A(I+B)^{-1/2}$, which are in $\Sym(\H)\cap \HS(\H)$, then
	\begin{align*}
		(I+B_n)^{-1/2}(I+A_n)(I+B_n)^{-1/2} &= I+C_n,
		(I+B)^{-1/2}(I+A)(I+B)^{-1/2} =  I+C,
	\end{align*}
	
	(i) If $A_n, A, B_n, B\in \Sym^{+}(\H) \cap \HS(\H)$, then we have by Lemma \ref{lemma:HS-inverse}
	\begin{align*}
		||B_n(I+B_n)^{-1} - B(I+B)^{-1}||_{\HS} \leq ||B_n - B||_{\HS}.
	\end{align*}
	
	Consider the difference between the second terms of $C_n$ and $C$, i.e.
	\begin{align}
		\label{equation:lemma-trace-inverse-2}
		&||(I+B_n)^{-1/2}A_n(I+B_n)^{-1/2} - (I+B)^{-1/2}A(I+B)^{-1/2}||_{\HS}
		\nonumber
		\leq ||(I+B_n)^{-1/2}A_n(I+B_n)^{-1/2} - (I+B_n)^{-1/2}A(I+B_n)^{-1/2}||_{\HS}
		\nonumber
		\\
		&+ ||(I+B_n)^{-1/2}A(I+B_n)^{-1/2} - (I+B_n)^{-1/2}A(I+B)^{-1/2}||_{\HS}
		\nonumber
		\\
		&
		+ ||(I+B_n)^{-1/2}A(I+B)^{-1/2} - (I+B)^{-1/2}A(I+B)^{-1/2}||_{\HS}.
	\end{align}
	The first term on the right hand side of the inequality in Eq.~(\ref{equation:lemma-trace-inverse-2}) is
	\begin{align*}
		||(I+B_n)^{-1/2}(A_n-A)(I+B_n)^{-1/2}||_{\HS} &\leq ||A_n - A||_{\HS}||(I+B_n)^{-1/2}||^2 
		\leq ||A_n-A||_{\HS}.
	\end{align*}
	By Corollary \ref{corollary:Schatten-class-rth-root}, since $||\;||\leq ||\;||_{\HS}$,
	\begin{align*}
		||(I+B_n)^{-1/2} - (I+B)^{-1/2}|| &\leq ||(I+B_n)^{-1/2} - (I+B)^{-1/2}||_{\HS} 
		\leq \frac{1}{2}||B_n - B||_{\HS}.
	\end{align*}
	Thus the second term in Eq.~(\ref{equation:lemma-trace-inverse-2}) satisfies
	\begin{align*}
		&||(I+B_n)^{-1/2}A[(I+B_n)^{-1/2} - (I+B)^{-1/2}]||_{\HS} 
		\\
		&\leq ||(I+B_n)^{-1/2}||\;||A||_{\HS}||(I+B_n)^{-1/2} - (I+B)^{-1/2}]||\leq \frac{1}{2}||A||_{\HS}||B_n-B||_{\HS}.
	\end{align*}
	Similarly, the third term in Eq.~(\ref{equation:lemma-trace-inverse-2}) satisfies
	\begin{align*}
		&||(I+B_n)^{-1/2}A(I+B)^{-1/2} - (I+B)^{-1/2}A(I+B)^{-1/2}||_{\HS}
		\\ 
		&\leq ||A(I+B)^{-1/2}||_{\HS}||[(I+B_n)^{-1/2} - (I+B)^{-1/2}]|| \leq \frac{1}{2}||A||_{\HS}||B_n -B||_{\HS}.
	\end{align*}
	The inequality in Eq.\eqref{equation:inequality-ABA-1} is obtained by combining all the above inequalities.
	
	(ii) In general, by the assumption
	$I+B > 0$, 
	as in the proof of Lemma \ref{lemma:HS-inverse}, 
	since 
	$\lim_{n \approach \infty}||B_n - B||= 0$, 
	for any $0 < \epsilon < M_B$, 
	$\exists N_B = N_B(\epsilon) \in \N$, such that $||B_n - B||< \epsilon$ $\forall n \geq N_B$ and 
	\begin{align*}
		I+B_n  = I+B +B_n - B\geq (M_B -\epsilon)I\;\;\; \forall n \geq N_B.
	\end{align*}
	By Lemma \ref{lemma:HS-inverse}, $\forall n \geq N_B$.
	\begin{align*}
		||B_n(I+B_n)^{-1} - B(I+B)^{-1}||_{\HS} \leq \frac{1}{M_B(M_B-\epsilon)}||B_n - B||_{\HS}. \;\;
	\end{align*}
	The first term on the right hand side of the inequality in Eq.~(\ref{equation:lemma-trace-inverse-2}) is
	\begin{align*}
		||(I+B_n)^{-1/2}(A_n-A)(I+B_n)^{-1/2}||_{\HS} &\leq ||A_n - A||_{\HS}||(I+B_n)^{-1/2}||^2 
		\leq \frac{||A_n - A||_{\HS}}{M_B - \epsilon}.
	\end{align*}
	for all $n \geq N_B$. 
	By Corollary \ref{corollary:limit-(I+A)r}, $\forall n \geq N_B$,
	\begin{align*}
		||[(I+B_n)^{-1/2} - (I+B)^{-1/2}]|| \leq \frac{1}{2(M_B-\epsilon)M_B^{1/2}}||B_n - B||.
	\end{align*}
	Thus the second term in Eq.~(\ref{equation:lemma-trace-inverse-2}) satisfies
	\begin{align*}
		&||(I+B_n)^{-1/2}A[(I+B_n)^{-1/2} - (I+B)^{-1/2}]||_{\HS} 
		\leq \frac{||A||_{\HS} ||[(I+B_n)^{-1/2} - (I+B)^{-1/2}]||}{\sqrt{M_B-\epsilon}}
		\\
		&\leq \frac{1}{2(M_B-\epsilon)^{3/2}M_B^{1/2}}||A||_{\HS}||B_n - B||_{\HS} \;\;\forall n \geq N_B.
	\end{align*}
	Similarly, for the third term in Eq.~(\ref{equation:lemma-trace-inverse-2}), we have
	\begin{align*}
		&||(I+B_n)^{-1/2}A(I+B)^{-1/2} - (I+B)^{-1/2}A(I+B)^{-1/2}||_{\HS} 
		\\
		&\leq ||A(I+B)^{-1/2}||_{\HS}||[(I+B_n)^{-1/2} - (I+B)^{-1/2}]||
		\leq \frac{1}{2(M_B-\epsilon)M_B}||A||_{\HS}||B_n - B||_{\HS} \;\;\forall n \geq N_B.
	\end{align*}
	Combining all of the above inequalities, we obtain
	\begin{align*}
		&||C_n - C||_{\HS} \leq \frac{1}{M_B-\epsilon}||A_n - A||_{\HS} 
		+ \frac{1}{M_B-\epsilon}\left(\frac{1}{M_B} + \frac{||A||_{\HS}}{2(M_B-\epsilon)^{1/2}M_B^{1/2}} + \frac{||A||_{\HS}}{2M_B}\right)||B_n - B||_{\HS} 
		\\
		&\leq
		\frac{1}{M_B-\epsilon}\left[||A_n - A||_{\HS} + \frac{1}{M_B-\epsilon}\left(1 + ||A||_{\HS}\right)||B_n - B||_{\HS} \right],
	\end{align*}
	$\forall n \geq N_B$, which is Eq.\eqref{equation:inequality-ABA-2}.
	\qed
\end{proof}

\begin{proof}
	[\textbf{Proof of Theorem \ref{theorem:affineHS-sequence-approx-1}}]
	As in Proposition \ref{proposition:HS-inverse-2}, we write
	\begin{align*}
		(I+B_n)^{-1/2}(I+A_n)(I+B_n)^{-1/2} &= I + C_n,
		(I+B)^{-1/2}(I+A)(I+B)^{-1/2} = I+C,
	\end{align*}
	where $C_n = - B_n(I+B_n)^{-1} + (I+B_n)^{-1/2}A_n(I+B_n)^{-1/2}$, $C = - B(I+B)^{-1} + (I+B)^{-1/2}A(I+B)^{-1/2} \in \Sym(\H) \cap \HS(\H)$.	
	By Proposition \ref{proposition:HS-inverse-2}, since $\lim_{n \approach \infty}||A_n-A||_{\HS}
	=\lim_{n \approach \infty}||B_n-B||_{\HS} =0$, we have $\lim_{n \approach \infty}||C_n - C||_{\HS} = 0$.
	Let $M_{AB} > 0$ be such that $\la x, (I+C)x\ra\geq M_{AB}||x||^2$, then by Theorem 
	\ref{theorem:logHS-convergence},
	$\forall \epsilon, 0 < \epsilon < M_{AB}$, $\exists N_{AB}(\epsilon) \in \Nbb$ such that
	$||C_n - C|| < \ep$ $\forall n \geq N_{AB}$ and
	\begin{align*}
		||\log(I+C_n) - \log(I+C)||_{\HS} \leq \frac{1}{M_{AB} - \epsilon}||C_n - C||_{\HS}
		\;\;\;\forall n \geq N_{AB}(\epsilon).
	\end{align*}
	
	(i) If $A,B, A_n, B_n \in \Sym^{+}(\H) \cap \HS(\H)$, then by Proposition \ref{proposition:HS-inverse-2},
	\begin{align*}
		&||\log(I+C_n) - \log(I+C)||_{\HS}
		\leq \frac{1}{M_{AB} - \epsilon}[||A_n - A||_{\HS} + (1+||A||_{\HS})||B_n - B||_{\HS}]
		\;\;\;\forall n \geq N_{AB}(\epsilon).
	\end{align*}
	
	(ii) In the general setting, let $M_B > 0$ be such that
	$\la x, (I+B)x\ra \geq M_B||x||^2$ $\forall x \in \H$. By Proposition \ref{proposition:HS-inverse-2},
	$\forall \epsilon, 0 < \epsilon < M_B$, $\exists N_B(\epsilon) \in \Nbb$ such that $\forall n \geq N_B$,
	\begin{align*}
		||C_n - C||_{\HS}\leq \frac{1}{M_B-\epsilon}\left[||A_n - A||_{\HS} + \frac{1}{M_B-\epsilon}\left(1 + ||A||_{\HS}\right)||B_n - B||_{\HS} \right].
	\end{align*}
	Thus by Theorem \ref{theorem:logHS-convergence},
	$\forall \epsilon, 0 < \epsilon < \min\{M_{AB}, M_B\}$, $\forall n \geq \max\{N_{AB}, N_B\}$,
	\begin{align*}
		&||\log(I+C_n) - \log(I+C)||_{\HS}
		\leq \frac{1}{(M_{AB} - \epsilon)(M_B-\epsilon)}\left[||A_n - A||_{\HS} + \frac{1}{M_B-\epsilon}\left(1 + ||A||_{\HS}\right)||B_n - B||_{\HS} \right].
	\end{align*}
	In both cases,
	$\lim_{n \approach \infty}||\log(I+C_n) - \log(I+C)||_{\HS} = 0$.
	\qed
\end{proof}

\begin{proof}
	[\textbf{Proof of Corollary \ref{corollary:affineHS-sequence-approx-1}}]
	Since
	$(I + \frac{A}{\gamma})^{-1/2}(I + \frac{B}{\gamma})(I + \frac{A}{\gamma})^{-1/2} = (\gamma I + A)^{-1/2}(\gamma I + B)(\gamma I + A)^{-1/2} \geq M_{AB}
	$, applying Theorem 
	\ref{theorem:affineHS-sequence-approx-1} gives, $\forall n \geq N_{AB}(\ep)$,
	\begin{align*}
		&||\log[(\gamma I+B_n)^{-1/2}(\gamma I+A_n)(\gamma I+B_n)^{-1/2}]  
		-\log[(\gamma I+B)^{-1/2}(\gamma I+A)(\gamma I+B)^{-1/2}]||_{\HS} 
		\nonumber
		\\
		& = \left\|\log\left[\left(I+\frac{B_n}{\gamma}\right)^{-1/2}\left(I+\frac{A_n}{\gamma}\right)\left( I+\frac{B_n}{\gamma}\right)^{-1/2}\right]  
		-\log\left[\left(I+\frac{B}{\gamma}\right)^{-1/2}\left( I+\frac{A}{\gamma}\right)\left( I+\frac{B}{\gamma}\right)^{-1/2}\right]\right\|_{\HS}
		\\
		& \leq \frac{1}{M_{AB} - \ep}\left[\frac{1}{\gamma}||A_n - A||_{\HS} + \left(1+\frac{1}{\gamma}||A||_{\HS}\right)\frac{1}{\gamma}||B_n-B||_{\HS}\right]
		\\
		& = \frac{1}{(M_{AB} - \ep)\gamma}\left[||A_n - A||_{\HS} + \left(1+\frac{1}{\gamma}||A||_{\HS}\right)||B_n-B||_{\HS}\right].
		\qed
	\end{align*}
\end{proof}

\bibliographystyle{plain}
\bibliography{/Users/Minhs/Dropbox/cite_RKHS}

\begin{thebibliography}{47}
\providecommand{\natexlab}[1]{#1}
\providecommand{\url}[1]{\texttt{#1}}
\providecommand{\urlprefix}{}

\bibitem[{Larotonda(2007)G. Larotonda}]{Larotonda:2007}
Larotonda G.
\newblock Nonpositive curvature: A geometrical approach to {H}ilbert-{S}chmidt
  operators.
\newblock Differential Geometry and its Applications 2007;25:679--700.

\bibitem[{Minh et~al.(2014)H. Q. Minh and M. San Biagio and V.
  Murino}]{MinhSB:NIPS2014}
Minh HQ, Biagio MS, Murino V.
\newblock Log-{H}ilbert-{S}chmidt metric between positive definite operators on
  {H}ilbert spaces.
\newblock In: Advances in Neural Information Processing Systems 27 (NIPS 2014);
  2014.p. 388--396.

\bibitem[{Ramsay and Silverman(2005)Ramsay, J. and Silverman,
  B.}]{Ramsay:2005functional}
Ramsay J, Silverman B.
\newblock Functional data analysis.
\newblock Springer; 2005.

\bibitem[{Ferraty and Vieu(2006)Ferraty, F. and Vieu,
  P.}]{Ferraty:2006nonparametric}
Ferraty F, Vieu P.
\newblock Nonparametric functional data analysis: theory and practice.
\newblock Springer; 2006.

\bibitem[{Horv\'ath and Kokoszka(2012)Horv\'ath, L. and Kokoszka,
  P.}]{Horvath:2012functional}
Horv\'ath L, Kokoszka P.
\newblock Inference for Functional Data with Applications.
\newblock Springer; 2012.

\bibitem[{Panaretos et~al.(2010)Panaretos, V. and Kraus, D. and Maddocks,
  J.}]{Panaretos:jasa2010}
Panaretos V, Kraus D, Maddocks J.
\newblock Second-order comparison of {Gaussian} random functions and the
  geometry of {DNA} minicircles.
\newblock Journal of the American Statistical Association
  2010;105(490):670--682.

\bibitem[{Fremdt et~al.(2013)Fremdt, S. and Steinebach, J. and Horv{\'a}th, L.
  and Kokoszka, P.}]{Fremdt:2013testing}
Fremdt S, Steinebach J, Horv{\'a}th L, Kokoszka P.
\newblock Testing the equality of covariance operators in functional samples.
\newblock Scandinavian Journal of Statistics 2013;40(1):138--152.

\bibitem[{Pigoli et~al.(2014)Pigoli, D. and Aston, J. and Dryden, I.L. and
  Secchi, P.}]{Pigoli:2014}
Pigoli D, Aston J, Dryden IL, Secchi P.
\newblock Distances and inference for covariance operators.
\newblock Biometrika 2014;101(2):409--422.

\bibitem[{Masarotto et~al.(2018)Masarotto, V. and Panaretos, V.M. and Zemel,
  Y.}]{Masarotto:2018Procrustes}
Masarotto V, Panaretos VM, Zemel Y.
\newblock Procrustes Metrics on Covariance Operators and Optimal Transportation
  of {Gaussian} Processes.
\newblock Sankhya A 2018;p. 1--42.

\bibitem[{Villani(2008)Villani, C.}]{Villani:2008Optimal}
Villani C.
\newblock Optimal transport: old and new, vol. 338.
\newblock Springer Science \& Business Media; 2008.

\bibitem[{Minh(2021{\natexlab{a}})H.Q.
  Minh}]{Minh:2021EntropicConvergenceGaussianMeasures}
Minh HQ.
\newblock Convergence and finite sample approximations of entropic regularized
  {Wasserstein} distances in {Gaussian} and {RKHS} settings.
\newblock arXiv preprint arXiv:210101429 2021;.

\bibitem[{Minh(2021{\natexlab{b}})Minh, H.Q.}]{Minh2021:FiniteEntropicGaussian}
Minh HQ.
\newblock Finite sample approximations of exact and entropic Wasserstein
  distances between covariance operators and Gaussian processes.
\newblock arXiv preprint arXiv:210412368 2021;.

\bibitem[{Matthews et~al.(2016)Matthews, A. and Hensman, J. and Turner, R. and
  Ghahramani, Z.}]{Matthews2016sparseKL}
Matthews A, Hensman J, Turner R, Ghahramani Z.
\newblock On sparse variational methods and the {Kullback-Leibler} divergence
  between stochastic processes.
\newblock In: Artificial Intelligence and Statistics PMLR; 2016. p. 231--239.

\bibitem[{Sun et~al.(2019)Sun, S. and Zhang, G. and Shi, J. and Grosse,
  R.}]{Sun2019functionalKL}
Sun S, Zhang G, Shi J, Grosse R.
\newblock Functional variational {Bayesian} neural networks.
\newblock International Conference on Learning Representation 2019;.

\bibitem[{H{\`a}~Quang(2020)H{\`a} Quang, Minh}]{Minh:2020Functional}
H{\`a}~Quang M.
\newblock Riemannian Distances between Covariance Operators and Gaussian
  Processes.
\newblock In: International Workshop on Functional and Operatorial Statistics
  Springer; 2020. p. 177--185.

\bibitem[{Sun(2005)Sun, H.}]{Sun2005MercerNoncompact}
Sun H.
\newblock Mercer theorem for {RKHS} on noncompact sets.
\newblock Journal of Complexity 2005;21(3):337--349.

\bibitem[{Rosasco et~al.(2010)L. Rosasco and M. Belkin and E. De
  Vito}]{Rosasco:IntegralOperatorsJMLR2010}
Rosasco L, Belkin M, Vito ED.
\newblock On Learning with Integral Operators.
\newblock Journal of Machine Learning Research 2010;11(30):905--934.

\bibitem[{Cucker and Smale(2002)F. Cucker and S. Smale}]{CuckerSmale}
Cucker F, Smale S.
\newblock On the Mathematical Foundations of Learning.
\newblock Bulletin of the American Mathematical Society 2002
  January;39(1):1--49.

\bibitem[{Rajput and Cambanis(1972)Rajput, B. and Cambanis,
  S.}]{Rajput1972gaussianprocesses}
Rajput B, Cambanis S.
\newblock Gaussian processes and {Gaussian} measures.
\newblock The Annals of Mathematical Statistics 1972;p. 1944--1952.

\bibitem[{Masarotto et~al.(2019)Masarotto, V. and Panaretos, V. and Zemel,
  Y.}]{masarotto2019procrustes}
Masarotto V, Panaretos V, Zemel Y.
\newblock Procrustes metrics on covariance operators and optimal transportation
  of Gaussian processes.
\newblock Sankhya A 2019;81(1):172--213.

\bibitem[{Dryden et~al.(2009)Dryden, I.L. and Koloydenko, A. and Zhou,
  D.}]{Dryden:2009}
Dryden IL, Koloydenko A, Zhou D.
\newblock Non-{E}uclidean Statistics for Covariance Matrices, with Applications
  to Diffusion Tensor Imaging.
\newblock Annals of Applied Statistics 2009;3:1102--1123.

\bibitem[{Dowson and Landau(1982)D.C. Dowson and B.V. Landau}]{Downson:1982}
Dowson DC, Landau BV.
\newblock The {Fr\'echet} distance between multivariate normal distributions.
\newblock Journal of Multivariate Analysis 1982;12(3):450 -- 455.

\bibitem[{Givens and Shortt(1984)Givens, C. and Shortt, R.}]{Givens:1984}
Givens C, Shortt R.
\newblock A class of {Wasserstein} metrics for probability distributions.
\newblock Michigan Math J 1984;31(2):231--240.

\bibitem[{Gelbrich(1990)Gelbrich, M.}]{Gelbrich:1990Wasserstein}
Gelbrich M.
\newblock On a formula for the {L2} {Wasserstein} metric between measures on
  {Euclidean} and {Hilbert} spaces.
\newblock Mathematische Nachrichten 1990;147(1):185--203.

\bibitem[{Bhatia et~al.(2018)Bhatia, R. and Jain, T. and Lim,
  Y.}]{Bhatia:2018Bures}
Bhatia R, Jain T, Lim Y.
\newblock On the {Bures--Wasserstein} distance between positive definite
  matrices.
\newblock Expositiones Mathematicae 2018;.

\bibitem[{Malag{\`o} et~al.(2018)Malag{\`o}, L. and Montrucchio, L. and
  Pistone, G.}]{Malago:WassersteinGaussian2018}
Malag{\`o} L, Montrucchio L, Pistone G.
\newblock Wasserstein {Riemannian} geometry of {Gaussian} densities.
\newblock Information Geometry 2018 Dec;1(2):137--179.

\bibitem[{Pennec et~al.(2006)X. Pennec and P. Fillard and N.
  Ayache}]{Pennec:IJCV2006}
Pennec X, Fillard P, Ayache N.
\newblock A {R}iemannian Framework for Tensor Computing.
\newblock International Journal of Computer Vision 2006;66(1):41--66.

\bibitem[{Bhatia(2007)R. Bhatia}]{Bhatia:2007}
Bhatia R.
\newblock Positive Definite Matrices.
\newblock Princeton University Press; 2007.

\bibitem[{Arsigny et~al.(2007)Arsigny, V. and Fillard, P. and Pennec, X. and
  Ayache, N.}]{LogEuclidean:SIAM2007}
Arsigny V, Fillard P, Pennec X, Ayache N.
\newblock Geometric means in a novel vector space structure on symmetric
  positive-definite matrices.
\newblock SIAM J on Matrix An and App 2007;29(1):328--347.

\bibitem[{Takatsu(2011)Takatsu, A.}]{Takatsu2011wasserstein}
Takatsu A.
\newblock Wasserstein geometry of {Gaussian} measures.
\newblock Osaka Journal of Mathematics 2011;48(4):1005--1026.

\bibitem[{Minh(2019{\natexlab{a}})Minh, H.Q.}]{Minh:GSI2019}
Minh HQ.
\newblock A unified formulation for the {Bures-Wasserstein} and
  {Log-Euclidean/Log-Hilbert-Schmidt} distances between positive definite
  operators.
\newblock In: International Conference on Geometric Science of Information
  Springer; 2019. .

\bibitem[{Minh(2019{\natexlab{b}})H.Q. Minh}]{Minh:Alpha2019}
Minh HQ.
\newblock Alpha {Procrustes} metrics between positive definite operators: a
  unifying formulation for the {Bures-Wasserstein} and
  {Log-Euclidean}/{Log-Hilbert-Schmidt} metrics.
\newblock arXiv preprint arXiv:190809275 2019;.

\bibitem[{Kadison and Ringrose(1983)R.V. Kadison and J.R.
  Ringrose}]{Kadison:1983}
Kadison RV, Ringrose JR.
\newblock Fundamentals of the theory of operator algebras. Volume I: Elementary
  Theory.
\newblock Academic Press; 1983.

\bibitem[{Petryshyn(1962)W.V. Petryshyn}]{Petryshyn:1962}
Petryshyn WV.
\newblock Direct and iterative methods for the solution of linear operator
  equations in {H}ilbert spaces.
\newblock Transactions of the American Mathematical Society 1962;105:136--175.

\bibitem[{Minh(2015)Minh, H.Q.}]{Minh:GSI2015}
Minh HQ.
\newblock Affine-invariant {Riemannian} distance between infinite-dimensional
  covariance operators.
\newblock In: International Conference on Geometric Science of Information;
  2015. .

\bibitem[{Minh(2017)Minh, {H.Q.}}]{Minh:LogDet2016}
Minh H.
\newblock Infinite-dimensional {Log-Determinant} divergences between positive
  definite trace class operators.
\newblock Linear Algebra and Its Applications 2017;528:331--383.

\bibitem[{Minh(2020)Minh, H.Q.}]{Minh:Positivity2020}
Minh HQ.
\newblock Infinite-dimensional {Log-Determinant} divergences between positive
  definite {Hilbert-Schmidt} operators.
\newblock Positivity 2020;24:631--662.

\bibitem[{Minh(2019)Minh, H.Q.}]{Minh:INGE2019}
Minh HQ.
\newblock {Alpha-Beta} {Log-Determinant} Divergences Between Positive Definite
  Trace Class Operators.
\newblock Information Geometry 2019 December;2(2):101--176.

\bibitem[{Minh(2020)H.Q. Minh}]{Minh2020:EntropicHilbert}
Minh HQ.
\newblock Entropic regularization of {Wasserstein} distance between
  infinite-dimensional {Gaussian} measures and {Gaussian} processes.
\newblock preprint arXiv:201107489 2020;.

\bibitem[{Smale and Zhou(2007)S. Smale and D.X. Zhou}]{SmaleZhou2007}
Smale S, Zhou DX.
\newblock Learning Theory Estimates via Integral Operators and Their
  Approximations.
\newblock Constructive Approximation 2007;26:153--172.

\bibitem[{Harandi et~al.(2014)M. Harandi and M. Salzmann and F.
  Porikli}]{Covariance:CVPR2014}
Harandi M, Salzmann M, Porikli F.
\newblock Bregman Divergences for Infinite Dimensional Covariance Matrices.
\newblock In: IEEE Coneference on Computer Vision and Pattern Recognition
  (CVPR); 2014. .

\bibitem[{Steinwart and Christmann(2008)Steinwart, I. and Christmann,
  A.}]{Steinwart:SVM2008}
Steinwart I, Christmann A.
\newblock Support vector machines.
\newblock Springer Science \& Business Media; 2008.

\bibitem[{Minh(2019)Minh, H. Q.}]{Minh:2019AlphaBeta}
Minh HQ.
\newblock {Alpha-Beta Log-Determinant} divergences between positive definite
  trace class operators.
\newblock Information Geometry 2019;2(2):101--176.

\bibitem[{Kittaneh and Kosaki(1987)Kittaneh, F. and Kosaki,
  H.}]{Kitta:InequalitiesV}
Kittaneh F, Kosaki H.
\newblock Inequalities for the {Schatten p-norm V}.
\newblock Publications of the Research Institute for Mathematical Sciences
  1987;23(2):433--443.

\bibitem[{Minh(2018)Minh, H.Q.}]{Minh:LogDetIII2018}
Minh HQ.
\newblock Infinite-Dimensional {Log-Determinant} Divergences {III}:
  {Log-Euclidean} and {Log-Hilbert--Schmidt} Divergences.
\newblock In: Information Geometry and its Applications IV Springer; 2018. p.
  209--243.

\bibitem[{Dunford and Schwartz(1988)Dunford, Nelson and Schwartz, Jacob
  T}]{Dunford1988linearOperators}
Dunford N, Schwartz JT.
\newblock Linear operators, part 1: general theory, vol.~10.
\newblock John Wiley \& Sons; 1988.

\bibitem[{Reed and Simon(1975)M. Reed and B. Simon}]{ReedSimon:Functional}
Reed M, Simon B.
\newblock Methods of Modern Mathematical Physics: Functional analysis.
\newblock Academic Press; 1975.

\end{thebibliography}

\end{document}